\documentclass{article} 
\usepackage{iclr2027_conference,times}

\usepackage{amsmath,amsfonts,bm}

\def\eqref#1{equation~\ref{#1}}

\def\1{\bm{1}}

\DeclareMathAlphabet{\mathsfit}{\encodingdefault}{\sfdefault}{m}{sl}
\SetMathAlphabet{\mathsfit}{bold}{\encodingdefault}{\sfdefault}{bx}{n}

\usepackage[utf8]{inputenc} 
\usepackage[T1]{fontenc}    
\usepackage{hyperref}       
\usepackage{url}            
\usepackage{booktabs}       
\usepackage{amsfonts}       
\usepackage{nicefrac}       
\usepackage{microtype}      
\usepackage{xcolor}         
\usepackage{etoc}
\usepackage{subcaption}
\usepackage{wrapfig}
\usepackage{enumitem}
\usepackage{tikz}
\usepackage{graphicx}
\usepackage{subcaption}
\PassOptionsToPackage{hyphens}{url}\usepackage{hyperref}

\usepackage{etoolbox}
\newbool{iclrtemp}
\setbool{iclrtemp}{true}

\usepackage[most]{tcolorbox}
\usepackage{hyperref}

\definecolor{linkblue}{HTML}{CD5700}

\usepackage[most]{tcolorbox}
\usepackage{hyperref}
\hypersetup{
  colorlinks=true,
  linkcolor=linkblue,  
  citecolor=linkblue,  
  urlcolor=linkblue    
}

\definecolor{coralaccent}{HTML}{E85D4A}
\definecolor{stegobox}{HTML}{1a1a2e}
\definecolor{darktext}{HTML}{263238}

\makeatletter
\let\origaddcontentsline\addcontentsline
\makeatother

\makeatletter
\let\addcontentsline\origaddcontentsline
\makeatother

\usepackage{amsmath}
\usepackage{amssymb}
\usepackage{mathtools}
\usepackage{amsthm}
\usepackage{enumitem}
\usepackage{comment}
\usepackage{pifont}

\usepackage[capitalize,noabbrev]{cleveref}

\usepackage{booktabs}
\usepackage{array}
\usepackage{enumitem}
\usepackage{longtable}
\usepackage{tabularx}

\usepackage{pifont}
\usepackage{xcolor}
\usepackage{makecell}

\newlist{cellitemize}{itemize}{1}
\setlist[cellitemize]{nosep,leftmargin=*,label=--}

\newcolumntype{L}[1]{>{\raggedright\arraybackslash}p{#1}}

\theoremstyle{plain}
\newtheorem{theorem}{Theorem}[section]
\newtheorem{proposition}[theorem]{Proposition}

\newtheorem{corollary}[theorem]{Corollary}
\theoremstyle{definition}

\theoremstyle{remark}

\usepackage[textsize=tiny]{todonotes}

\tcbuselibrary{listings,skins,breakable}
\usepackage{fontawesome5}
\definecolor{custom_green}{HTML}{009E73}
\definecolor{custom_blue}{HTML}{0072B2}

\usepackage{tcolorbox}
\tcbuselibrary{breakable,skins}

\definecolor{takeawayblue}{HTML}{315F7D}

\newenvironment{highlight}[1][Takeaways]
{%
\begin{tcolorbox}[
    enhanced,
    breakable,
    frame hidden,
    colback=takeawayblue!5,
    borderline west={2pt}{0pt}{takeawayblue},
    boxsep=0pt,
    left=2.5mm,
    right=2.5mm,
    top=2mm,
    bottom=2mm
]
\noindent
\textcolor{takeawayblue}{%
    \faLightbulb\hspace{0.4em}\textbf{#1.}\hspace{0.5em}%
}\hspace{0.35em}\ignorespaces
}
{%
\unskip
\end{tcolorbox}
}

\title{On-Policy or Off-Policy Learning? \\ A Systematic Study of Distillation Dynamics}

\author{
Julianna Piskorz\thanks{Authors contributed equally.},\hspace{0.2em}
Antonin Berthon\footnotemark[1] \hspace{0.2em} \&
Mihaela van der Schaar\\
University of Cambridge\\
Cambridge, UK\\
\texttt{\{jp2048,armb3\}@cam.ac.uk}
}

\iclrfinalcopy 
\begin{document}

\maketitle

\etocdepthtag.toc{main}

\begin{abstract}
On-policy learning has been argued to reduce catastrophic forgetting, produce sparser parameter updates, and improve generalisation. However, existing comparisons between supervised fine-tuning and reinforcement learning vary many factors simultaneously, making the contribution of rollout policy difficult to isolate. We study the effect of rollout policy in a controlled strong-to-weak distillation setting, by independently varying rollout policy, token-level KL direction, and learning rate across the Llama~3 and Qwen2.5 model families and reasoning tasks spanning scientific, medical, and arithmetic domains. Our analysis reveals a nuanced picture of distillation dynamics in which rollout policy does not necessarily play a central role. Instead, token-level KL direction more clearly shapes task performance and output coverage, while learning rate governs forgetting and update sparsity. Analysis of KL gradients and experiments along a continuous student--teacher rollout-policy spectrum explain this pattern: forward KL is remarkably robust to rollout policy, with its performance stable and strong despite changes to the rollout policy, whereas reverse KL is substantially more sensitive and favours student-generated rollouts. On-policy data nevertheless improves generalisation to harder variants of the Countdown arithmetic task under both KL directions, although this advantage does not reliably persist after subsequent RLVR. Our broader conclusions remain robust to removing gradient clipping, using sampled KL estimators, and training on tasks requiring longer reasoning chains. Overall, our results challenge the view that on-policy rollouts are inherently preferable and show that their value depends critically on the objective, evaluation setting, and optimisation hyperparameters.
\end{abstract}

\section{Introduction}

Post-training plays a central role in developing the reasoning capabilities of large language models across a wide range of domains. Prominent approaches include supervised fine-tuning (SFT) on high-quality reasoning traces \citep{guha_openthoughts_2025, team_kimi_2026}, reinforcement learning with verifiable rewards (RLVR) \citep{deepseek-ai_deepseek-r1_2025, shao_deepseekmath_2024, yu_dapo_2025}, and knowledge distillation (KD)~\citep{hinton_distilling_2015, agarwal_-policy_2024, team_mimo-v2-flash_2026}. 
Strong-to-weak distillation has become especially prominent in modern post-training pipelines, sing stronger teacher models to transfer reasoning behaviour to weaker students~\citep{team2025qwen3, deepseekai2026deepseekv4highlyefficientmilliontoken}. As the number of available methods grows, classifying their relative strengths and limitations, becomes essential, guiding the development of more effective and efficient training recipes.

One increasingly influential hypothesis singles out the policy used to generate training data as a key factor differentiating post-training methods. Off-policy methods, such as SFT, learn from a fixed dataset or data-generating policy, whereas on-policy methods, such as RLVR, repeatedly train on outputs sampled from the model being optimised. Although SFT and RLVR differ in several other respects, this difference in rollout policy has been proposed as a central factor underlying their differing behaviour.
In particular, relative to off-policy learning, on-policy learning has been argued to reduce or prevent catastrophic forgetting \citep{shenfeld_rls_2025, chen_retaining_2025, shenfeld_self-distillation_2026}, produce substantially sparser parameter updates \citep{mukherjee_reinforcement_2025}, and improve generalisation \citep{chu_sft_2025, zhang_towards_2026, yuan2026f, ming_one-token_2026}. These findings have contributed to broader interest in on-policy post-training, including on-policy distillation \citep{song2026surveyonpolicydistillationlarge, agarwal_-policy_2024, lu2025onpolicydistillation, team_mimo-v2-flash_2026}. 

Yet the causal role of rollout policy remains unclear. Comparisons between SFT and RLVR also change the objective, source and density of supervision, and optimisation procedure, so effects attributed to on-policy data may arise from these accompanying differences. Resolving this ambiguity is practically important: on-policy methods require continual generation throughout training, whereas off-policy methods can reuse existing data. If their proposed benefits arise from other design choices, this additional cost may often be unnecessary.

In this work, we use strong-to-weak distillation as a controlled testbed for isolating the effect of rollout policy. Unlike SFT vs RLVR comparisons, this setting allows us to change the data-generating policy while holding the remaining training pipeline fixed. Across the Llama~3 and Qwen2.5 model families and reasoning tasks spanning scientific, medical, and arithmetic domains, we independently vary the rollout policy and KL direction, breaking the conventional pairing of forward KL with off-policy learning and reverse KL with on-policy learning, while controlling for the learning rate. We find \textit{no consistent advantage from on-policy distillation} in final in-distribution accuracy, catastrophic forgetting, or parameter-update sparsity. Instead, KL direction more clearly determines task accuracy, while learning rate governs forgetting and sparsity. In particular, \textbf{we find no evidence that on-policy rollouts intrinsically preserve prior capabilities or induce sparser updates.}

To understand when rollout policy does matter, we examine its interaction with the token-level KL objective through a theoretical and empirical analysis. We vary the rollout policy along a student--teacher spectrum, interpolating between the two policies and extrapolating beyond them to favour tokens preferred by one model over the other. Our analysis reveals a clear objective-dependent asymmetry: forward KL is remarkably robust to changes in the rollout policy, whereas reverse KL is substantially more sensitive and favours on-policy rollouts. Nevertheless, \textit{forgetting and update sparsity remain governed by learning rate}, with little to no variation caused by the rollout policy.

The benefits of different configurations become more nuanced when considering evaluations beyond final accuracy. Forward KL produces larger pass@$k$ gains on the in-distribution task compared to reverse KL, while \textbf{generalisation to harder task variants on Countdown consistently benefits from on-policy data}, under both KL directions. However, this initial generalisation advantage of on-policy learning does not reliably persist after subsequent RLVR. When combined with reverse KL, on-policy rollouts also reduce incidental teacher-style transfer. Our principal findings remain robust when removing gradient clipping, replacing full-vocabulary KL with commonly used sampled estimators, and training on tasks requiring longer reasoning chains. \textbf{Our main contributions are:}
\begin{itemize}[leftmargin=20pt]
\item \textbf{A controlled study of distillation.} We disentangle rollout policy, KL direction, and learning rate in a controlled study of strong-to-weak distillation, finding no consistent advantage from on-policy rollouts in final accuracy, catastrophic forgetting, or parameter-update sparsity.
\item \textbf{An objective-dependent account of rollout-policy sensitivity.} Through gradient analysis and experiments along a continuous student--teacher spectrum, we show that forward KL is robust to rollout policy, whereas reverse KL is substantially more sensitive.
\item \textbf{A characterisation of when on-policy data helps.} While training stability and output coverage are largely driven by KL direction, and learning rate determines the degree of catastrophic forgetting and update sparsity, we find that on-policy rollouts consistently help to improve model's generalisation to harder task variants and could reduce incidental teacher-style transfer.
\end{itemize}

Together, these findings challenge the view that on-policy rollouts are universally preferable for strong-to-weak distillation.
Instead, the choice of token-level KL divergence strongly affects task performance, training stability and output coverage, while catastrophic forgetting and update sparsity are driven by the learning rate. 
We hope that these findings encourage practitioners to weight the additional cost of on-policy distillation over cheaper off-policy alternatives, and help clarify which training outcomes can be attributed to rollout policy rather than KL direction or learning rate.

\section{Related Works}

\textbf{Strong-to-Weak and On-Policy Distillation.} Knowledge distillation transfers capabilities from a teacher to a student by matching their output distributions \citep{hinton_distilling_2015, kim-rush-2016-sequence}. In language-model post-training, strong-to-weak distillation has become a common approach for transferring reasoning capabilities from larger models to smaller ones \citep{abdin2024phi3technicalreporthighly, team_mimo-v2-flash_2026, team2025qwen3, deepseekai2026deepseekv4highlyefficientmilliontoken, team2026kimik3}. The rollout source defines an important distinction within these methods: in OffPD, the student trains on fixed teacher-generated responses \citep{sanh2020distilbertdistilledversionbert}, while in OnPD, the student generates trajectories and queries the teacher for token-level supervision along them \citep{agarwal_-policy_2024}. Hybrid variants interpolate between the two, for example by constructing trajectories from mixtures of student- and teacher-generated tokens \citep{xu_speculative_2025}. The growing literature on OnPD has investigated its algorithmic variants and failure modes \citep{song2026surveyonpolicydistillationlarge, li_rethinking_2026, armandpour_unmasking_2026, jia_asymmetric_2026, li_filter_2026, zhu2026facesonpolicydistillationpitfalls}, but few works directly contrast OnPD with the cheaper OffPD alternative while holding the remaining training configuration fixed.

\textbf{Claims About On-Policy Post-Training.}
In the broader post-training literature, recent studies have attributed several advantages to on-policy training, such as better preservation of previously acquired capabilities, thereby reducing catastrophic forgetting \citep{shenfeld_rls_2025, chen_retaining_2025, shenfeld_self-distillation_2026, lu2025onpolicydistillation}, sparser parameter updates \citep{mukherjee_reinforcement_2025, yu_dense_2026} and improved generalisation beyond the training distribution \citep{chu_sft_2025, yuan2026f, zhang_towards_2026, ming_one-token_2026}. However, much of this evidence compares SFT with RLVR, analyses public checkpoints, or otherwise varies multiple aspects of the training recipe simultaneously. Such comparisons can entangle rollout source with the learning objective, reward signal, supervision density, optimisation scale, use of gradient clipping, and learning rate, which has itself been shown to strongly affect learning--forgetting trade-offs \citep{catalan-tatjer_learning-forgetting_2026, rofin_how_2026}. We therefore treat these reported results as hypotheses about the advantages of on-policy learning and test them in a strong-to-weak distillation setting where the rollout policy, KL direction, and learning rate are varied explicitly.

\section{Background}
\label{sec:background}

\textbf{Notation.} For the language-modelling tasks considered in this work, let $x$ and $y$ denote the input and output sequences, respectively, each consisting of tokens from the vocabulary $\mathcal{V}$ with $|\mathcal{V}|=M$. Let $y_{<n} = (y_1, \dots, y_{n-1})$ denote the output prefix preceding the $n^{\text{th}}$ token, and let $L_y$ denote the length of $y$. We assume that the input sequences (prompts) are sampled from a fixed data distribution $p_{\mathrm{data}}$, while the output sequences are generated autoregressively from a policy $\pi$, which for a given input $x$ and a partially generated output $y_{<n}$ outputs a discrete distribution over the entire token vocabulary $\mathcal{V}$, $\pi(\cdot | x, y_{<n}) \in \Delta(\mathcal{V})$. 

\textbf{Strong-to-Weak Distillation.} In strong-to-weak knowledge distillation we assume access to two autoregressive language models: a student $\pi_S^\theta$ (parametrised by $\theta$) and a teacher $\pi_T$ providing the supervision signal. Let $\rho$ denote the rollout policy used to generate the output sequences on which distillation is performed. Then, the distillation objective can be defined as:
\begin{equation}
\label{eq:distillation_loss}
    \mathcal{L}(\theta) = \mathbb{E}_{x \sim p_{\mathrm{data}}}\mathbb{E}_{y \sim \rho( \cdot | x)}\left[\frac{1}{L_y}\sum_{n=1}^{L_y} \mathcal{D}\left(\pi_S^\theta(\cdot |x, y_{<n}) \, || \, \pi_T(\cdot |x,y_{<n}) \right) \right],
\end{equation}
where $\mathcal{D}$ is a distance function quantifying the discrepancy between the next-token student and teacher distributions. The objective therefore averages the token-level discrepancy over prefixes visited under the rollout policy $\rho$.

Under this formulation, we obtain \textbf{off-policy distillation} (OffPD) by setting $\rho = \pi_T$ and \textbf{on-policy distillation} (OnPD) by setting $\rho = \pi_S^{\theta}$. Consequently, in OnPD the loss function depends on $\theta$ also through the trajectory sampling procedure $y \sim \pi_S^\theta(\cdot | x)$. However, for computational tractability, we stop gradients through the sampling process and treat sampled trajectories as fixed when computing each update \citep{agarwal_-policy_2024, tang_few_2025}.

\textbf{KL Divergence as the Distance Function.} We use the Kullback--Leibler divergence to measure the distance between the student and teacher next-token distributions. Because we have access to their full output distributions, we compute the divergence over the entire vocabulary. This avoids the additional variance introduced by estimating the token-level KL from sampled tokens and uses the complete supervision signal provided by the teacher \citep{deepseekai2026deepseekv4highlyefficientmilliontoken}.
Given that the KL divergence is not symmetric, we can distinguish between the forward and the reverse KL:
\begin{equation}
    \text{Forward-KL:} \quad D_{\text{F-KL}}(\pi_S, \pi_T) = \mathbb{E}_{y_n \sim \pi_T}\left[\log{\frac{\pi_T(y_n)}{\pi_S(y_n)}}\right] = \sum_{v \in \mathcal{V}} \pi_T(v) \log{\frac{\pi_T(v)}{\pi_S(v)}},
\end{equation}
\begin{equation}
    \text{Reverse-KL:} \quad D_{\text{R-KL}}(\pi_S, \pi_T) = \mathbb{E}_{y_n \sim \pi_S}\left[\log{\frac{\pi_S(y_n)}{\pi_T(y_n)}}\right] = \sum_{v \in \mathcal{V}} \pi_S(v) \log{\frac{\pi_S(v)}{\pi_T(v)}}.
\end{equation}
For conciseness, we write $\pi(y_n) = \pi(\cdot | x, y_{<n})$ for the distribution at the $n^{\text{th}}$ token. Forward KL places greater weight on tokens to which the teacher assigns substantial probability and strongly penalises the student for failing to cover them. It is therefore commonly described as \textit{mode-covering}. Reverse KL in turn places greater weight on tokens favoured by the student, and penalises assigning probability to tokens which are unlikely under the teacher, giving rise to its \textit{mode-seeking} behaviour.

\textbf{Rollout Policy and KL Direction Coupling.} Although OnPD is often associated with reverse KL and OffPD with forward KL, the rollout policy and token-level KL direction are conceptually distinct. The rollout policy $\rho$ determines which prefixes are visited, while the KL direction determines how the student and teacher distributions are compared at those prefixes. The conventional pairings have two motivations.
Firstly, when considering the sequence-level KL divergence between the student and the teacher, teacher rollouts paired with forward KL and student rollouts paired with reverse KL are exactly the two chain-rule decompositions of sequence-level KL (see Appendix\ref{app:sequence-level-kl}). However, because we stop gradients through student-sampled prefixes, this pairing holds only at the objective-value level: the OnPD reverse-KL update is generally \textit{not} the total gradient of sequence-level reverse KL. Second, at a fixed prefix, $\log[\pi_T(v)/\pi_S^\theta(v)]$ with $v\sim\pi_T$ is an unbiased one-sample estimator of forward KL, while $\log[\pi_S^\theta(v)/\pi_T(v)]$ with $v\sim\pi_S^\theta$ is an unbiased estimator of reverse KL. When token-level KLs are estimated from single samples, the rollout tokens therefore naturally pair OffPD with forward KL and OnPD with reverse KL. In our setting, however, we compute the token-level KL \textit{over the full vocabulary} $\mathcal{V}$, so the KL direction does not constrain the rollout policy.  Moreover, neither the chain-rule identities nor the sampling interpretation implies that the conventional pairings are easier to optimise or outperform the crossed pairings. We therefore treat rollout policy and KL direction as independent design choices and study their interaction experimentally.



\section{A Controlled Comparison of On- and Off-Policy Distillation}
\label{sec:on_vs_offpd}
\subsection{Experimental Setup}

\textbf{Experimental design.}
We compare OnPD and OffPD while independently varying token-level KL direction and learning rate ($1\times10^{-5}$ or $5 \times 10^{-5}$). We compute both KL objectives over the full vocabulary and keep the remaining optimisation settings fixed. Each student is trained with full-parameter fine-tuning for 150 optimisation steps, with three random seeds per configuration. Complete training details are provided in Appendix~\ref{app:details:training}.


\textbf{Models and tasks.}
Our main experiments distil Llama-3.1-8B teachers into Llama-3.2-1B students~\citep{grattafiori_llama_2024}, while Appendix~\ref{app:exp:qwen25} reports corresponding experiments with Qwen2.5-7B teachers and Qwen2.5-1.5B students~\citep{qwen2025qwen25technicalreport}. We consider three reasoning tasks spanning different domains: \textit{MedReason} \citep{wu2025medreasonelicitingfactualmedical} for medical reasoning, \textit{Science} \citep{feng2024sciknoweval} for scientific reasoning, and \textit{Countdown} \citep{tinyzero} for arithmetic reasoning.
Teachers and students are taken from the same model family to ensure a shared tokenizer and vocabulary, allowing to directly compare their token-level distributions.
Since overlapping capabilities within a model family can limit knowledge transfer \citep{li_rethinking_2026}, we train a separate teacher on each task to establish a sufficient capability gap, then freeze it throughout distillation.
Dataset preparation and teacher training are detailed in Appendix~\ref{app:experimental_details:datasets} and~\ref{app:details:trained_teacher}.

\textbf{Evaluation.}
For each trained checkpoint, we assess held-out accuracy on the target task, catastrophic forgetting, and parameter-update sparsity.
We measure catastrophic forgetting by the decrease in the mean score across seven out-of-distribution benchmarks from before to after training (Appendix~\ref{app:ood_eval}).
Following \citet{mukherjee_reinforcement_2025}, we measure update sparsity as the fraction of parameters whose absolute change from the initial checkpoint is below $10^{-6}$ (Appendix~\ref{app:update-sparsity}).

\begin{figure}
    \centering
    \includegraphics[width=0.9\linewidth]{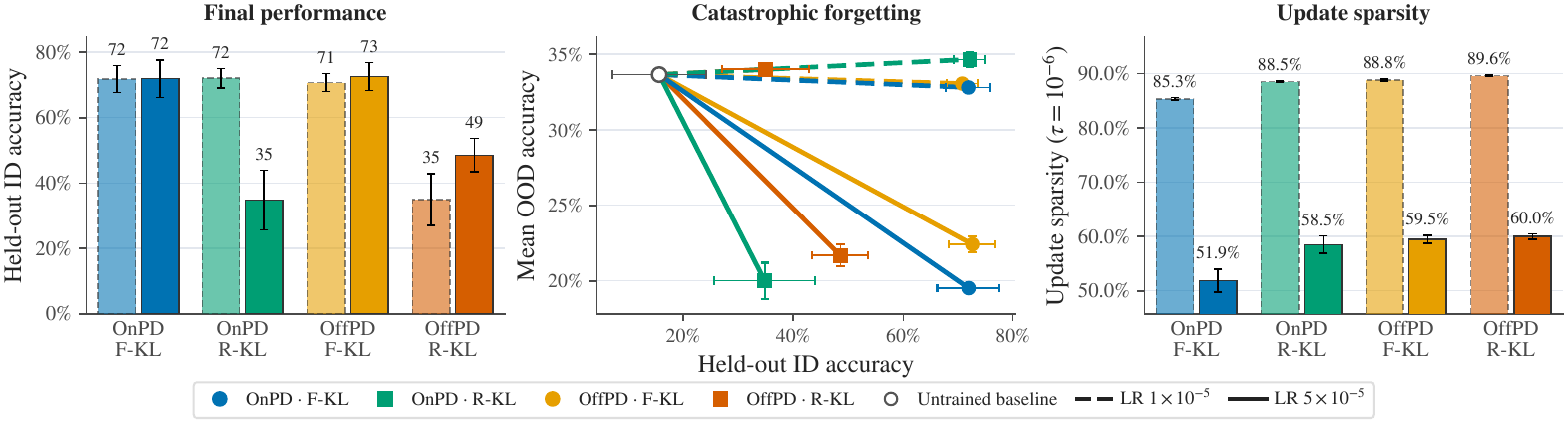}
    \vspace{-0.5em}
    \caption{\textbf{Comparison of OnPD and OffPD across three reasoning tasks.} Held-out ID accuracy in the left and middle panels is averaged across MedReason, Science, and Countdown-3. Per-dataset results are reported in~\Cref{fig:app:per_dataset_performance}. Error bars show SEM computed over $N=3$ seeds.}
    \label{fig:on_vs_offpd}
    \vspace{-1em}
\end{figure}

\subsection{Results}

\textbf{Final task performance.} 
On-policy rollouts offer no consistent advantage in target-task accuracy: the best mean accuracies across the three datasets are nearly identical, reaching 72\% for OnPD and 73\% for OffPD (\Cref{fig:on_vs_offpd}, left).
KL direction produces a clearer pattern: forward KL achieves 71--73\% across all rollout policies and learning rates, whereas reverse KL ranges from 35\% to 72\% and is substantially more sensitive to the learning rate.

\textbf{Catastrophic forgetting.} Similar task performance might conceal different amounts of forgetting on OOD tasks. If on-policy rollouts mitigate forgetting, as suggested in prior work~\citep{shenfeld_rls_2025, chen_retaining_2025}, we would expect OnPD to preserve OOD performance better than OffPD with the same KL objective and learning rate.
Instead, mean OOD performance changes by at most 1.3 percentage points at the lower learning rate but drops by 11.2--14.0 points at the higher rate (\Cref{fig:on_vs_offpd}, middle). Rollout-policy differences are modest compared to that. Remarkably, with forward KL, lowering the learning rate mitigates most of this forgetting while maintaining comparable ID accuracy. In this setting, learning rate, rather than on-policy rollouts, is the dominant factor.

\textbf{Parameter-update sparsity.} The same separation by learning rate appears in the parameter updates (\Cref{fig:on_vs_offpd}, right). Sparsity ranges from 85.3--89.6\% at the lower learning rate, compared with 51.9--60.0\% at the higher learning rate. Again, differences between rollout policies are much smaller: OffPD produces at least as sparse updates as OnPD in every matched comparison. KL direction has a secondary effect, with reverse KL generally producing greater sparsity than forward KL on on-policy rollouts. Thus, sparser updates do not emerge as a benefit of on-policy rollouts in these experiments.


\textbf{Learning-rate sweep.} The comparison above uses two learning rates. To examine these effects more systematically, we sweep the learning rate from $1\times10^{-5}$ to $6\times10^{-5}$ on Countdown-3, keeping the remaining training configuration fixed. \Cref{fig:learning_rate} shows that update sparsity decreases approximately linearly as learning rate increases, while OOD performance also deteriorates consistently, with similar trends under both rollout policies. At matched learning rates, OffPD exhibits less forgetting in nearly every comparison, although this difference remains smaller than the effect of learning rate.
Crucially, \textbf{greater forgetting does not appear to be a necessary cost of achieving strong target-task performance}: under forward-KL, using small learning rate allows to achieve top performance, without leading to a decrease in the prior capabilities. Models with similar final task performance can therefore exhibit markedly different degrees of forgetting, resembling the ``cliff'' phenomenon previously identified in the SFT setting by \citet{catalan-tatjer_learning-forgetting_2026}.

\begin{figure}
    \centering
    \includegraphics[width=0.9\linewidth]{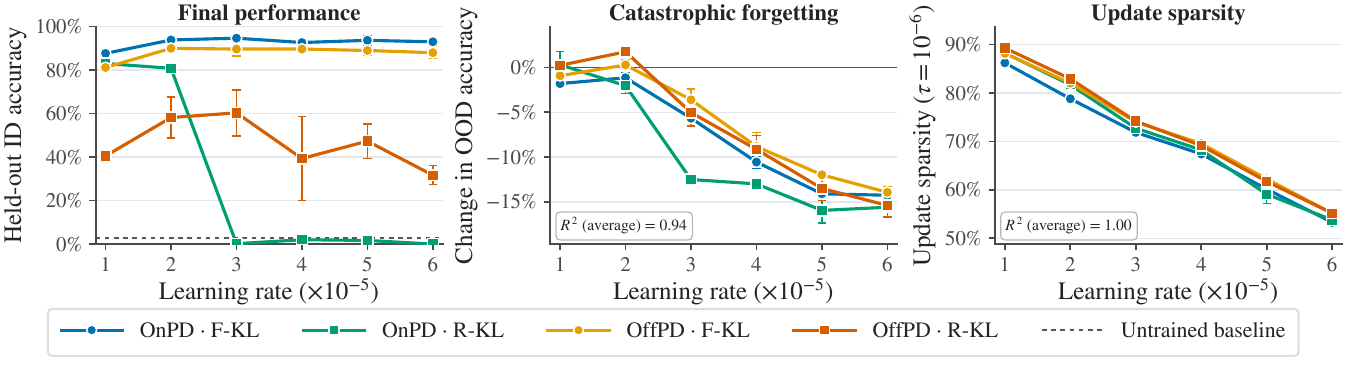}
    \vspace{-0.5em}
    \caption{\textbf{The effect of learning rate on Countdown-3.} Forgetting increases and update sparsity decreases with the learning rate under both rollout policies. Points report means over $N=3$ seeds; error bars show SEM.}
    \label{fig:learning_rate}
    \vspace{-1.5em}
\end{figure}

\begin{highlight}
Across our controlled comparisons, on-policy rollouts offer no consistent advantage in in-distribution performance, forgetting, or update sparsity. KL direction more clearly distinguishes task performance, while learning rate largely governs forgetting and sparsity.
\end{highlight}

\section{When does rollout policy actually matter?}
\label{sec:rollout_spectrum}

Our results in \Cref{sec:on_vs_offpd} indicate that when controlling for the direction of the KL divergence and the learning rate, the differences between OnPD and OffPD seem negligible, challenging the view that rollout policy strongly affects performance. This inspires a broader question: \textit{when, if at all, does the rollout policy actually matter}? A closer look at \Cref{fig:on_vs_offpd} reveals that the answer might depend on the direction of the token-level KL divergence. Indeed, in this section we carry out a controlled analysis which reveals that forward KL is robust to the changes in the rollout policy, while for reverse KL rollout policy matters a lot, with student-favoured trajectories leading to better performance.


\subsection{Logit Gradients Reveal Differing Sensitivity to the Rollout Policy}
To understand how the rollout policy affects optimisation, we analyse the token-level gradients of forward and reverse KL. Let $(z_S^\theta)_v, \, v \in \mathcal{V}$ denote the logit values produced by student model at a fixed prefix, with $\pi_S^\theta(v) = \operatorname{softmax}(z_S^\theta)_v$. Then, we can describe the parameter gradients as follows:
\begin{align}
    \nabla_\theta D_{\text{F-KL}} &= \sum_{v \in \mathcal{V}} (\pi_S^\theta(v) - \pi_T(v))\nabla_\theta (z_S^\theta)_v, \label{eq:gradient_forward_kl} \\
    \nabla_\theta D_{\text{R-KL}} &= \sum_{v \in \mathcal{V}} \pi_S^\theta(v)\left[\log\frac{\pi_S^\theta(v)}{\pi_T(v)} - D_{\text{R-KL}}\right] \nabla_\theta (z_S^\theta)_v. \label{eq:gradient_reverse_kl}
\end{align}
We provide the derivation in Appendix~\ref{app:kl-gradients}. These expressions reveal an important asymmetry. The forward-KL derivative with respect to the student logits is
$\pi_S^\theta(v)-\pi_T(v)$; it is therefore nonzero
whenever the teacher and student next-token distributions differ, and each of its coordinates lies in $[-1,1]$. Consequently, under bounded student-logit Jacobians $\nabla_\theta(z_S^\theta)$, the difference between the forward-KL updates induced by two rollout policies is bounded linearly by the total-variation distance between the prefix distributions they induce (see Appendix~\ref{app:bullshit} for the precise statement and proof). Hence, small changes in the trajectories generated by the rollout policy produce proportionally small changes in the forward-KL gradient.

For reverse KL, however, even a small rollout change can produce an arbitrarily large gradient change when the teacher and student assign very different probabilities to some tokens. Its derivative with respect to student logits is weighted by $\pi_S^\theta(v)$, so it vanishes as the student probability approaches zero, even when the teacher assigns the token substantial probability. Reverse KL may therefore struggle to recover teacher modes omitted by the student, reflecting its mode-seeking behaviour. Conversely, when the student assigns appreciable probability to a token that the teacher considers extremely unlikely, the log-ratio $\log{\pi_S^\theta(v)/\pi_T(v)}$ an become arbitrarily
large, potentially producing sharp, high-variance updates, which can potentially destabilise training. Accordingly, unlike forward KL, reverse KL admits no bound dependent solely on the rollout distance and the student-logit Jacobian. We formalise this result in Appendix~\ref{app:bullshit}.

These properties suggest that reverse KL is more sensitive to the rollout policy. While forward KL provides signal at any visited prefix where the policies disagree, reverse KL emphasises student-supported, teacher-disfavoured tokens and therefore depends more strongly on visiting the student’s own prefix distribution. We consequently expect reverse KL to benefit more from on-policy rollouts.

\subsection{Empirical Validation Using a Rollout-Policy Spectrum}

\textbf{Rollout-policy spectrum.} To empirically validate the differing sensitivity of forward and reverse KL to the rollout policy, we systematically vary the rollout policy along a student--teacher spectrum parametrised by $\lambda$. Specifically, we define a likelihood-ratio-controlled policy $\pi_\lambda$, where, for each $v\in\mathcal{V}$, $\pi_\lambda(v)=\operatorname{softmax}(z_\lambda)_v$, with
$$
(z_\lambda)_v
= \frac{1}{2}\left(\log\pi_S(v)+\log\pi_T(v)\right)
+
\frac{\lambda}{2}\left(\log\pi_S(v)-\log\pi_T(v)\right).
$$
Negative values of $\lambda$ produce teacher-favoured rollouts, positive values produce student-favoured rollouts, and $\lambda=0$ is the symmetric midpoint. In particular, $\lambda=-1$ recovers standard OffPD, while $\lambda=1$ recovers OnPD. Choosing $\lambda<-1$ biases sampling towards tokens which the teacher finds plausible but the student does not, and vice versa for $\lambda>1$. To avoid aggressively amplifying tokens that both models consider unlikely, we additionally employ log-ratio clipping and plausibility masking, similar to \citet{li2023contrastivedecodingopenendedtext}, as detailed in Appendix~\ref{app:rollout_spectrum}. \Cref{fig:rollout_spectrum_likelihood} validates that this sampling scheme allows us to vary the student and teacher likelihoods on the generated trajectories.

\begin{figure}[t]
    \centering
    \includegraphics[width=0.9\linewidth]{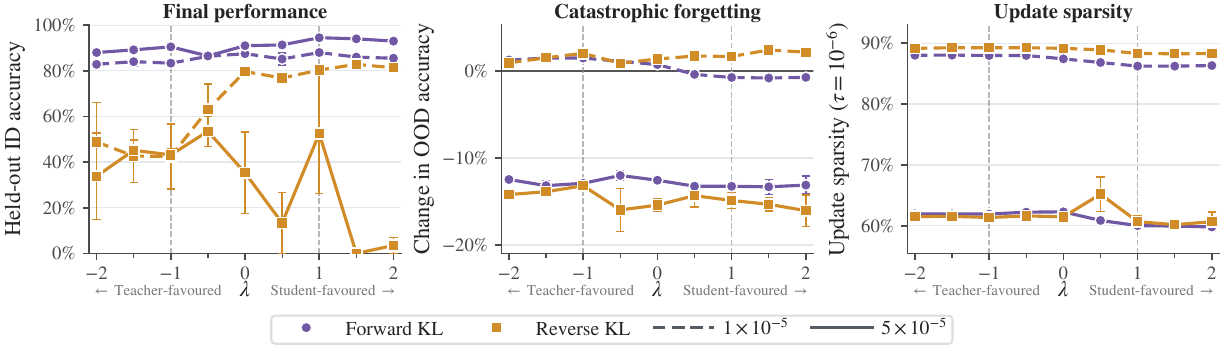}
    \vspace{-1em}
    \caption{\textbf{Performance of forward and reverse KL across the rollout-policy spectrum.} Forward KL maintains strong final performance across the spectrum, whereas reverse KL varies substantially. Rollout policy has comparatively little effect on catastrophic forgetting or update sparsity. Points show means over $N=3$ seeds; error bars denote SEM.}
    \label{fig:rollout_spectrum}
\end{figure}
\begin{wrapfigure}[12]{r}{0.28\linewidth}
\centering
\vspace{-1em}
    \includegraphics[width=\linewidth]{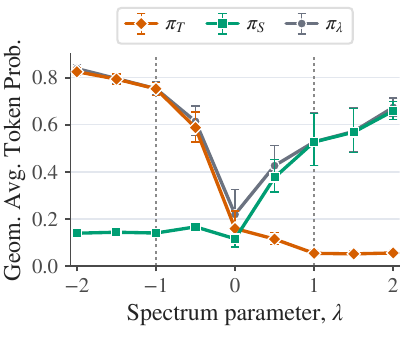}
    \vspace{-2em}
    \caption{\textbf{Likelihood in the rollout-policy spectrum}.}
    \label{fig:rollout_spectrum_likelihood}
\end{wrapfigure}
\textbf{Empirical results.} Following the setup of the previous section, we train the student on Countdown-3 for 150 optimisation steps using either forward or reverse KL objective, with trajectories sampled from $\pi_\lambda$. The results in \Cref{fig:rollout_spectrum} support the conclusions from our gradient analysis. Forward KL is remarkably robust to the rollout policy: held-out accuracy remains above 80\% across the entire spectrum, varying by only 5.2 percentage points at a learning rate of $1\times10^{-5}$. 
Remarkably, OnPD with forward KL can learn effectively despite passing through a regime of largely incoherent, high-entropy rollouts (Appendix~\ref{app:onpd-fkl-entropy-collapse}). Similar final performance under OnPD and OffPD may therefore conceal markedly different optimisation paths.
Reverse KL is substantially more sensitive, exhibiting large changes in mean performance across $\lambda$ and considerable variance across seeds for several settings. At the lower learning rate, reverse KL benefits strongly from student-favoured rollouts ($\lambda>0$), while performance deteriorates for teacher-favoured rollouts. At the higher learning rate, however, reverse KL remains unstable and occasionally collapses.

We further examine catastrophic forgetting and update sparsity over the same rollout policy spectrum. Middle and right panels of \Cref{fig:rollout_spectrum} show that varying the rollout policy has a comparatively modest effect on mean OOD performance and sparsity. Both are influenced much more strongly by the learning rate, consistent with our earlier findings.

\begin{highlight}
Gradient analysis and empirical experiments reveal that forward KL is remarkably robust to the rollout policy, while reverse KL favours student-generated rollouts. 
Catastrophic forgetting and update sparsity are only marginally affected by rollout policy.
\end{highlight}

\subsection{Are there any other differences induced by the direction of KL divergence?}
So far, we have seen that the effect of rollout policy on task performance depends strongly on KL direction. However, pass@1 accuracy alone does not fully capture what a model has learned. We therefore examine output coverage under repeated sampling, generalisation to harder Countdown variants, and performance under subsequent RLVR.



\textbf{Forward KL yields larger gains from repeated sampling.} Improvements in pass@1 do not necessarily translate into improvements at higher pass@$k$~\citep{yue2025does}. \Cref{fig:countdown_spectrum} (left pair) compares pass@1 and pass@10 on Countdown-3 along the rollout-policy spectrum with learning $1\times10^{-5}$ (LR $5\times10^{-5}$ shown \Cref{fig:app:countdown_spectrum_lr5e5}). Reverse KL yields smaller gains at comparable pass@1: for $\lambda>0$ (OnPD), both KL directions have similar pass@1, but forward KL leads to significantly higher pass@10. Results on Qwen2.5 show the same effect across the rollout-policy spectrum (\Cref{fig:app:qwen_countdown_spectrum}). This suggests that KL direction influences output coverage more than rollout policy.

\textbf{More on-policy rollouts generalise better to harder tasks.} We evaluate the same Countdown-3 checkpoints without further training on the harder task Countdown-4E, which uses four operands drawn from 1--10. \Cref{fig:countdown_spectrum} (right pair) shows pass@1 and pass@10 along the rollout-policy spectrum with learning $1\times10^{-5}$ (LR $5\times10^{-5}$ shown \Cref{fig:app:countdown_spectrum_lr5e5}). For both KL directions, performance increases gradually across the spectrum, with more on-policy rollouts ($\lambda > 0$) consistently yielding $10-15\%$ higher pass@k accuracy than off-policy rollouts. Results on Qwen2.5 with the same LR show consistent trends. Hence, by investigating generalisation to harder tasks, we identify the first setting where using on-policy rollouts leads to a significant performance improvement compared to off-policy rollouts. These findings are consistent with prior works studying generalisation in the context of SFT and RL \citep{yuan2026f, zhang_towards_2026, ming_one-token_2026}.

\begin{figure}[t]
\centering
    \includegraphics[width=\linewidth]{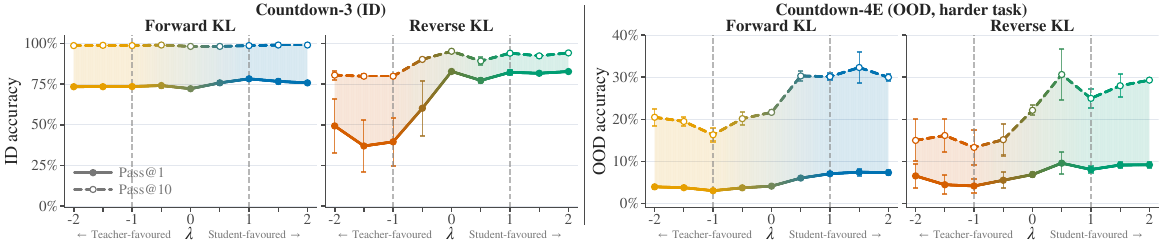}
    \vspace{-2em}
    \caption{\textbf{Output coverage and generalisation along the rollout-policy spectrum.} While in-distribution forward KL leads to better output coverage, when generalising to harder tasks we can see a clear benefit of using on-policy data. Error bars show SEM across three seeds. The easier Countdown-3 uses $T=1$, while the more difficult Countdown-4E uses $T=0.5$.}
    \label{fig:countdown_spectrum}
    \vspace{-2em}
\end{figure}


\begin{wrapfigure}{r}{0.35\linewidth}
\vspace{-0.5cm}
    \centering
    \includegraphics[width=\linewidth]{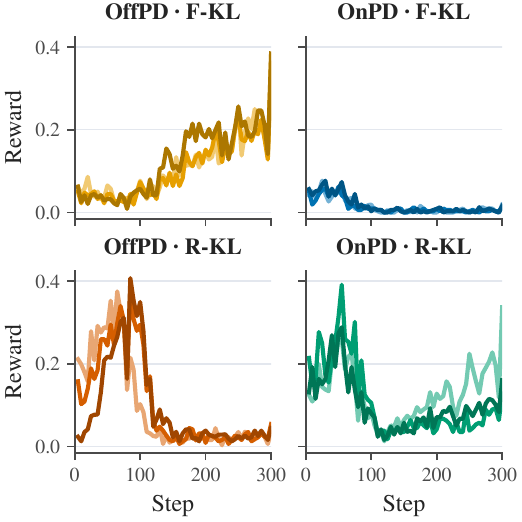}
    \vspace{-2em}
    \caption{\textbf{RLVR rewards on Countdown-4} when starting from distillation checkpoints trained with LR $1\times10^{-5}$; each curve corresponds to a distillation seed.}
    \label{fig:countdown4_downstream_rl_1e5}
\vspace{-0.5cm}
\end{wrapfigure}

\textbf{Consequences for subsequent RLVR.} Which configurations, then, provide better starting points for further RLVR? We train the Countdown-3 distillation checkpoints on Countdown-4 for 300 RLVR steps, using three seeds per distillation configuration and an otherwise fixed RL setup. \Cref{fig:countdown4_downstream_rl_1e5} shows that reverse-KL checkpoints distilled at a learning rate of $1\times10^{-5}$ improve quickly---including OffPD checkpoints with initially low accuracy---but later undergo reward collapse. In contrast, OffPD checkpoints distilled with forward KL at $1\times10^{-5}$ or reverse KL at $5\times10^{-5}$ (Appendix~\Cref{fig:app:countdown4_downstream_rl_5e5}) improve steadily and reach the highest performance at the end of training, despite starting near $0\%$ accuracy. 
The strongest sustained RLVR performance therefore comes from off-policy checkpoints, despite their lower initial Countdown-4 accuracy. The initial generalisation advantage of on-policy checkpoints does not translate into a reliable advantage after RLVR. These findings are consistent with \citet{zhang_good_2026}, who show that the SFT checkpoints with the strongest initial performance are not necessarily the best starting points for a subsequent RLVR stage, suggesting that additional research is required to determine how to best combine post-training methods into multi-step pipelines.


\textbf{Teacher-style transfer.}
Beyond output coverage and generalisation, we proceed to study one more aspect of training: the transfer of incidental teacher behaviours. We instruct the teacher to reason in Spanish and measure how often the student subsequently produces Spanish responses (\Cref{app:spanish}). Surprisingly, OnPD with reverse KL largely preserves the student’s English response style, whereas the other configurations exhibit near-complete transfer to Spanish. We hypothesise that, on student-generated English prefixes, the teacher still supports English continuations despite its Spanish instruction. The mode-seeking reverse-KL objective can therefore match this English mode without requiring the student to cover Spanish alternatives. These results suggest that combining OnPD with reverse KL may help suppress incidental style transfer.

\begin{highlight}
    While forward KL improves in-distribution output coverage, when generalising to harder tasks we observe consistent benefits from using on-policy data, across both forward and reverse KL. However, this advantage of OnPD does not reliably persist after RLVR.
\end{highlight}

\section{Do these findings generalise beyond our experimental design?}

We test whether the findings in the preceding sections depend on three choices in our experimental design: computing KL over the full vocabulary rather than sampled tokens, applying gradient clipping, and studying tasks with relatively short rollouts.


\textbf{Sampled rather than full-vocabulary KL.}
Our main experiments compute the KL divergence over the full vocabulary, since this has been shown to improve training stability \citep{deepseekai2026deepseekv4highlyefficientmilliontoken}. Several recent works instead use a more memory-efficient sampled-KL estimator \citep{lu2025onpolicydistillation,li_rethinking_2026}. We therefore repeat our experiments using sampled KL. As in our main analysis, we disentangle the policy used to generate the rollout from the distribution used to sample the token-level KL gradient, evaluating all four combinations of rollout policy and KL direction. In the results, forward KL remains robust to rollout policy, reverse KL remains substantially more sensitive and brittle, and learning rate continues to govern forgetting and update sparsity (Appendix~\ref{app:sampled_kl}).

\textbf{Removing gradient clipping.} Our main experiments clip the global gradient norm at $1.0$, and the unclipped norm exceeds this threshold throughout training (\Cref{fig:app:gradient_norms}). Clipping could therefore suppress meaningful differences in gradient magnitudes between rollout policies. Disabling clipping exposes a stronger difference between KL directions but no consistent OnPD--OffPD gap. Reverse-KL performance deteriorates substantially, whereas forward KL remains effective at larger learning rates. Forgetting and sparsity remain primarily determined by learning rate (Appendix~\ref{app:exp:gradient_clipping}). Thus, clipping stabilises reverse KL but does not explain the limited effect of rollout policy.

\textbf{Training on datasets requiring longer rollouts.} Our main experiments cover medical, scientific, and arithmetic reasoning, but all three have average teacher-response lengths of 94--135 tokens (\Cref{tab:trained-teacher-performance}). The benefits of on-policy distillation may become more apparent over longer trajectories, where small differences between the student and teacher policies can compound across many generated tokens. We therefore repeat the comparison using Qwen2.5 models on Numina--MATH, where teacher responses average 622 tokens. At the lower learning rate, OnPD with reverse KL achieves the highest MATH-500 accuracy among the evaluated configurations, suggesting that on-policy rollouts may help in harder tasks requiring longer reasoning. Nevertheless, forward KL remains more robust than reverse KL, while forgetting and sparsity remain governed by learning rate (Appendix~\ref{app:math_qwen}). Because this experiment contains one run per condition, we treat the OnPD advantage as suggestive.

\begin{highlight}
Our main conclusions are not explained by full-vocabulary KL, gradient clipping, or short rollouts: forward KL remains more robust to rollout policy, while learning rate remains the strongest predictor of forgetting and parameter-update sparsity.
\end{highlight}

\section{Discussion, Limitations and Future Work}

\textbf{Discussion.} In conclusion, our controlled comparison offers a nuanced view of the differences between on-policy and off-policy learning in the context of distillation. While on-policy learning can improve generalisation to harder tasks, and potentially reduce incidental teacher-style transfer, learning rate and KL direction explain substantially more of the observed variation. Given its lower computational cost, we encourage future work on on-policy distillation to include OffPD as a standard baseline. More broadly, our results suggest that it is difficult to attribute most of the observed differences between SFT and RL to rollout policy alone, motivating closer study of other factors such as the learning objective, reward signal, supervision density, and optimisation procedure.

\textbf{Limitations and Future Work.} Our controlled experimental setting enables systematic comparisons across rollout policies, KL directions, and learning rates, but necessarily focuses on a bounded regime: student models of at most 1.5B parameters, and (consequently) tasks requiring reasoning traces of at most 2,000 tokens. 
Extending this analysis to larger models, and tasks requiring longer rollouts represents a promising direction for future work. Further, while in our study we keep the teacher model fixed, in future work we would like to investigate whether differences between OnPD and OffPD emerge under specific student--teacher combinations.


\subsection*{AI use statement}

In this work, we used generative AI tools to obtain demonstrations for solving reasoning tasks, to help formulate mathematical claims, to provide guidelines for proving those claims and assist in the writing of proofs, to implement and maintain the code for running experiments, and to prepare plots for the figures. We have not used generative AI tools to help develop theoretical models or conceptual frameworks, to propose or refine hypotheses, to design our research methodology, to support qualitative and thematic data analysis, to suggest a structure for the paper, or to interpret results, and the generation of synthetic datasets and assistance with translation are not applicable to this work. Additionally, we have used generative AI tools to edit individual paragraphs in the paper to improve readability. We have reviewed all AI-assisted work: all mathematical claims and their proofs were checked in full by the authors, AI-assisted code was verified and tested for correctness by two authors, and all AI-edited text was reviewed by all the authors to confirm that it accurately reflects our intended meaning and claims. The plotted values were verified with the matching plots from wandb. We take responsibility for the final content of this work, including text, claims or artifacts produced with the aid of generative AI.

\subsection*{Reproducibility statement}

We provide complete definitions of OnPD, OffPD, forward and reverse KL, and the rollout-policy spectrum, including derivations and implementation details for sampled KL estimation, plausibility masking, and log-ratio clipping. The appendix documents the model checkpoints, datasets, prompts, data splits, generation settings, optimisation hyperparameters, learning rates, gradient-clipping choices, training durations, and evaluation procedures used in each experiment. Unless otherwise stated, results are averaged over three random seeds and reported with the standard error of the mean; experiments containing fewer runs, including the long-rollout ablation, are explicitly identified. All matched comparisons use the same training and evaluation pipeline, differing only in the factors under investigation.


\subsubsection*{Acknowledgments}
We would like to thank Alicia Curth and Usman Anwar for their feedback on the earlier versions of this work. JP's PhD studentship is funded by AstraZeneca, while AB's studentship is funded by Eedi. This work was supported by Azure sponsorship credits granted by Microsoft’s AI for Good Research Lab.

\bibliography{references}

@misc{agarwal_-policy_2024,
  title = {On-{Policy} {Distillation} of {Language} {Models}: {Learning} from {Self}-{Generated} {Mistakes}},
  shorttitle = {On-{Policy} {Distillation} of {Language} {Models}},
  url = {http://arxiv.org/abs/2306.13649},
  doi = {10.48550/arXiv.2306.13649},
  urldate = {2026-04-02},
  publisher = {arXiv},
  author = {Agarwal, Rishabh and Vieillard, Nino and Zhou, Yongchao and Stanczyk, Piotr and Ramos, Sabela and Geist, Matthieu and Bachem, Olivier},
  year = {2023},
  journal = {International Conference on Learning Representations},
}

@article{lu2025onpolicydistillation,
  author = {Kevin Lu and Thinking Machines Lab},
  title = {On-Policy Distillation},
  journal = {Thinking Machines Lab: Connectionism},
  year = {2025},
  note = {https://thinkingmachines.ai/blog/on-policy-distillation},
  doi = {10.64434/tml.20251026},
}

@misc{hendrycks2021measuringmathematicalproblemsolving,
  title = {Measuring Mathematical Problem Solving With the MATH Dataset},
  author = {Dan Hendrycks and Collin Burns and Saurav Kadavath and Akul Arora and Steven Basart and Eric Tang and Dawn Song and Jacob Steinhardt},
  year = {2021},
  eprint = {2103.03874},
  archiveprefix = {arXiv},
  primaryclass = {cs.LG},
  url = {https://arxiv.org/abs/2103.03874},
  journal = {NeurIPS Datasets and Benchmarks},
}

@misc{yang2024qwen25mathtechnicalreportmathematical,
  title = {Qwen2.5-Math Technical Report: Toward Mathematical Expert Model via Self-Improvement},
  author = {An Yang and Beichen Zhang and Binyuan Hui and Bofei Gao and Bowen Yu and Chengpeng Li and Dayiheng Liu and Jianhong Tu and Jingren Zhou and Junyang Lin and others},
  year = {2024},
  eprint = {2409.12122},
  archiveprefix = {arXiv},
  primaryclass = {cs.CL},
  url = {https://arxiv.org/abs/2409.12122},
  journal = {arXiv.org},
  doi = {10.48550/arXiv.2409.12122},
}

@misc{li2023contrastivedecodingopenendedtext,
  title = {Contrastive Decoding: Open-ended Text Generation as Optimization},
  author = {Xiang Lisa Li and Ari Holtzman and Daniel Fried and Percy Liang and Jason Eisner and Tatsunori Hashimoto and Luke Zettlemoyer and Mike Lewis},
  year = {2023},
  eprint = {2210.15097},
  archiveprefix = {arXiv},
  primaryclass = {cs.CL},
  url = {https://arxiv.org/abs/2210.15097},
  journal = {Proceedings of the 61st Annual Meeting of the Association for Computational Linguistics (Volume 1: Long Papers)},
  pages = {12286-12312},
  doi = {10.18653/v1/2023.acl-long.687},
  publisher = {Association for Computational Linguistics},
}

@misc{guha_openthoughts_2025,
  title = {{OpenThoughts}: {Data} {Recipes} for {Reasoning} {Models}},
  shorttitle = {{OpenThoughts}},
  url = {http://arxiv.org/abs/2506.04178},
  doi = {10.48550/arXiv.2506.04178},
  urldate = {2026-08-26},
  publisher = {arXiv},
  author = {Guha, Etash and Marten, Ryan and Keh, Sedrick and Raoof, Negin and Smyrnis, Georgios and Bansal, Hritik and Nezhurina, Marianna and Mercat, Jean and Vu, Trung and Sprague, Zayne and others},
  year = {2025},
  journal = {arXiv.org},
}

@misc{hinton_distilling_2015,
  title = {Distilling the {Knowledge} in a {Neural} {Network}},
  url = {http://arxiv.org/abs/1503.02531},
  doi = {10.48550/arXiv.1503.02531},
  urldate = {2026-08-26},
  publisher = {arXiv},
  author = {Hinton, Geoffrey and Vinyals, Oriol and Dean, Jeff},
  year = {2015},
  journal = {arXiv.org},
}

@inproceedings{kwon2023efficient,
  title = {Efficient Memory Management for Large Language Model Serving with PagedAttention},
  author = {Woosuk Kwon and Zhuohan Li and Siyuan Zhuang and Ying Sheng and Lianmin Zheng and Cody Hao Yu and Joseph E. Gonzalez and Hao Zhang and Ion Stoica},
  booktitle = {Symposium on Operating Systems Principles},
  year = {2023},
  journal = {Symposium on Operating Systems Principles},
  pages = {611-626},
  doi = {10.1145/3600006.3613165},
  publisher = {ACM},
}

@misc{numina_math_datasets,
  author = {Jia LI and Edward Beeching and Lewis Tunstall and Ben Lipkin and Roman Soletskyi and Shengyi Costa Huang and Kashif Rasul and Longhui Yu and Albert Jiang and Ziju Shen and Zihan Qin and Bin Dong and Li Zhou and Yann Fleureau and Guillaume Lample and Stanislas Polu},
  title = {NuminaMath},
  year = {2024},
  publisher = {Numina},
  journal = {Hugging Face repository},
  howpublished = {\url{[https://huggingface.co/AI-MO/NuminaMath-CoT](https://github.com/project-numina/aimo-progress-prize/blob/main/report/numina_dataset.pdf)}},
}

@article{team2025qwen3,
  title = {Qwen3 Technical Report},
  author = {An Yang and Anfeng Li and Baosong Yang and Beichen Zhang and Binyuan Hui and Bo Zheng and Bowen Yu and Chang Gao and Chengen Huang and Chenxu Lv and others},
  journal = {arXiv},
  year = {2025},
}

@misc{shao_deepseekmath_2024,
  title = {{DeepSeekMath}: {Pushing} the {Limits} of {Mathematical} {Reasoning} in {Open} {Language} {Models}},
  shorttitle = {{DeepSeekMath}},
  url = {http://arxiv.org/abs/2402.03300},
  doi = {10.48550/arXiv.2402.03300},
  urldate = {2026-08-26},
  publisher = {arXiv},
  author = {Shao, Zhihong and Wang, Peiyi and Zhu, Qihao and Xu, Runxin and Song, Junxiao and Bi, Xiao and Zhang, Haowei and Zhang, Mingchuan and Li, Y. K. and Wu, Y. and others},
  year = {2024},
  journal = {arXiv.org},
}

@article{deepseek-ai_deepseek-r1_2025,
  title = {{DeepSeek}-{R1}: {Incentivizing} {Reasoning} {Capability} in {LLMs} via {Reinforcement} {Learning}},
  volume = {645},
  issn = {0028-0836, 1476-4687},
  shorttitle = {{DeepSeek}-{R1}},
  url = {http://arxiv.org/abs/2501.12948},
  doi = {10.48550/arXiv.2501.12948},
  number = {8081},
  urldate = {2026-08-26},
  journal = {arXiv.org},
  author = {DeepSeek-AI},
  year = {2025},
}

@misc{yu_dapo_2025,
  title = {{DAPO}: {An} {Open}-{Source} {LLM} {Reinforcement} {Learning} {System} at {Scale}},
  shorttitle = {{DAPO}},
  url = {http://arxiv.org/abs/2503.14476},
  doi = {10.52202/085713-3775},
  urldate = {2026-08-26},
  publisher = {Neural Information Processing Systems Foundation, Inc. (NeurIPS)},
  author = {Yu, Qiying and Zhang, Zheng and Zhu, Ruofei and Yuan, Yufeng and Zuo, Xiaochen and Yue, Yu and Dai, Weinan and Fan, Tiantian and Liu, Gaohong and Liu, Lingjun and others},
  year = {2025},
  journal = {Advances in Neural Information Processing Systems 38},
  pages = {125532-125554},
}

@misc{team_kimi_2026,
  title = {Kimi {K2}: {Open} {Agentic} {Intelligence}},
  shorttitle = {Kimi {K2}},
  url = {http://arxiv.org/abs/2507.20534},
  doi = {10.48550/arXiv.2507.20534},
  urldate = {2026-08-26},
  publisher = {arXiv},
  author = {Bai, Kimi Team Yifan and Bao, Yiping and Charles, Y. and Chen, Cheng and Chen, Guanduo and Chen, Hai-Ting and Chen, Hua-Rong and Chen, Jiahao and Chen, Ning-Xin and Chen, Ruijue and others},
  year = {2025},
  journal = {arXiv},
}

@misc{wu2025medreasonelicitingfactualmedical,
  title = {MedReason: Eliciting Factual Medical Reasoning Steps in LLMs via Knowledge Graphs},
  author = {Juncheng Wu and Wenlong Deng and Xingxuan Li and Sheng Liu and Taomian Mi and Yifan Peng and Ziyang Xu and Yi Liu and Hyunjin Cho and Chang-In Choi and others},
  year = {2025},
  eprint = {2504.00993},
  archiveprefix = {arXiv},
  primaryclass = {cs.CL},
  url = {https://arxiv.org/abs/2504.00993},
  journal = {arXiv.org},
  doi = {10.48550/arXiv.2504.00993},
}

@article{feng2024sciknoweval,
  title = {Sciknoweval: Evaluating multi-level scientific knowledge of large language models},
  author = {Feng, Kehua and Ding, Keyan and Wang, Weijie and Zhuang, Xiang and Wang, Zeyuan and Qin, Ming and Zhao, Yu and Yao, Jianhua and Zhang, Qiang and Chen, Huajun},
  journal = {arXiv.org},
  year = {2024},
  doi = {10.48550/arXiv.2406.09098},
}

@misc{tinyzero,
  author = {Jiayi Pan and Junjie Zhang and Xingyao Wang and Lifan Yuan and Hao Peng and Alane Suhr},
  title = {TinyZero},
  howpublished = {https://github.com/Jiayi-Pan/TinyZero},
  note = {Accessed: 2025-01-24},
  year = {2025},
}

@misc{qwen2025qwen25technicalreport,
  title = {Qwen2.5 Technical Report},
  author = {Yang, Qwen An and Yang, Baosong and Zhang, Beichen and Hui, Binyuan and Zheng, Bo and Yu, Bo-Wen and Li, Chengyuan and Liu, Dayiheng and Huang, Fei and Dong, Guanting and others},
  year = {2025},
  eprint = {2412.15115},
  archiveprefix = {arXiv},
  primaryclass = {cs.CL},
  url = {https://arxiv.org/abs/2412.15115},
  journal = {arXiv.org},
  doi = {10.48550/arXiv.2412.15115},
}

@misc{eval-harness,
  author = {Gao, Leo and Tow, Jonathan and Abbasi, Baber and Biderman, Stella and Black, Sid and DiPofi, Anthony and Foster, Charles and Golding, Laurence and Hsu, Jeffrey and Le Noac'h, Alain and others},
  title = {The Language Model Evaluation Harness},
  month = {07},
  year = {2024},
  publisher = {Zenodo},
  version = {v0.4.3},
  doi = {10.5281/zenodo.12608602},
  url = {https://zenodo.org/records/12608602},
}

@inproceedings{lin-etal-2022-truthfulqa,
  title = {{T}ruthful{QA}: Measuring How Models Mimic Human Falsehoods},
  author = {Lin, Stephanie  and Hilton, Jacob  and Evans, Owain},
  editor = {Muresan, Smaranda  and Nakov, Preslav  and Villavicencio, Aline},
  booktitle = {Annual Meeting of the Association for Computational Linguistics},
  year = {2021},
  journal = {Annual Meeting of the Association for Computational Linguistics},
  pages = {3214-3252},
  doi = {10.18653/v1/2022.acl-long.229},
  publisher = {Association for Computational Linguistics},
}

@misc{zhou2023instructionfollowingevaluationlargelanguage,
  title = {Instruction-Following Evaluation for Large Language Models},
  author = {Jeffrey Zhou and Tianjian Lu and Swaroop Mishra and Siddhartha Brahma and Sujoy Basu and Yi Luan and Denny Zhou and Le Hou},
  year = {2023},
  eprint = {2311.07911},
  archiveprefix = {arXiv},
  primaryclass = {cs.CL},
  url = {https://arxiv.org/abs/2311.07911},
  journal = {arXiv.org},
  doi = {10.48550/arXiv.2311.07911},
}

@misc{chen2021evaluating,
  title = {Evaluating Large Language Models Trained on Code},
  author = {Mark Chen and Jerry Tworek and Heewoo Jun and Qiming Yuan and Henrique Ponde de Oliveira Pinto and Jared Kaplan and Harri Edwards and Yuri Burda and Nicholas Joseph and Greg Brockman and others},
  year = {2021},
  eprint = {2107.03374},
  archiveprefix = {arXiv},
  primaryclass = {cs.LG},
  journal = {arXiv.org},
}

@misc{paech2024eqbenchemotionalintelligencebenchmark,
  title = {EQ-Bench: An Emotional Intelligence Benchmark for Large Language Models},
  author = {Samuel J. Paech},
  year = {2024},
  eprint = {2312.06281},
  archiveprefix = {arXiv},
  primaryclass = {cs.CL},
  url = {https://arxiv.org/abs/2312.06281},
  journal = {arXiv.org},
  doi = {10.48550/arXiv.2312.06281},
}

@inproceedings{parrish-etal-2022-bbq,
  title = {{BBQ}: A hand-built bias benchmark for question answering},
  author = {Parrish, Alicia  and Chen, Angelica  and Nangia, Nikita  and Padmakumar, Vishakh  and Phang, Jason  and Thompson, Jana  and Htut, Phu Mon  and Bowman, Samuel R.},
  editor = {Muresan, Smaranda  and Nakov, Preslav  and Villavicencio, Aline},
  booktitle = {Findings},
  year = {2021},
  journal = {Findings},
  pages = {2086-2105},
  doi = {10.18653/v1/2022.findings-acl.165},
  publisher = {Association for Computational Linguistics},
}

@inproceedings{hartvigsen-etal-2022-toxigen,
  title = {{T}oxi{G}en: A Large-Scale Machine-Generated Dataset for Adversarial and Implicit Hate Speech Detection},
  author = {Hartvigsen, Thomas  and Gabriel, Saadia  and Palangi, Hamid  and Sap, Maarten  and Ray, Dipankar  and Kamar, Ece},
  editor = {Muresan, Smaranda  and Nakov, Preslav  and Villavicencio, Aline},
  booktitle = {Proceedings of the 60th Annual Meeting of the Association for Computational Linguistics (Volume 1: Long Papers)},
  year = {2022},
  journal = {Proceedings of the 60th Annual Meeting of the Association for Computational Linguistics (Volume 1: Long Papers)},
  pages = {3309-3326},
  doi = {10.18653/v1/2022.acl-long.234},
  publisher = {Association for Computational Linguistics},
}

@inproceedings{kim-rush-2016-sequence,
  title = {Sequence-Level Knowledge Distillation},
  author = {Kim, Yoon  and Rush, Alexander M.},
  editor = {Su, Jian  and Duh, Kevin  and Carreras, Xavier},
  booktitle = {Conference on Empirical Methods in Natural Language Processing},
  year = {2016},
  journal = {Conference on Empirical Methods in Natural Language Processing},
  pages = {1317-1327},
  doi = {10.18653/v1/D16-1139},
  publisher = {Association for Computational Linguistics},
}

@misc{song2026surveyonpolicydistillationlarge,
  title = {A Survey of On-Policy Distillation for Large Language Models},
  author = {Mingyang Song and Mao Zheng},
  year = {2026},
  eprint = {2604.00626},
  archiveprefix = {arXiv},
  primaryclass = {cs.LG},
  url = {https://arxiv.org/abs/2604.00626},
  journal = {arXiv.org},
  doi = {10.48550/arXiv.2604.00626},
}

@misc{sanh2020distilbertdistilledversionbert,
  title = {DistilBERT, a distilled version of BERT: smaller, faster, cheaper and lighter},
  author = {Victor Sanh and Lysandre Debut and Julien Chaumond and Thomas Wolf},
  year = {2020},
  eprint = {1910.01108},
  archiveprefix = {arXiv},
  primaryclass = {cs.CL},
  url = {https://arxiv.org/abs/1910.01108},
  journal = {arXiv.org},
}

@misc{abdin2024phi3technicalreporthighly,
  title = {Phi-3 Technical Report: A Highly Capable Language Model Locally on Your Phone},
  author = {Abdin, Marah and Jacobs, Sam Adé and Awan, A. and Aneja, J. and Awadallah, Ahmed and Awadalla, H. and Bach, Nguyen and Bahree, Amit and Bakhtiari, Arash and Behl, Harkirat Singh and others},
  year = {2024},
  eprint = {2404.14219},
  archiveprefix = {arXiv},
  primaryclass = {cs.CL},
  url = {https://arxiv.org/abs/2404.14219},
  journal = {arXiv.org},
  doi = {10.48550/arXiv.2404.14219},
}

@inproceedings{NEURIPS2024_ad236edc,
  author = {Wang, Yubo and Ma, Xueguang and Zhang, Ge and Ni, Yuansheng and Chandra, Abhranil and Guo, Shiguang and Ren, Weiming and Arulraj, Aaran and He, Xuan and Jiang, Ziyan and others},
  booktitle = {Advances in Neural Information Processing Systems 37},
  doi = {10.52202/079017-3018},
  editor = {A. Globerson and L. Mackey and D. Belgrave and A. Fan and U. Paquet and J. Tomczak and C. Zhang},
  pages = {95266--95290},
  publisher = {Neural Information Processing Systems Foundation, Inc. (NeurIPS)},
  title = {MMLU-Pro: A More Robust and Challenging Multi-Task Language Understanding Benchmark},
  url = {https://proceedings.neurips.cc/paper_files/paper/2024/file/ad236edc564f3e3156e1b2feafb99a24-Paper-Datasets_and_Benchmarks_Track.pdf},
  volume = {37},
  year = {2024},
  journal = {Advances in Neural Information Processing Systems 37},
}

@article{kingma2014adam,
  title = {Adam: A method for stochastic optimization},
  author = {Kingma, Diederik P and Ba, Jimmy},
  journal = {International Conference on Learning Representations},
  year = {2014},
}

@article{reddi2019convergence,
  title = {On the convergence of adam and beyond},
  author = {Reddi, Sashank J and Kale, Satyen and Kumar, Sanjiv},
  journal = {International Conference on Learning Representations},
  year = {2018},
}

@misc{team_mimo-v2-flash_2026,
  title = {{MiMo}-{V2}-{Flash} {Technical} {Report}},
  url = {http://arxiv.org/abs/2601.02780},
  doi = {10.48550/arXiv.2601.02780},
  urldate = {2026-08-26},
  publisher = {arXiv},
  author = {Xiao, Xiao-Yu and Xia, Bing and Yang, Bo and Gao, Bofei and Shen, Bowen and Zhang, Chen and He, Chenhong and Lou, Chiheng and Luo, Fu-Li and Wang, Gang and others},
  year = {2026},
  journal = {arXiv.org},
}

@misc{shenfeld_rls_2025,
  title = {{RL}'s {Razor}: {Why} {Online} {Reinforcement} {Learning} {Forgets} {Less}},
  shorttitle = {{RL}'s {Razor}},
  url = {https://arxiv.org/abs/2509.04259v1},
  language = {en},
  urldate = {2026-04-02},
  journal = {arXiv.org},
  author = {Shenfeld, Idan and Pari, Jyothish and Agrawal, Pulkit},
  year = {2025},
  doi = {10.48550/arXiv.2509.04259},
}

@misc{chen_retaining_2025,
  title = {Retaining by {Doing}: {The} {Role} of {On}-{Policy} {Data} in {Mitigating} {Forgetting}},
  shorttitle = {Retaining by {Doing}},
  url = {https://arxiv.org/abs/2510.18874v2},
  language = {en},
  urldate = {2026-06-01},
  journal = {arXiv.org},
  author = {Chen, Howard and Razin, Noam and Narasimhan, Karthik and Chen, Danqi},
  year = {2025},
  doi = {10.48550/arXiv.2510.18874},
}

@misc{mukherjee_reinforcement_2025,
  title = {Reinforcement {Learning} {Finetunes} {Small} {Subnetworks} in {Large} {Language} {Models}},
  url = {https://arxiv.org/abs/2505.11711v2},
  language = {en},
  urldate = {2026-05-12},
  journal = {Advances in Neural Information Processing Systems 38},
  author = {Mukherjee, Sagnik and Yuan, Lifan and Hakkani-Tur, Dilek and Peng, Hao},
  year = {2025},
  pages = {146434-146453},
  doi = {10.52202/085713-4399},
  publisher = {Neural Information Processing Systems Foundation, Inc. (NeurIPS)},
}

@misc{chu_sft_2025,
  title = {{SFT} {Memorizes}, {RL} {Generalizes}: {A} {Comparative} {Study} of {Foundation} {Model} {Post}-training},
  shorttitle = {{SFT} {Memorizes}, {RL} {Generalizes}},
  url = {http://arxiv.org/abs/2501.17161},
  doi = {10.48550/arXiv.2501.17161},
  urldate = {2026-04-02},
  publisher = {arXiv},
  author = {Chu, Tianzhe and Zhai, Yuexiang and Yang, Jihan and Tong, Shengbang and Xie, Saining and Schuurmans, Dale and Le, Quoc V. and Levine, Sergey and Ma, Yi},
  year = {2025},
  journal = {International Conference on Machine Learning},
}

@misc{shenfeld_self-distillation_2026,
  title = {Self-{Distillation} {Enables} {Continual} {Learning}},
  url = {http://arxiv.org/abs/2601.19897},
  doi = {10.48550/arXiv.2601.19897},
  urldate = {2026-03-02},
  publisher = {arXiv},
  author = {Shenfeld, Idan and Damani, Mehul and Hübotter, Jonas and Agrawal, Pulkit},
  year = {2026},
  journal = {arXiv.org},
}

@misc{zhang_towards_2026,
  title = {Towards {On}-{Policy} {SFT}: {Distribution} {Discriminant} {Theory} and its {Applications} in {LLM} {Training}},
  shorttitle = {Towards {On}-{Policy} {SFT}},
  url = {http://arxiv.org/abs/2602.12222},
  doi = {10.48550/arXiv.2602.12222},
  urldate = {2026-07-13},
  publisher = {arXiv},
  author = {Zhang, Miaosen and Liu, Yishan and Lin, Shuxia and Yang, Xu and Dai, Qi and Luo, Chong and Jiang, Weihao and Hou, Peng and Zeng, Anxiang and Geng, Xin and others},
  year = {2026},
  journal = {arXiv.org},
}

@misc{ming_one-token_2026,
  title = {One-{Token} {Rollout}: {Guiding} {Supervised} {Fine}-{Tuning} of {LLMs} with {Policy} {Gradient}},
  shorttitle = {One-{Token} {Rollout}},
  url = {http://arxiv.org/abs/2509.26313},
  doi = {10.48550/arXiv.2509.26313},
  urldate = {2026-07-13},
  publisher = {arXiv},
  author = {Ming, Rui and Wu, Haoyuan and Hu, Shoubo and He, Zhuolun and Yu, Bei},
  year = {2025},
  journal = {arXiv.org},
}

@misc{li_rethinking_2026,
  title = {Rethinking {On}-{Policy} {Distillation} of {Large} {Language} {Models}: {Phenomenology}, {Mechanism}, and {Recipe}},
  copyright = {Creative Commons Attribution 4.0 International},
  shorttitle = {Rethinking {On}-{Policy} {Distillation} of {Large} {Language} {Models}},
  url = {https://arxiv.org/abs/2604.13016},
  doi = {10.48550/ARXIV.2604.13016},
  language = {en},
  urldate = {2026-06-11},
  publisher = {arXiv},
  author = {Li, Yaxuan and Zuo, Yuxin and He, Bingxiang and Zhang, Jinqian and Xiao, Chaojun and Qian, Cheng and Yu, Tianyu and Gao, Huan-ang and Yang, Wenkai and Liu, Zhiyuan and others},
  year = {2026},
  note = {Version Number: 2},
  journal = {arXiv.org},
}

@misc{tang_few_2025,
  title = {On a few pitfalls in {KL} divergence gradient estimation for {RL}},
  url = {http://arxiv.org/abs/2506.09477},
  doi = {10.48550/arXiv.2506.09477},
  urldate = {2026-06-12},
  publisher = {arXiv},
  author = {Tang, Yunhao and Munos, Rémi},
  year = {2025},
  journal = {arXiv.org},
}

@misc{grattafiori_llama_2024,
  title = {The {Llama} 3 {Herd} of {Models}},
  url = {http://arxiv.org/abs/2407.21783},
  doi = {10.48550/arXiv.2407.21783},
  urldate = {2025-05-22},
  publisher = {arXiv},
  author = {Grattafiori, Aaron and Dubey, Abhimanyu and Jauhri, Abhinav and Pandey, Abhinav and Kadian, Abhishek and Al-Dahle, Ahmad and Letman, Aiesha and Mathur, Akhil and Schelten, A. and Vaughan, Alex and others},
  year = {2024},
  journal = {arXiv},
}

@misc{rofin_how_2026,
  title = {({How}) {Learning} {Rates} {Regulate} {Catastrophic} {Overtraining}},
  url = {http://arxiv.org/abs/2604.13627},
  doi = {10.48550/arXiv.2604.13627},
  urldate = {2026-07-13},
  publisher = {arXiv},
  author = {Rofin, Mark and Varre, Aditya and Flammarion, Nicolas},
  year = {2026},
  journal = {arXiv.org},
}

@inproceedings{catalan-tatjer_learning-forgetting_2026,
  title = {Learning-{Forgetting} {Optimality} in {Supervised} {Finetuning}: {A} {Cliff} {Perspective}},
  shorttitle = {Learning-{Forgetting} {Optimality} in {Supervised} {Finetuning}},
  url = {https://openreview.net/forum?id=aWl2TIUjxP&referrer=%5Bthe%20profile%20of%20Jonas%20Geiping%5D%28%2Fprofile%3Fid%3D~Jonas_Geiping1%29},
  language = {en},
  urldate = {2026-07-15},
  author = {Catalan-Tatjer, Albert and Geiping, Jonas},
}

@misc{xu_speculative_2025,
  title = {Speculative {Knowledge} {Distillation}: {Bridging} the {Teacher}-{Student} {Gap} {Through} {Interleaved} {Sampling}},
  shorttitle = {Speculative {Knowledge} {Distillation}},
  url = {http://arxiv.org/abs/2410.11325},
  doi = {10.48550/arXiv.2410.11325},
  urldate = {2026-05-19},
  publisher = {arXiv},
  author = {Xu, Wenda and Han, Rujun and Wang, Zifeng and Le, Long T. and Madeka, Dhruv and Li, Lei and Wang, William Yang and Agarwal, Rishabh and Lee, Chen-Yu and Pfister, Tomas},
  year = {2024},
  journal = {arXiv.org},
}

@misc{yu_dense_2026,
  title = {Dense {Supervision}, {Sparse} {Updates}: {On} the {Sparsity} and {Geometry} of {On}-{Policy} {Distillation}},
  shorttitle = {Dense {Supervision}, {Sparse} {Updates}},
  url = {http://arxiv.org/abs/2606.13657},
  doi = {10.48550/arXiv.2606.13657},
  urldate = {2026-07-20},
  publisher = {arXiv},
  author = {Yu, Guo and Liu, Wenlin and Hu, Yulan and Ma, Hao-Xuan and Jiang, Jun-Peng and Ye, Han-Jia},
  year = {2026},
  journal = {arXiv.org},
}

@inproceedings{yuan2026f,
  title = {From f (x) and g (x) to f (g (x)): Llms learn new skills in rl by composing old ones},
  author = {Yuan, Lifan and Chen, Weize and Zhang, Yuchen and Cui, Ganqu and Wang, Hanbin and You, Ziming and Ding, Ning and Liu, Zhiyuan and Sun, Maosong and Peng, Hao},
  booktitle = {arXiv.org},
  volume = {2026},
  pages = {147547--147574},
  year = {2026},
  journal = {arXiv.org},
  doi = {10.48550/arXiv.2509.25123},
}

@misc{zhu2026facesonpolicydistillationpitfalls,
  title = {The Many Faces of On-Policy Distillation: Pitfalls, Mechanisms, and Fixes},
  author = {Siqi Zhu and Xuyan Ye and Hongyu Lu and Weiye Shi and Ge Liu},
  year = {2026},
  eprint = {2605.11182},
  archiveprefix = {arXiv},
  primaryclass = {cs.AI},
  url = {https://arxiv.org/abs/2605.11182},
  journal = {arXiv.org},
  doi = {10.48550/arXiv.2605.11182},
}

@article{lambert2024tulu3,
  title = {Tülu 3: Pushing Frontiers in Open Language Model Post-Training},
  author = {Nathan Lambert and Jacob Morrison and Valentina Pyatkin and Shengyi Huang and Hamish Ivison and Faeze Brahman and Lester James V. Miranda and Alisa Liu and Nouha Dziri and Shane Lyu and others},
  year = {2024},
  email = {tulu@allenai.org},
  journal = {arXiv.org},
  doi = {10.48550/arXiv.2411.15124},
}

@article{cui2025process,
  title = {Process reinforcement through implicit rewards},
  author = {Cui, Ganqu and Yuan, Lifan and Wang, Zefan and Wang, Hanbin and Zhang, Yuchen and Li, Wendi and He, Bingxiang and Fan, Yuchen and Yu, Tianyu and Xu, Qi-Xin and others},
  journal = {Trans. Mach. Learn. Res.},
  year = {2025},
  doi = {10.48550/arXiv.2502.01456},
}

@misc{hong2024orpo,
  title = {ORPO: Monolithic Preference Optimization without Reference Model},
  author = {Jiwoo Hong and Noah Lee and James Thorne},
  year = {2024},
  eprint = {2403.07691},
  archiveprefix = {arXiv},
  primaryclass = {cs.CL},
  journal = {Proceedings of the 2024 Conference on Empirical Methods in Natural Language Processing},
  pages = {11170-11189},
  doi = {10.18653/v1/2024.emnlp-main.626},
  publisher = {Association for Computational Linguistics},
}

@misc{yuan2024advancing,
  title = {Advancing LLM Reasoning Generalists with Preference Trees},
  author = {Lifan Yuan and Ganqu Cui and Hanbin Wang and Ning Ding and Xingyao Wang and Jia Deng and Boji Shan and Huimin Chen and Ruobing Xie and Yankai Lin and others},
  year = {2024},
  eprint = {2404.02078},
  archiveprefix = {arXiv},
  primaryclass = {cs.AI},
  journal = {arXiv.org},
  doi = {10.48550/arXiv.2404.02078},
}

@misc{meng2024simposimplepreferenceoptimization,
  title = {SimPO: Simple Preference Optimization with a Reference-Free Reward},
  author = {Yu Meng and Mengzhou Xia and Danqi Chen},
  year = {2024},
  eprint = {2405.14734},
  archiveprefix = {arXiv},
  primaryclass = {cs.CL},
  url = {https://arxiv.org/abs/2405.14734},
  journal = {Advances in Neural Information Processing Systems 37},
  pages = {124198-124235},
  doi = {10.52202/079017-3946},
  publisher = {Neural Information Processing Systems Foundation, Inc. (NeurIPS)},
}

@misc{wang2024mathshepherdverifyreinforcellms,
  title = {Math-Shepherd: Verify and Reinforce LLMs Step-by-step without Human Annotations},
  author = {Wang, Peiyi and Li, Lei and Shao, Zhihong and Xu, R. X. and Dai, Damai and Li, Yifei and Chen, Deli and Y.Wu and Sui, Zhifang},
  year = {2024},
  eprint = {2312.08935},
  archiveprefix = {arXiv},
  primaryclass = {cs.AI},
  url = {https://arxiv.org/abs/2312.08935},
  journal = {Proceedings of the 62nd Annual Meeting of the Association for Computational Linguistics (Volume 1: Long Papers)},
  pages = {9426-9439},
  doi = {10.18653/v1/2024.acl-long.510},
  publisher = {Association for Computational Linguistics},
}

@misc{deepseekai2026deepseekv4highlyefficientmilliontoken,
  title = {DeepSeek-V4: Towards Highly Efficient Million-Token Context Intelligence},
  author = {DeepSeek-AI and Anyi Xu and Bangcai Lin and Bing Xue and Bingxuan Wang and Bingzheng Xu and Bochao Wu and Bowei Zhang and Chaofan Lin and Chen Dong and others},
  year = {2026},
  eprint = {2606.19348},
  archiveprefix = {arXiv},
  primaryclass = {cs.CL},
  url = {https://arxiv.org/abs/2606.19348},
  journal = {arXiv},
}

@misc{armandpour_unmasking_2026,
  title = {Unmasking {On}-{Policy} {Distillation}: {Where} {It} {Helps}, {Where} {It} {Hurts}, and {Why}},
  shorttitle = {Unmasking {On}-{Policy} {Distillation}},
  url = {http://arxiv.org/abs/2605.10889},
  doi = {10.48550/arXiv.2605.10889},
  urldate = {2026-07-22},
  publisher = {arXiv},
  author = {Armandpour, Mohammadreza and Ilhan, Fatih and Harrison, David and Jaiswal, Ajay and Hoang, Duc N. M. and Faghri, Fartash and Zhang, Yizhe and Cho, Minsik and Farajtabar, Mehrdad},
  year = {2026},
  journal = {arXiv.org},
}

@misc{jia_asymmetric_2026,
  title = {Asymmetric {On}-{Policy} {Distillation}: {Bridging} {Exploitation} and {Imitation} at the {Token} {Level}},
  shorttitle = {Asymmetric {On}-{Policy} {Distillation}},
  url = {http://arxiv.org/abs/2605.06387},
  doi = {10.48550/arXiv.2605.06387},
  urldate = {2026-07-22},
  publisher = {arXiv},
  author = {Jia, Nan and Yang, Haojin and Ma, Xing and Lian, Jiesong and Zhang, Shuailiang and Zhang, Weipeng and Zeng, Ke and Cai, Xunliang and Sun, Zequn},
  year = {2026},
  journal = {arXiv.org},
}

@misc{li_filter_2026,
  title = {Filter, {Then} {Reweight}: {Rethinking} {Optimization} {Granularity} in {On}-{Policy} {Distillation}},
  shorttitle = {Filter, {Then} {Reweight}},
  url = {http://arxiv.org/abs/2606.02684},
  doi = {10.48550/arXiv.2606.02684},
  urldate = {2026-07-22},
  publisher = {arXiv},
  author = {Li, Yuying and Zheng, Leqi and Yu, Yongzi and Zhou, Wenrui and Zhong, Xuchang and Hu, Xing and Jin, Jing and Yuan, Hangjie and Feng, Tao},
  year = {2026},
  journal = {arXiv.org},
}

@misc{zhang_good_2026,
  title = {Good {SFT} {Optimizes} for {SFT}, {Better} {SFT} {Prepares} for {Reinforcement} {Learning}},
  url = {http://arxiv.org/abs/2602.01058},
  doi = {10.48550/arXiv.2602.01058},
  urldate = {2026-02-06},
  publisher = {arXiv},
  author = {Zhang, Dylan and Xu, Yufeng and Wang, Haojin and Chen, Qingzhi and Peng, Hao},
  year = {2026},
  journal = {arXiv.org},
}

@inproceedings{yue2025does,
  title = {Does Reinforcement Learning Really Incentivize Reasoning Capacity in {LLM}s Beyond the Base Model?},
  author = {Yang Yue and Zhiqi Chen and Rui Lu and Andrew Zhao and Zhaokai Wang and Yang Yue and Shiji Song and Gao Huang},
  booktitle = {Advances in Neural Information Processing Systems 38},
  year = {2025},
  journal = {Advances in Neural Information Processing Systems 38},
  pages = {64304-64339},
  doi = {10.52202/085713-1933},
  publisher = {Neural Information Processing Systems Foundation, Inc. (NeurIPS)},
}

@article{team2026kimik3,
  title = {Kimi k3: Open frontier intelligence},
  author = {Bai, Kimi Team Yifan and Bai, Yifan and Bao, Yiping and M, C. and Cai, Jianfeng and Cai, Xin-Hao and Cao, Peizhou and Cao, Yuxuan and Chai, Ziwei and Charles, Y. and others},
  journal = {arXiv},
  year = {2026},
}
\bibliographystyle{abbrvnat}

\clearpage
\appendix
\onecolumn
\etocdepthtag.toc{appendix}

\begingroup
    \etocsettocstyle{\section*{Appendix Contents}}{}
    \etocsettagdepth{main}{none}
    \etocsettagdepth{appendix}{subsection}
    \tableofcontents
\endgroup

\clearpage

\section{Additional Experimental Results}

\subsection{Training without, rather than with, gradient clipping}
\label{app:exp:gradient_clipping}

\textbf{Setup.} our main experiments clip the global gradient norm at $1.0$, and as shown in \Cref{fig:app:gradient_norms}, the global gradient norm measured before clipping exceeds our threshold of $1.0$ during most training steps. To test whether the limited effect of rollout policy observed in \Cref{sec:on_vs_offpd} is caused by the use of gradient clipping, we repeat the Countdown-3 experiments with clipping disabled. For each combination of OnPD and OffPD with forward and reverse KL, we sweep learning rates from $1\times10^{-5}$ to $5\times10^{-5}$ while keeping all other experimental conditions fixed.

\begin{figure}
    \centering
    \includegraphics[width=0.8\linewidth]{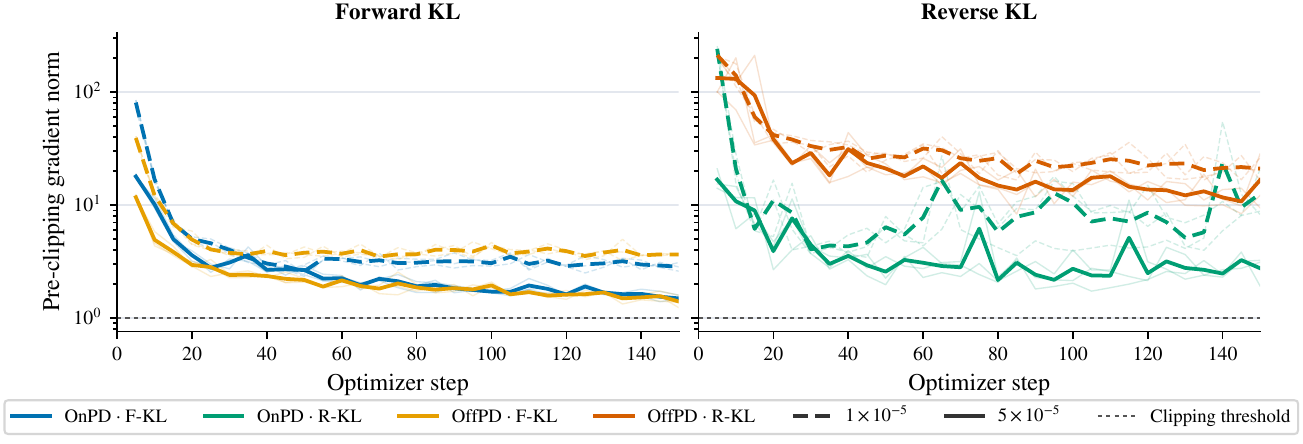}
    \caption{\textbf{Gradient clipping is active throughout Countdown-3 training.}
We show the global gradient norm before clipping for the Countdown-3 runs which use gradient clipping, with thin lines denoting individual seeds and thick lines their mean across three seeds. The horizontal line marks the clipping threshold of $1.0$. Every recorded norm exceeds this threshold, demonstrating that clipping rescales the gradients at every logged optimization step. Reverse KL produces substantially larger and more variable gradient norms than forward KL, including pronounced early spikes, supporting the role of clipping in stabilising reverse-KL training.}
    \label{fig:app:gradient_norms}
\end{figure}

\textbf{Results.} The results in \cref{fig:no_grad_cliping} reveal a pronounced difference between forward and reverse KL when gradient clipping is removed. Forward KL remains effective at larger learning rates, particularly with OffPD, whereas reverse-KL performance deteriorates substantially. This reinforces the analysis in \Cref{sec:rollout_spectrum}: the token-level reverse-KL gradient is weighted by an unbounded log-probability ratio (see \Cref{eq:gradient_reverse_kl}), allowing individual tokens to produce unusually large gradients. Without clipping, these high-magnitude contributions can dominate an update and increase optimisation variance, making reverse-KL training less stable.

When it comes to catastrophic forgetting and parameter-update sparsity, here, even without the gradient clipping, the performance differences are still largely determined by the learning rate, with some extra variability coming from the rollout policy (OnPD with forward KL leads to particularly low OOD performance after training). Overall, this ablation suggests that gradient clipping is important for stabilising reverse-KL optimisation, but does not explain the generally limited effect of rollout policy observed in our main experiments.

\begin{figure}[h!]
    \centering
    \includegraphics[width=1\linewidth]{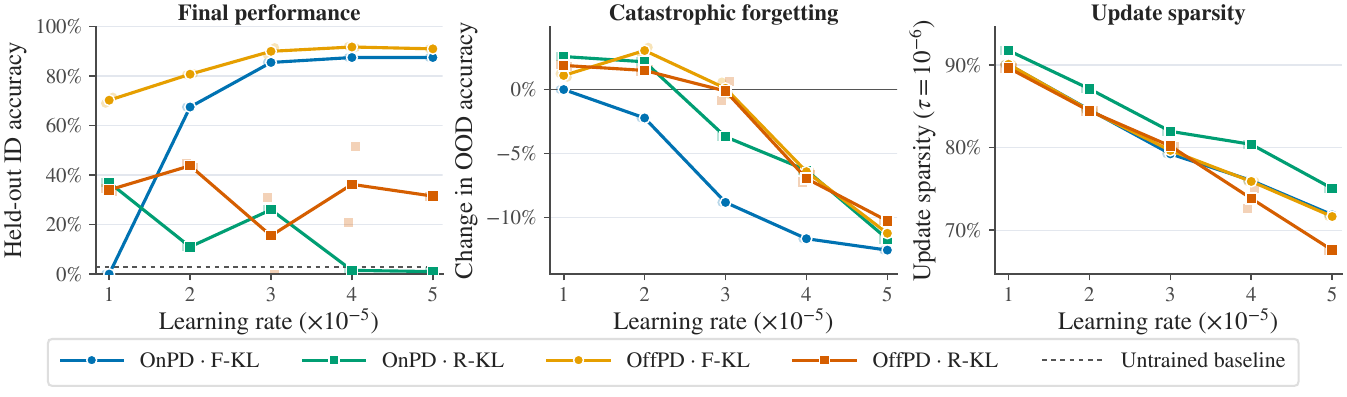}
    \caption{\textbf{The effect of learning-rate \textit{without} gradient clipping.} Removing clipping reveals a substantial difference in final performance between forward and reverse KL, with reverse KL suffering from unstable training, while learning rate remains the dominant determinant of forgetting and parameter-update sparsity.}
    \label{fig:no_grad_cliping}
\end{figure}

\textbf{Why does gradient clipping affect final performance, but not forgetting or update sparsity?} Gradient clipping appears important for stabilising reverse KL, but does not explain the dominant effect of learning rate over rollout policy on forgetting and parameter-update sparsity. One plausible explanation lies in Adam’s dynamics. Adam updates parameters using $\hat{m}_t/(\sqrt{\hat{v}_t}+\epsilon)$, where $m_t$ and $v_t$ are moving averages of the gradient and its coordinate-wise square \citep{kingma2014adam}. Because this normalisation largely cancels uniform changes in gradient scale, a larger raw gradient norm does not necessarily produce a proportionally larger parameter update. An isolated unclipped gradient spike can nevertheless disrupt optimisation: it briefly biases $m_t$ towards the direction of an atypical batch while inflating the more slowly decaying $v_t$, which subsequently suppresses corrective updates \citep{reddi2019convergence}. This may be particularly harmful for OnPD, because an early disruptive update changes the policy generating subsequent rollouts, producing poorer training trajectories, while the inflated second moment $v_t$ inhibits recovery.

Crucially, this mechanism can impair target-task learning without increasing forgetting: if Adam normalises the initial spike and then suppresses later updates, the cumulative parameter displacement (including in the sensitive directions related to prior capabilities) may remain small. Consequently, gradient clipping can stabilise training, by removing the disruptive spikes in the gradient norms, while having much less influence on forgetting and parameter-update sparsity. This result is consistent with prior evidence that gradient clipping has limited influence on accumulated update sparsity in LLM post-training \citep{mukherjee_reinforcement_2025}.

\clearpage
\subsection{Training with Sampled KL, Rather than Full-Vocabulary KL}
\label{app:sampled_kl}

\begin{figure}
    \centering
    \includegraphics[width=\linewidth]{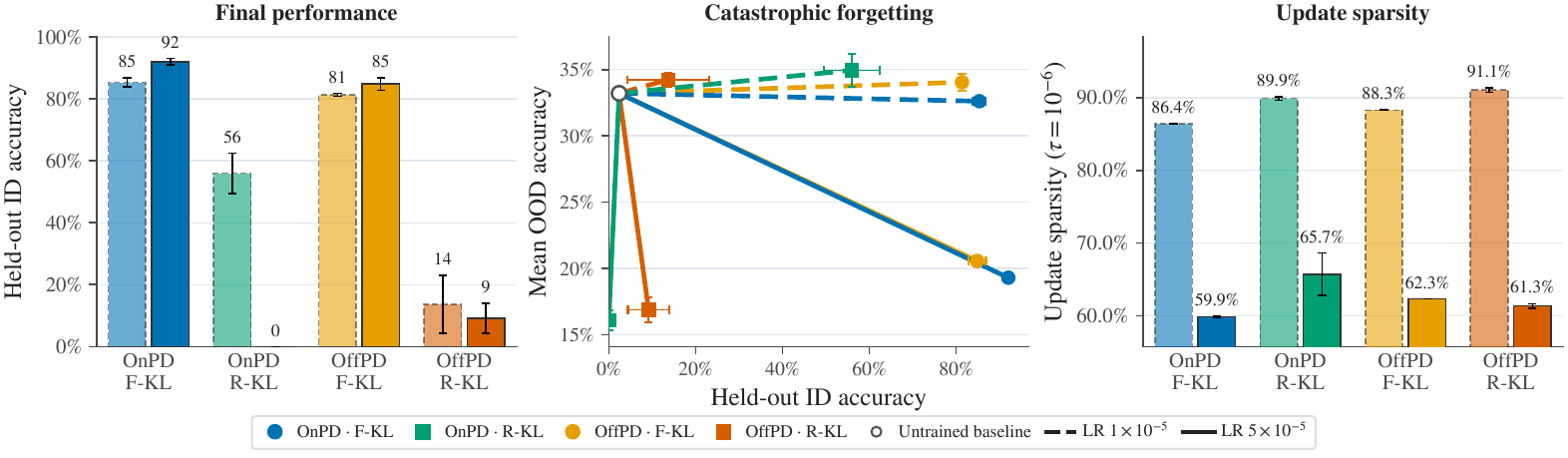}
    \caption{\textbf{Performance after training with sampled KL.} Error bars show SEMs computed over $N=3$ seeds. Sampled-KL training exhibits the same qualitative trends as full-vocabulary KL. Forward KL achieves substantially higher and more robust final performance than reverse KL, while differences between OnPD and OffPD are comparatively limited. Catastrophic forgetting and update sparsity are governed primarily by the learning rate. These results indicate that replacing full-vocabulary KL with a single-sample gradient estimator does not materially alter our main conclusions regarding rollout policy, KL direction, and learning rate.}
    \label{fig:app:sampled_kl}
\end{figure}

\textbf{Setup.} To test whether our conclusions depend on computing the KL divergence over the full vocabulary, we repeat the Countdown-3 experiments with single-sample stochastic estimators, using Llama-3.2-1B-Instruct, three seeds, learning rates \(1\times10^{-5}\) and \(5\times10^{-5}\), and 150 training steps. Let $\pi_S^\theta(\cdot\mid x, y_{<n})$ and $\pi_T(\cdot\mid x, y_{<n})$ denote the student and teacher next-token distributions conditioned on rollout prefix $(x, y_{<n})$. OnPD constructs these prefixes from student-generated rollouts, whereas OffPD uses teacher-generated rollouts. For forward KL, the gradient is
\[
\nabla_\theta D_{\mathrm{F-KL}}(\pi^\theta_S, \pi_T)
=
-\mathbb{E}_{y_n\sim \pi_T}
\left[\nabla_\theta\log \pi^\theta_S(y_n)\right].
\]
We therefore draw one token \(y_n\sim \pi_T(\cdot\mid x,y_{<n})\) and use
\[
\widehat{\mathcal L}_{\mathrm{F\text{-}KL}}
=
-\log \pi_S^\theta(y_n\mid x, y_{<n}).
\]
The omitted teacher log-probability is constant with respect to \(\theta\), so this gives an unbiased estimator of the full forward-KL gradient. For reverse KL,
\[
\nabla_\theta D_{\mathrm{R-KL}}(\pi^\theta_S, \pi_T)
=
\mathbb{E}_{y_n\sim \pi_S^\theta}
\left[
\bigl(\log \pi_S^\theta(y_n\mid x, y_{<n})-\log \pi_T(y_n\mid x, y_{<n})\bigr)
\nabla_\theta\log \pi_S^\theta(y_n\mid x, y_{<n})
\right].
\]
Accordingly, we draw \(x_n\sim \pi_S^\theta(\cdot\mid x, y_{<n})\) and optimise the score-function surrogate
\[
\widehat{\mathcal L}_{\mathrm{R\text{-}KL}}
=
\operatorname{sg}\!\left[
\log \pi_S^\theta(y_n\mid x, y_{<n})-\log \pi_T(y_n\mid x, y_{<n})
\right]
\log \pi_S^\theta(y_n\mid x, y_{<n}),
\]
where \(\operatorname{sg}\) denotes stop-gradient operation. Thus, the rollout policy determines the distribution of prefixes, while the KL direction determines the distribution from which the additional token is sampled: the teacher for forward KL and the student for reverse KL. This token is sampled independently at every valid completion position and need not coincide with the token appearing in the rollout. We average the sampled losses over all valid tokens within each completion and subsequently over the training batch, keeping the remaining training and evaluation settings matched to the full-vocabulary experiments.

\textbf{Results.} The results presented in \Cref{fig:app:sampled_kl} for training with sampled KL show patterns which we have also identified when using full-vocabulary KL. Namely, we observe that OnPD and OffPD lead to very similar final performance when controlling for the KL direction and learning rate. Further, forward KL is much more robust to the rollout policy than reverse KL, further supporting the findings of \Cref{sec:rollout_spectrum}. When it comes to catastrophic forgetting and update sparsity, we observe that both of these properties are primarily governed by the learning rate, similarly as when training using the full-vocabulary KL. Thus, we conclude that using the full-vocabulary KL in our main experimental results is not significantly affecting our final conclusions, and does not significantly contribute to the observed limited effects of rollout policy in the settings we study.

\clearpage
\subsection{Training on a mathematical reasoning dataset requiring longer rollouts}
\label{app:math_qwen}

\textbf{Setup.}
To test whether our conclusions extend to a more challenging reasoning domain, where finding the correct answer requires producing longer reasoning traces, we distill Qwen2.5-Math-1.5B-Instruct \citep{yang2024qwen25mathtechnicalreportmathematical} into Qwen2.5-1.5B-Instruct using the 20K-example Numina--MATH training set described below. We consider all combinations of OnPD and OffPD rollouts with forward and reverse full-vocabulary KL, using learning rates of $1\times10^{-5}$ and $5\times10^{-5}$. OnPD completions are sampled from the student, whereas OffPD completions are sampled from the teacher; in both cases, we use generation temperature $0.7$, top-$p=0.8$ (following guidelines from \citet{yang2024qwen25mathtechnicalreportmathematical}), and a maximum completion length of 1,024 tokens. The student and teacher distributions used to compute the KL objective have temperature $1$. We train all model parameters for 1,000 optimization steps with an effective batch size of 128 and a constant learning rate following 10 warm-up steps. We evaluate checkpoints every 100 optimisation steps and report, for each condition, the maximum observed MATH-500 accuracy along its training trajectory. This metric is intended to compare the peak performance reached under a common training and evaluation budget, rather than to estimate the held-out performance of a checkpoint selected by an independent validation procedure. Because checkpoint selection and reporting use the same evaluations, the reported maxima may be optimistic, particularly for conditions with more variable learning curves. We therefore additionally report the complete trajectories and final-checkpoint performance. Catastrophic forgetting and update sparsity are evaluated at the corresponding selected checkpoint, and also reported over the course of training.

\paragraph{Numina--MATH training set.}
We construct a 20,000-example mathematical-reasoning dataset by combining all 7,498 problems from the training split of MATH \citep{hendrycks2021measuringmathematicalproblemsolving} with 12,502 synthetic problems sampled from the \texttt{synthetic\_math} subset of NuminaMath-CoT \citep{numina_math_datasets}. For each NuminaMath example, we require a non-empty problem and solution containing a non-empty, balanced final \texttt{\textbackslash boxed\{\}} answer. We normalize problem text by lowercasing and collapsing whitespace, remove duplicates, and exclude every synthetic problem matching either the MATH training or test split. We additionally restrict synthetic solutions to at most 384 tokens, measured using the DeepSeek-R1-Distill-Llama-8B tokenizer, and sample the remaining examples using seed 42. Responses are considered correct when their final boxed answer is symbolically equivalent to the reference answer. We use the prompt format recommended for Qwen2.5-Math. Each example contains a system message stating, ``Please reason step by step, and put your final answer within \texttt{\textbackslash boxed\{\}},'' followed by the problem as the user message. We render these messages using the model's native chat template and append the assistant-generation prompt. No demonstrations or reference solutions are included. We use the same prompt format for teacher and student rollouts, KL scoring, and MATH-500 evaluation.

\paragraph{Inference during training.} We used vLLM \citep{kwon2023efficient} as a colocated inference engine for efficient rollout generation. In the OnPD setting, vLLM samples from the student policy, whose latest weights are synchronized with the training model before each rollout batch. To correct residual differences between the vLLM sampling distribution \(q\) and the policy evaluated by the training backend \(p_\theta\), following \citet{shenfeld_self-distillation_2026} we assign each completion the detached importance weight
\[
w(y)=\frac{1}{|y|}\sum_{t=1}^{|y|}\min\!\left(2,\exp\!\left[\log p_\theta(y_t\mid x,y_{<t})-\log q(y_t\mid x,y_{<t})\right]\right).
\]
This sequence-level weight is applied uniformly to the token-level distillation loss. Clipping the token ratios at \(2\) limits the variance of the correction, while synchronizing weights at every optimizer step keeps the generated trajectories on-policy.

\begin{figure}[h]
    \centering
    \includegraphics[width=\linewidth]{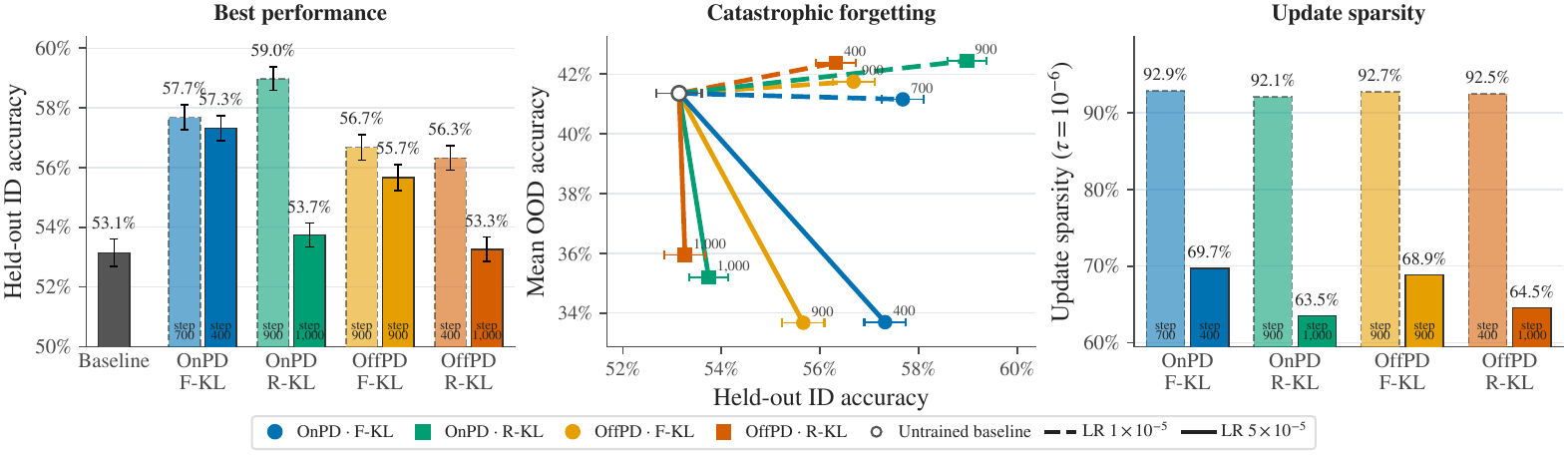}
    \caption{\textbf{Math distillation results on MATH-500.} We report the best MATH-500 performance attained during training, movement from the untrained baseline in ID--OOD space, and parameter-update sparsity at the corresponding selected checkpoint. In this regime, OnPD with reverse-KL outperforms other conditions. However, the learning rate still largely controls catastrophic forgetting and parameter-update sparsity. The performance of forward KL is also more robust to the changes in rollout policy and learning rate.}
    \label{fig:app:math_performance}
\end{figure}

\begin{figure}[h]
    \centering
    \includegraphics[width=\linewidth]{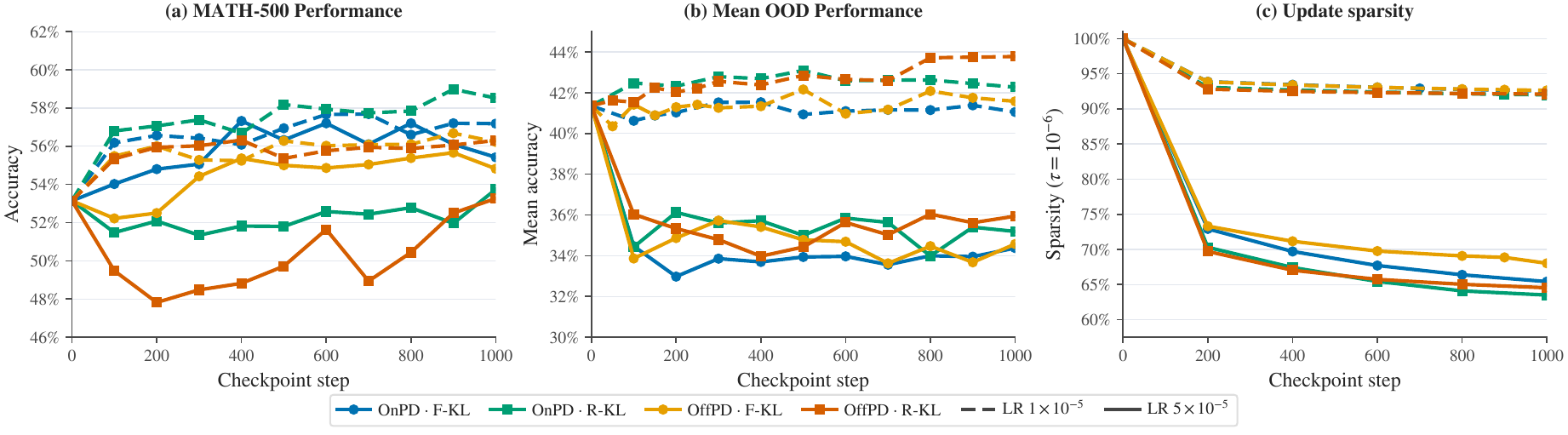}
    \caption{\textbf{MATH-500 performance throughout training.} Accuracy is evaluated every 100 optimization steps using ten generations per problem. Performance generally improves beyond the untrained baseline, although the optimal checkpoint varies considerably across conditions. The lower learning rate produces more stable learning trajectories, while the larger learning rate is particularly unstable with reverse KL. Looking at catastrophic forgetting and parameter-update sparsity, differences between OnPD and OffPD remain comparatively modest relative to the effects of learning rate and KL direction.}
    \label{fig:app:math_learning_curves}
\end{figure}

\textbf{Results.} 
The results in \Cref{fig:app:math_performance} largely reproduce the patterns observed in our main experiments. All conditions improve upon the untrained MATH-500 accuracy of $53.1\%$, with the selected checkpoints obtaining between $53.3\%$ and $59.0\%$. When controlling for the KL direction and learning rate, we observe that in this regime OnPD outperforms OffPD on average, which is consistent with the finding that OnPD helps with generalisation to more difficult tasks (given that MATH-500 is a particularly challenging subset of the MATH test set). More broadly, we observe that again, forward KL seems to be more robust to the rollout policy than reverse KL, with the performance of the reverse KL being highly dependent on the learning rate. Most importantly, \textbf{the learning rate continues to govern catastrophic forgetting and update sparsity}. At $1\times10^{-5}$, all methods approximately preserve the baseline OOD performance while retaining $92$--$93\%$ update sparsity. Increasing the learning rate to $5\times10^{-5}$ reduces mean OOD accuracy from approximately $41.4\%$ to $33.7$--$35.9\%$ and reduces update sparsity to $63$--$70\%$. Thus, even in a longer-horizon mathematical-reasoning setting, the rollout policy has a comparatively limited effect, while the learning rate primarily determines the extent of forgetting and the density of parameter updates. We additionally provide the learning curves of the models in \Cref{fig:app:math_learning_curves}, demonstrating how the MATH-500 performance, OOD performance, and update sparsity vary over the course of training. Just as in the case of other datasets (cf. \Cref{fig:app:learning_curves}), the majority of the changes to the OOD performance happen in the initial steps of training.

\clearpage
\subsection{Output Coverage and Generalisation at the Higher Learning Rate}
\label{app:exp:countdown_spectrum}

\Cref{fig:app:countdown_spectrum_lr5e5} complements \Cref{fig:countdown_spectrum} with learning rate $5\times10^{-5}$. Both figures use checkpoint 150 and matched distillation seeds 42, 43, and 44. Each evaluation uses the first 200 test questions, ten sampled responses per question, and evaluation seed 42. Countdown-3 uses temperature $1.0$; Countdown-4E uses temperature $0.5$ and a 4,096-token completion limit. Countdown-4 has four operands instead of three which increases the search space significantly, with operands drawn from 1--10. Shading denotes the pass@1-to-pass@10 span, not uncertainty; error bars denote SEM across distillation seeds. Under the larger learning rate, the performance generally decreases significantly, particularly when training with Reverse KL. Hence, we consider the trends observed in these figures less informative than those for learning rate $1\times10^{-5}$.

\begin{figure}[htbp]
    \centering
    \includegraphics[width=\linewidth]{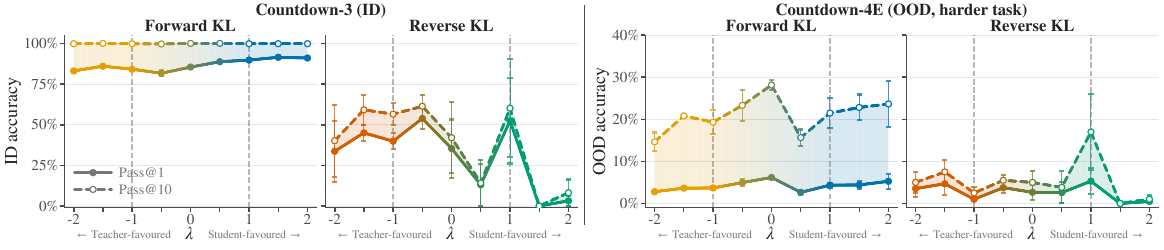}
    \caption{\textbf{Output coverage and generalisation at LR $5\times10^{-5}$ for Llama3.2-1B.} Same protocol and plotting conventions as \Cref{fig:countdown_spectrum}; means and SEM over three distillation checkpoints.}
    \label{fig:app:countdown_spectrum_lr5e5}
\end{figure}

\clearpage
\subsection{Additional Results on Qwen2.5-1.5B-Instruct}
\label{app:exp:qwen25}

To assess whether our findings generalise beyond the Llama~3 family, we repeat the main experiments using Qwen2.5-1.5B-Instruct as the student and Qwen2.5-7B-Instruct as the teacher \citep{qwen2025qwen25technicalreport}. We strengthen the teacher using the same training pipeline described in Appendix~\ref{app:details:trained_teacher}. We use Qwen2.5 rather than Qwen3 to avoid potential confounding effects from Qwen3's explicit reasoning training which can be activated in the ``thinking mode''.

\textbf{Performance and forgetting.} Figure~\ref{fig:app:qwen_id_eval} reports final held-out performance across the three datasets, while Figure~\ref{fig:app:qwen_ood_forgetting} shows the corresponding changes in OOD performance. The results are consistent with those obtained for Llama3 in the main text. We find no consistent performance advantage for either OnPD or OffPD across tasks and learning rates. Instead, KL direction has a clearer effect on target-task performance, with forward KL generally outperforming reverse KL, while learning rate remains the primary determinant of catastrophic forgetting. These results suggest that our main conclusions are robust across the two model families.


\begin{figure}[h]
    \centering
    \includegraphics[width=\linewidth]{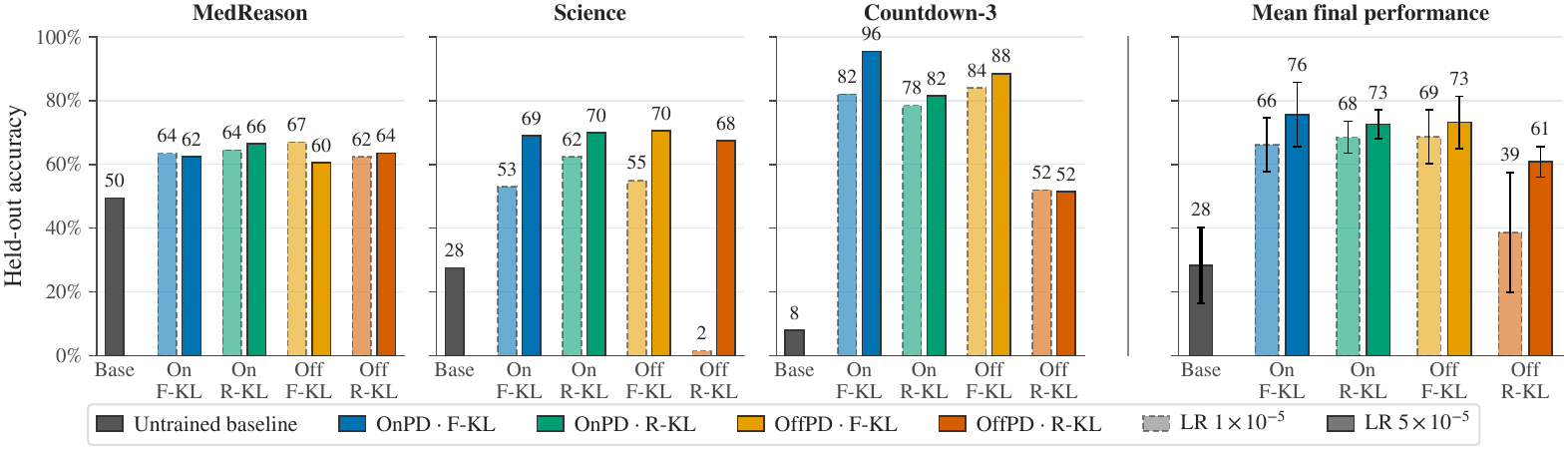}
    \caption{\textbf{Comparison of the effectiveness of OnPD and OffPD training for Qwen2.5-1.5B.} Dataset panels report mean held-out accuracy over $N=3$ repeats; error bars show SEM. The right-most panel shows the average across the nine runs. Across the evaluated settings, KL direction appears to have a larger effect on final performance than rollout policy, with forward KL outperforming reverse KL on average.}
    \label{fig:app:qwen_id_eval}
\end{figure}

\begin{figure}[h]
    \centering
    \begin{subfigure}[t]{0.48\textwidth}
        \centering
        \includegraphics[width=\linewidth]{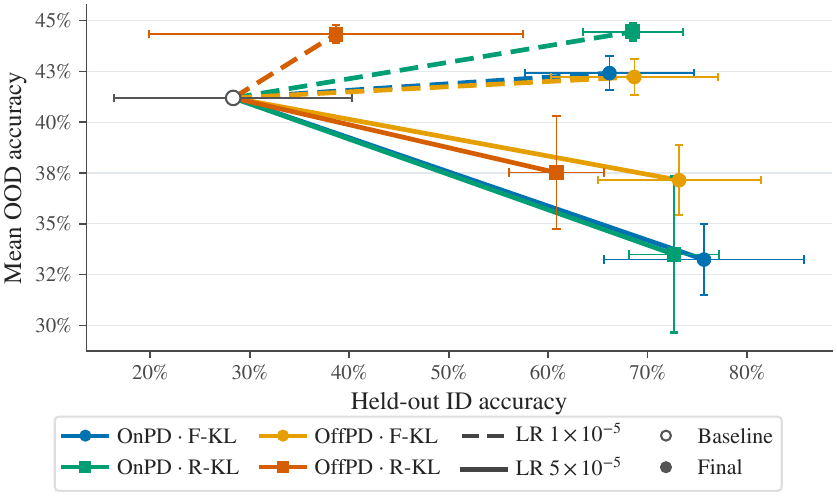}
        \caption{Catastrophic Forgetting}
        \label{fig:app:qwen_ood_forgetting}
    \end{subfigure}
    \hfill
    \begin{subfigure}[t]{0.48\textwidth}
        \centering
        \includegraphics[width=\linewidth]{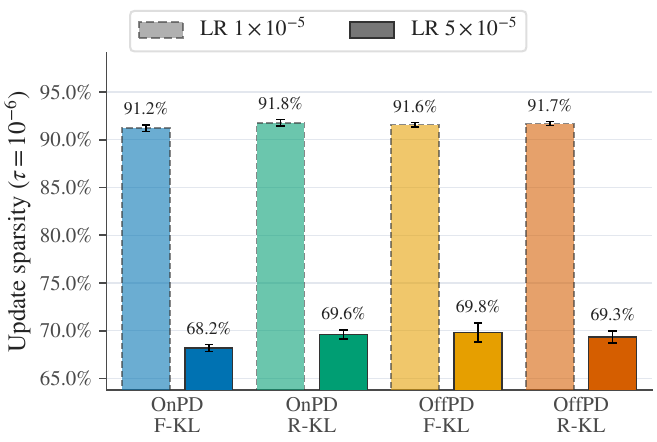}
        \caption{Parameter-Update Sparsity}
        \label{fig:sparsity-qwen}
    \end{subfigure}

    \caption{\textbf{Catastrophic forgetting and parameter-update sparsity for Qwen2.5-1.5B.} Similarly as for Llama3.2-1B, we observe that for Qwen2.5-1.5B learning rate is a stronger predictor of the degree of catastrophic forgetting and the parameter-update sparsity than the rollout policy. \textit{Left:} The mean OOD accuracy consistently decreases over the course of training when the larger learning rate of $5 \times 10^{-5}$ is used, but increases when the smaller learning rate is used instead. OffPD does not lead to stronger forgetting in any of the conditions. \textit{Right:} Similarly, parameter-update sparsity does not vary significantly between the rollout policies, but depends strongly on the learning rate. For both figures, error bars mark SEM computed over the three datasets.}
    \label{fig:qwen-forgetting-sparsity}
\end{figure}

\paragraph{Sensitivity to rollout policy.}
Further, we repeat the rollout-spectrum experiment from \Cref{fig:rollout_spectrum} using Qwen2.5-1.5B. As shown in \Cref{fig:app:qwen_rollout_spectrum}, forward-KL remains comparatively robust across the spectrum, whereas reverse-KL is highly sensitive to teacher-favoured rollouts ($\lambda<0$), with several settings exhibiting severe performance degradation. For $\lambda\geq0$, reverse-KL recovers substantially and approaches forward-KL performance. The learning rate primarily controls forgetting and update sparsity. At $5\times10^{-5}$, updates are markedly less sparse and cause substantially more forgetting than at $10^{-5}$. Surprisingly, within this high-learning-rate regime, student-favoured rollouts ($\lambda>0$) produce increasingly severe forgetting under forward-KL and generally greater forgetting under reverse-KL, although the latter is less monotonic. Thus, the Qwen results reproduce the main qualitative asymmetry observed with Llama-3.2, while providing additional evidence showing that on-policy rollouts \textit{need not} reduce catastrophic forgetting.

\begin{figure}
    \centering
    \includegraphics[width=\linewidth]{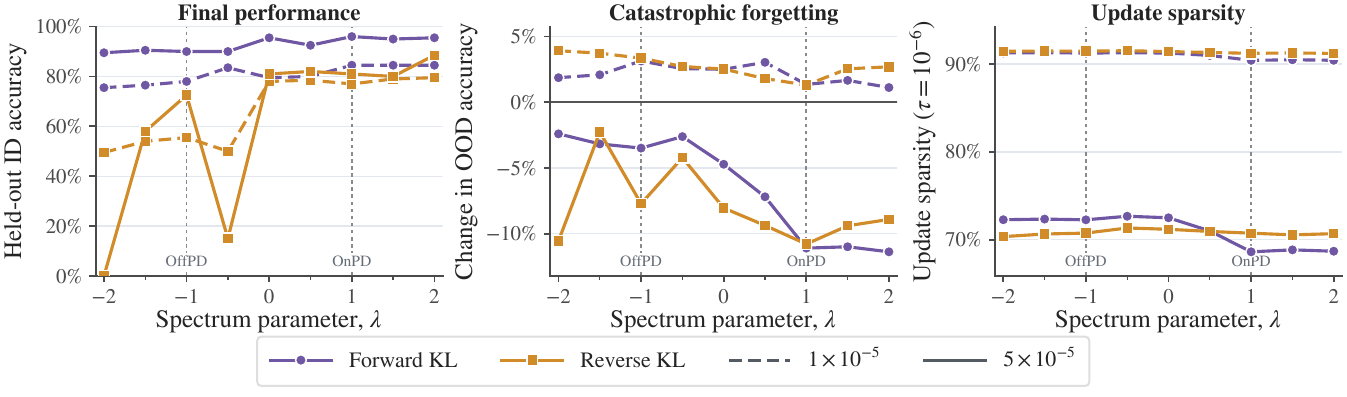}
    \caption{\textbf{Sensitivity of Qwen2.5-1.5B to the rollout policy.} Results at checkpoint 150 across the contrastive-decoding spectrum, using seed 42. The performance of forward-KL is robust to the rollout policy, while the performance of reverse-KL degrades sharply under teacher-favoured rollouts. Consistent with previous findings, larger learning rates amplify forgetting, particularly for student-favoured rollouts, while also producing denser parameter updates.}
    \label{fig:app:qwen_rollout_spectrum}
\end{figure}

\textbf{Output coverage and generalisation along the rollout-policy spectrum.}
\Cref{fig:app:qwen_countdown_spectrum} repeats the Countdown-3/Countdown-4E comparison in \Cref{fig:countdown_spectrum} for Qwen2.5-1.5B-Instruct at both learning rates. The evaluation protocol is the same as in Appendix~\ref{app:exp:countdown_spectrum}, using the distillation checkpoint from seed 42.

\begin{figure}[htbp]
    \centering
    \includegraphics[width=0.8\linewidth]{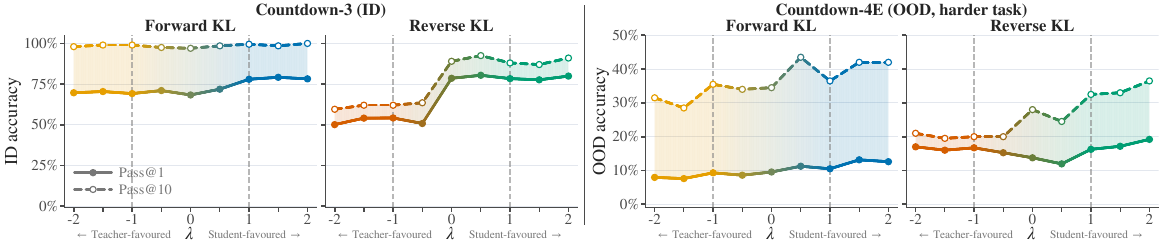}
    \includegraphics[width=0.8\linewidth]{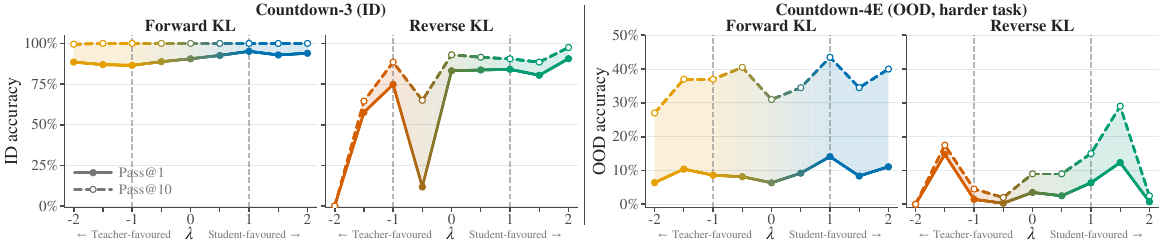}
    \caption{\textbf{Output coverage and generalisation for Qwen2.5-1.5B.} LR $1\times10^{-5}$ (top) and $5\times10^{-5}$ (bottom). Same plotting conventions as \Cref{fig:countdown_spectrum}.}
    \label{fig:app:qwen_countdown_spectrum}
\end{figure}

\clearpage
\subsection{Per-benchmark analysis}
\label{app:exp:ood_eval}
Figure~\ref{fig:ood_change_by_dataset} shows the accuracy across all 7 OOD datasets. Higher learning rate leads to strong losses on MMLU-Pro, HumanEval-Instruct, IFEval, and EQ-Bench under both rollout policies. TruthfulQA and BBQ improve slightly for most configurations. Within individual benchmarks, differences between OnPD and OffPD are generally smaller than the learning-rate effect and do not have a consistent direction. The three training-task markers show that this conclusion is also largely stable across the domain on which distillation is performed. 

\begin{figure*}[h]
    \centering
    \includegraphics[
        width=\textwidth
    ]{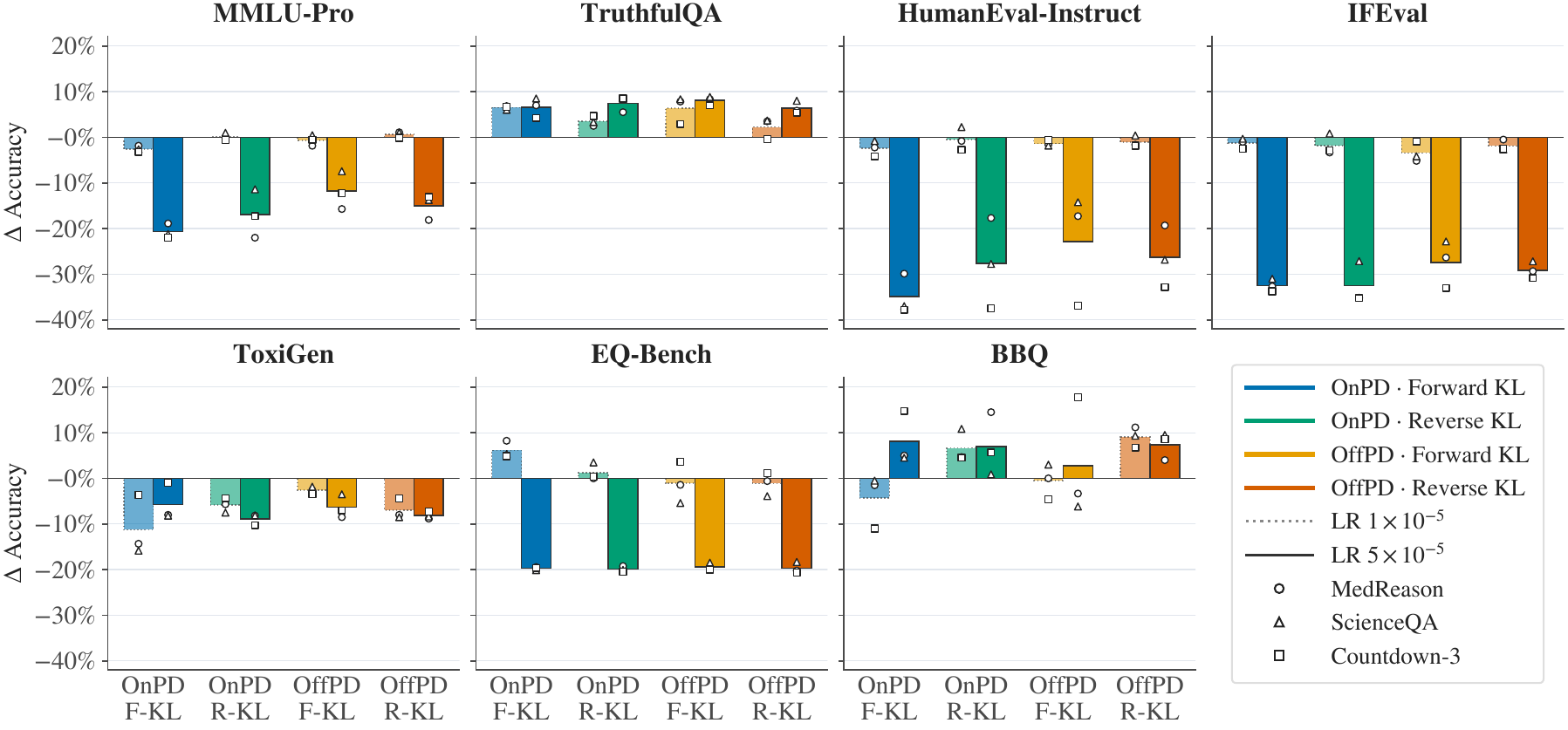}
    \caption{
        \textbf{Per-benchmark changes in out-of-distribution performance.}
        Each panel reports the change from the untrained student baseline to
        the final distilled checkpoint on one OOD benchmark. Bars average
        repeated runs within each training task and then weight MedReason,
        Science, and Countdown-3 equally. White markers show the
        corresponding training-task means.
    }
    \label{fig:ood_change_by_dataset}
\end{figure*}

\begin{wrapfigure}[14]{r}{0.5\linewidth}
    \centering
    \vspace{-2em}
    \includegraphics[width=0.8\linewidth]{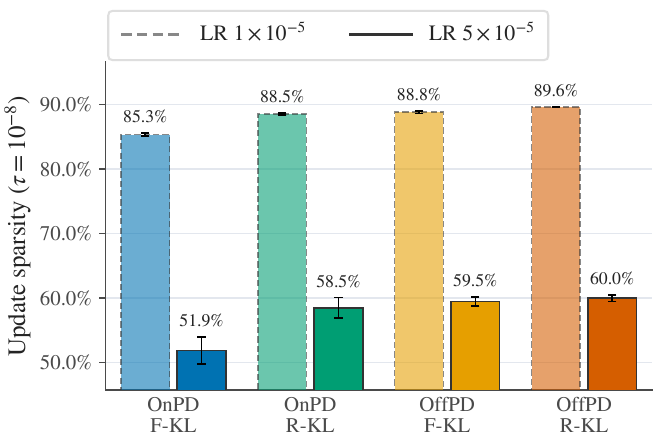}
    \vspace{-1em}
    \caption{\textbf{Parameter-update sparsity with a smaller threshold.} Each bar pools the three datasets and $N=3$ repeats per dataset; error bars show SEM across the nine runs. Learning rate, more than rollout policy, determines sparsity of updates.}
    \label{fig:sparsity_threshold}
\end{wrapfigure}

\subsection{Sparsity Analysis with a Smaller Threshold Value}

To show the robustness of our results evaluating the sparsity of the parameter updates, we also report the results obtained with a more conservative threshold, $\tau=10^{-8}$ (compared to $\tau=10^{-6}$ reported in the main text). \Cref{fig:sparsity_threshold} shows that using a smaller threshold does not substantially affect the results, and particularly their qualitative interpretation.

\clearpage
\subsection{Learning and Forgetting Curves for Science, MedReason and Countdown}

\begin{figure}[h]
    \centering
    \includegraphics[width=\linewidth]{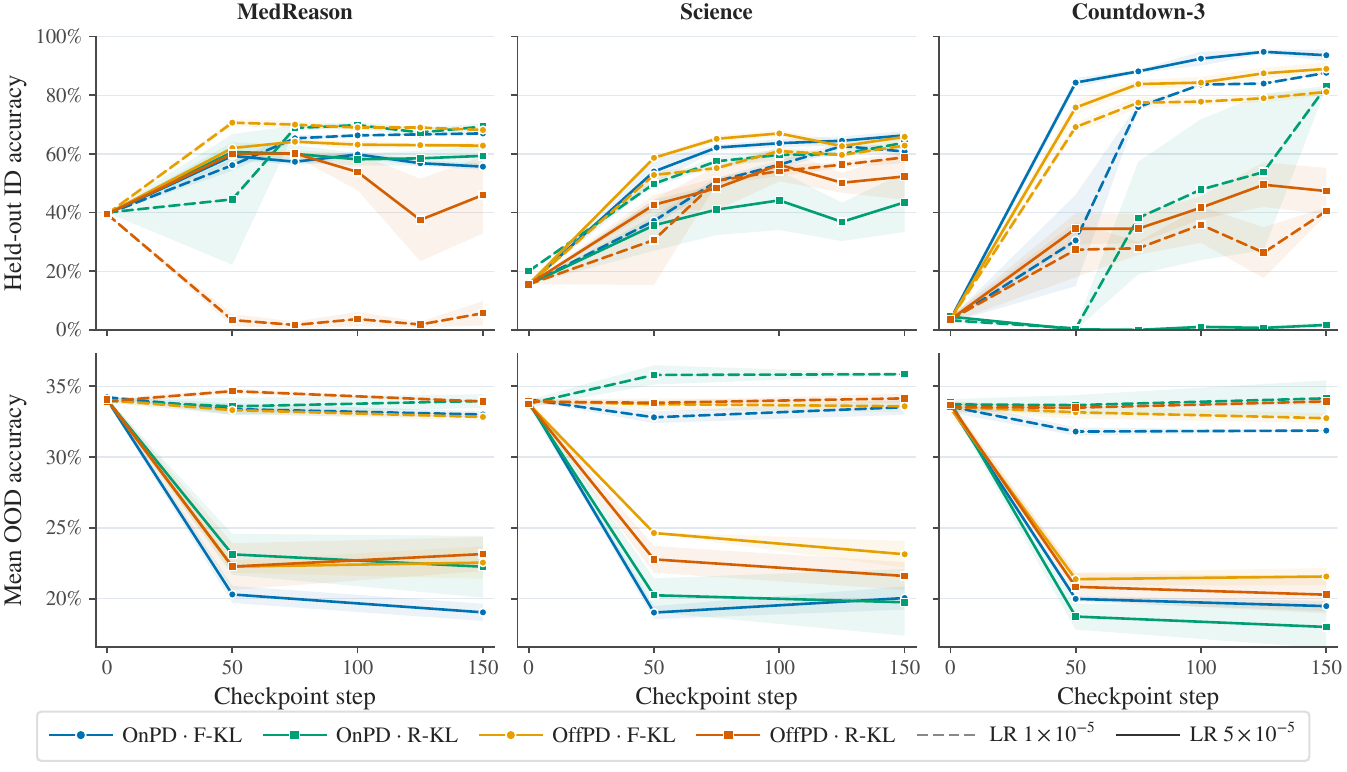}
    \caption{\textbf{Learning and forgetting curves for the main runs on Llama3.2-1B-Instruct.} We report the accuracy on the held-out test set of the training task (top row) and the mean accuracy over the out-of-distribution tasks (bottom row) over the course of training. We can observe that for the majority of the runs, the target task performance converges to a stable value after 150 optimisation steps. Additionally, most of the forgetting happens within the first 50 steps of training, with the OOD performance maintaining a relatively stable value at steps 50-150. The shading represents mean $\pm$ one standard error of the mean (SEM) computed over three seeds.}
    \label{fig:app:learning_curves}
\end{figure}

\begin{figure}[h]
    \centering
    \includegraphics[width=\linewidth]{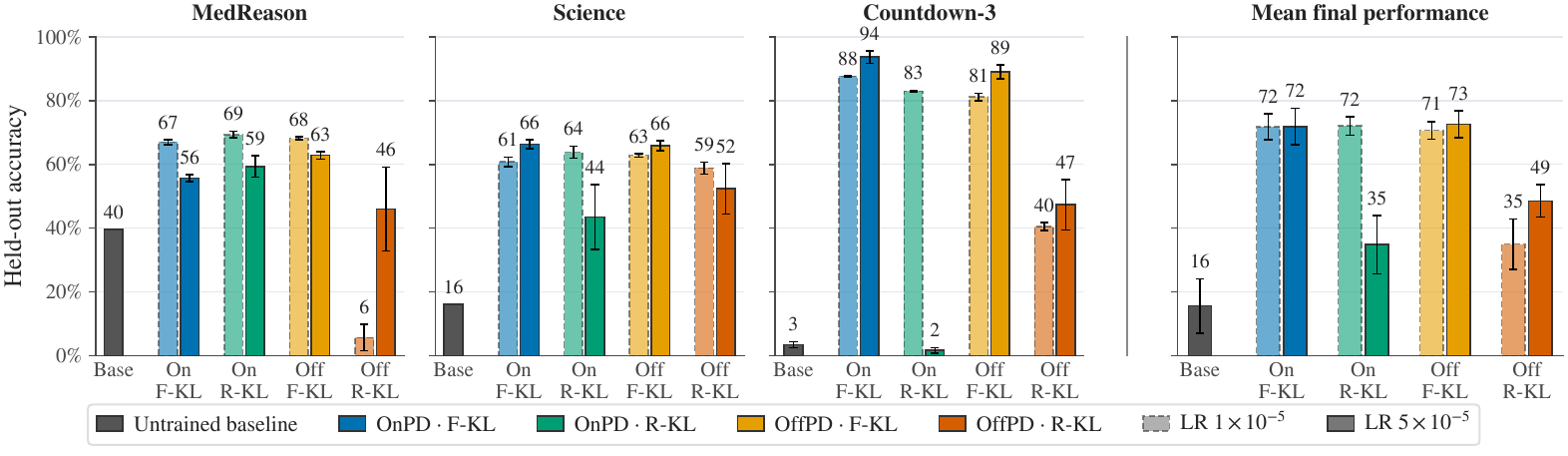}
    \caption{\textbf{Breakdown of the final performance obtained on each dataset.} We demonstrate a breakdown of the final performance after 150 steps of training, across the three datasets. Error bars denote SEM computed over $N=3$ seeds.}
    \label{fig:app:per_dataset_performance}
\end{figure}

\clearpage
\subsection{Comparison of the Likelihoods of Student and Teacher Trajectories During OnPD and OffPD}

In \Cref{fig:app:likelihoods_over_training} we visualise the likelihoods of the teacher and the student over the trajectories generated over the course of training, to validate to what extent the small visible differences between OnPD and OffPD stem from the small differences in the trajectories generated by the student and the teacher. The figure demonstrates that across the datasets, we observe significant differences between the likelihood dynamics of OnPD and OffPD.

During OffPD training runs, the teacher models are very confident on the (teacher-) generated trajectories, with average confidence $\approx90\%$. The student confidence on the teacher-generated trajectories slowly increases over the course of training, following a stable path. However, most of the time it remains below the confidence of the teacher.

For OnPD runs, where the trajectories are generated by the student, the situation looks different. Both when using F-KL and R-KL, the student is initially confident on the trajectories (average likelihood $\approx 80\%$), while the teacher is not (average likelihood $\approx 20\%$). When training with R-Kl, both the likelihood under the teacher and under the student tends to increase during training, suggesting that the student becomes more confident in its predictions, and the quality of the generation increases (as indicated by the increasing likelihood under the teacher). However, under F-KL the situation looks different: the confidence of the student drops drastically in the initial steps, and then slowly recovers over the course of training. The average likelihood of the teacher on the student-generated traces remains relatively low throughout training, suggesting that the quality of the supervision might be low. We further analyse these peculiar dynamics of the OnPD training with F-KL in Appendix~\ref{app:onpd-fkl-entropy-collapse}. 

\begin{figure}[h!]
    \centering
    \includegraphics[width=\linewidth]{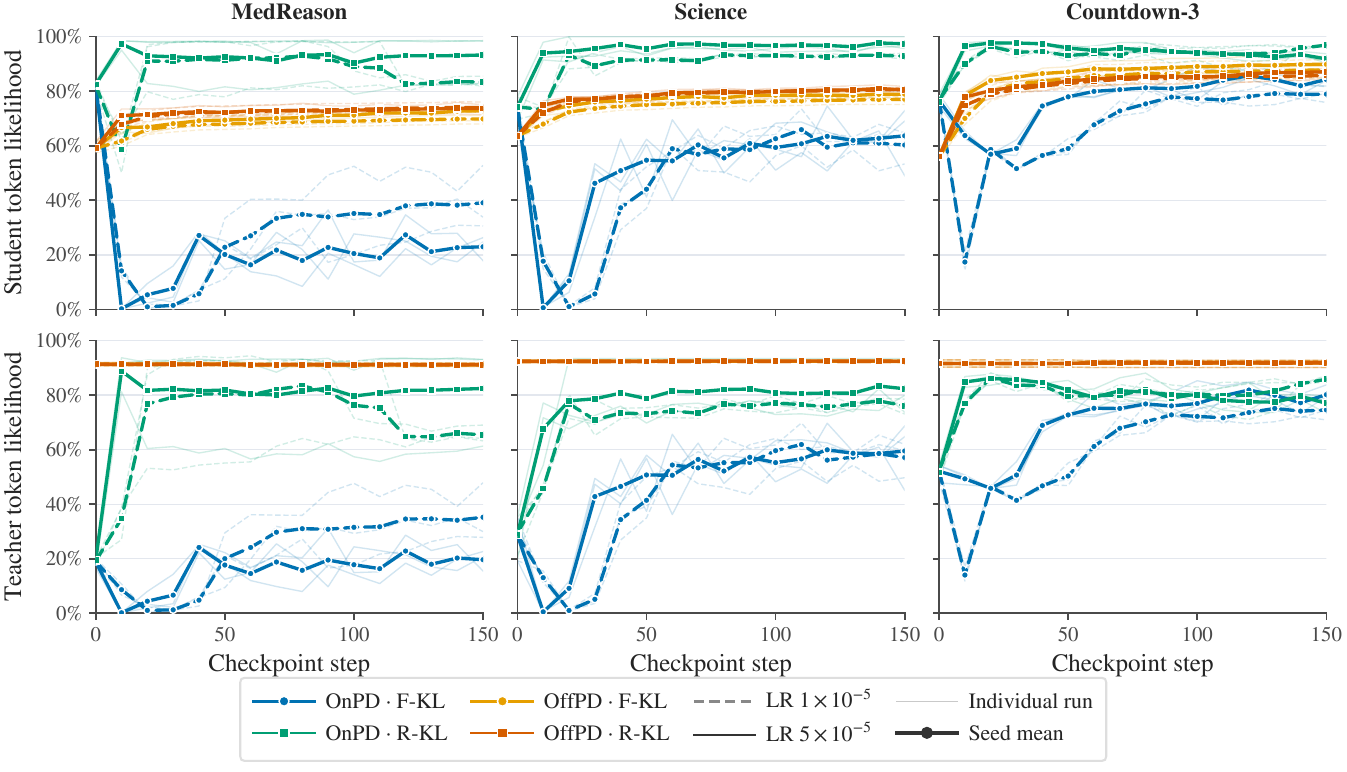}
    \caption{\textbf{Teacher and student likelihoods on the trajectories generated during OnPD and OffPD training.} Both the teacher and student likelihoods were computed at temperature 1.0.}
    \label{fig:app:likelihoods_over_training}
\end{figure}

\clearpage
\subsection{A Low Forward KL Can Conceal High-Entropy On-Policy Rollouts}
\label{app:onpd-fkl-entropy-collapse}

Forward KL can be expressed as:
\begin{equation}
    D_{\mathrm{KL}}(T\Vert S)
    = \mathrm{CE}(T,S) - H(T),
\end{equation}
where $T$ and $S$ denote the teacher and student next-token distributions on the sampled prefix. Consequently, a small forward KL can conceal both high cross-entropy and teacher entropy. We observed a particularly clear instance of this effect when distilling \texttt{Llama-3.2-1B-Instruct} on student-generated trajectories.

\paragraph{Setup.}
We analyze OnPD with the forward-KL objective on Science and MedReason, using dataset-specific trained Llama-3.1-8B teachers. The student is fully fine-tuned with a learning rate of $1\times10^{-5}$, an effective batch size of 32, student sampling temperature 1.0 and teacher distribution temperature 0.5. Every 10 steps, we evaluate token-level $D_{\mathrm{KL}}(T\Vert S)$, $H(T)$, and $\mathrm{CE}(T,S)$ on freshly sampled student trajectories. Metrics are binned in intervals of 50 completion tokens and reported separately for each dataset. Curves show the mean and range over 3 seeds at token positions reached by all three runs.

\begin{figure*}[h]
    \centering
    \includegraphics[width=\textwidth]{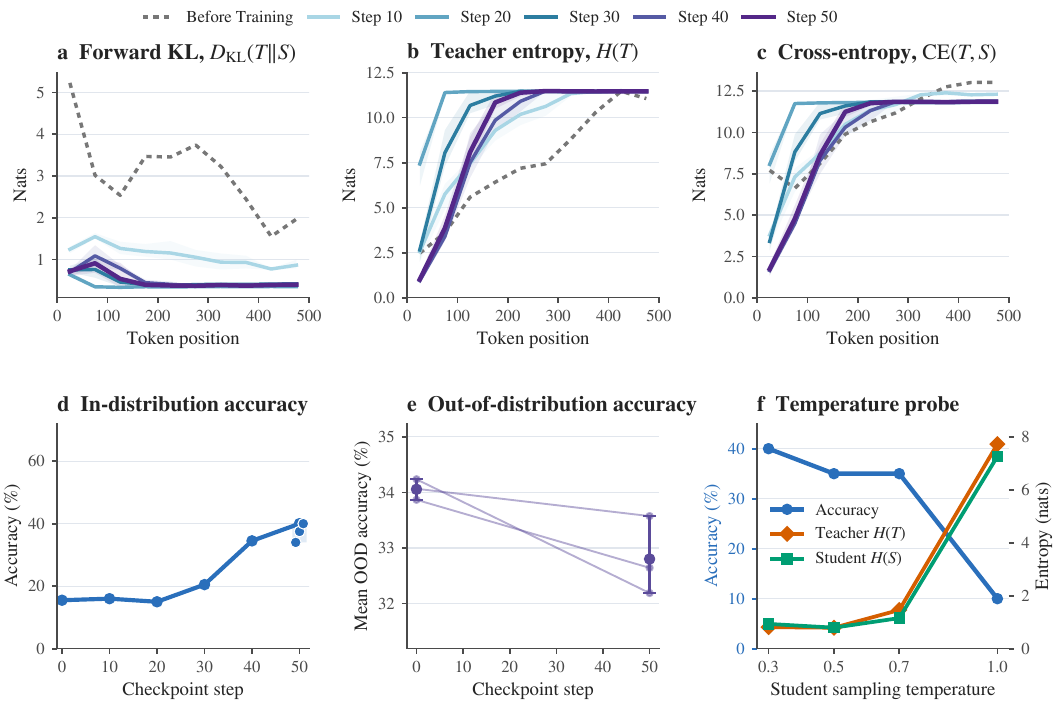}
    \caption{\textbf{Science OnPD forward-KL dynamics.} (a)--(c): token-position profiles at initialization (dotted grey) and every 10 training steps through step 50 (solid lines: three-seed means; shading: min--max ranges). (d): held-out accuracy for seed 42, with all three repeats at step 50. (e): OOD accuracy for individual seeds, with mean markers and min--max bars at initialization and step 50. (f): accuracy and teacher/student entropy when re-sampling the seed-42 checkpoint-50 model on the same 20 prompts at different temperatures.}
    \label{fig:app:onpd-fkl-entropy-collapse-science}
\end{figure*}


\begin{figure*}[h]
    \centering
    \includegraphics[width=\textwidth]{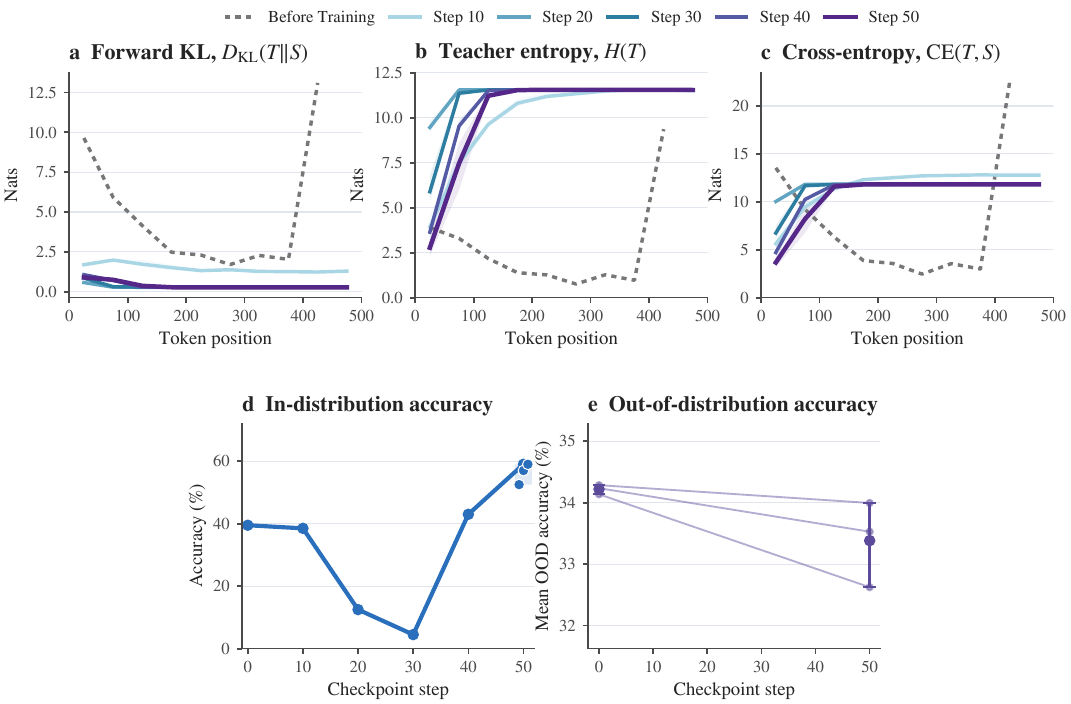}
    \caption{\textbf{MedReason OnPD forward-KL dynamics.} (a)--(c): token-position profiles of forward KL, teacher entropy, and cross-entropy. (d)--(e): held-out MedReason and OOD accuracy. Plotting conventions match Figure~\ref{fig:app:onpd-fkl-entropy-collapse-science}.}
    \label{fig:app:onpd-fkl-entropy-collapse-medreason}
\end{figure*}

\paragraph{Low KL masks high-entropy rollouts.}
\Cref{fig:app:onpd-fkl-entropy-collapse-science} shows that, during steps 10--20, forward KL becomes small while teacher entropy and cross-entropy approach the vocabulary scale (12 nats$\approx$163k effective token choices). 
Since $D_{\mathrm{KL}}(T\Vert S)=\mathrm{CE}(T,S)-H(T)$, low KL here reflects the near-cancellation of two large terms rather than confident predictions. 
Where teacher and student are nearly uniform and closely matched, the distillation gradient is weak, so the useful supervision signal is concentrated at earlier token positions. As training progresses, the low-entropy region advances along the trajectory. This suggests a self-reinforcing mechanism: improvements to early tokens produce prefixes on which the teacher can provide informative supervision further into the completion, enabling subsequent positions to improve in turn. 
This particularly acute example illustrates how on-policy learning may operate when student and teacher distributions have limited overlap: by progressively extending the student's access to useful teacher supervision and steering its trajectories to remain within teacher-supported regions for longer.

\paragraph{Learning and temperature sensitivity.}
Despite these high-entropy rollouts, step-50 target performs surprisingly well, averaging 37\% accuracy on Science and 56\% on MedReason, with OOD declines of only 1.3 and 0.8 percentage points. 
We hypothesize that this discrepancy reflects differences in the sampling temperature (evaluation uses $T=0.5$, whereas training use $T=1.0$). To measure this effect, we re-sample the Science checkpoint-50 model on the same 20 prompts at temperatures $\{0.3, 0.5, 0.7, 1.0\}$. Accuracy remains high up to temperature 0.7 but drops sharply at temperature 1.0, where both teacher and student entropy exceed 7 nats. This suggests that rollout instability does not necessarily imply a loss of capability, but rather a failure to reliably access that capability when sampling at temperature 1.0.

\clearpage
\subsection{Rollout Source Affects Teacher-Style Adoption}
\label{app:spanish}

\begin{wrapfigure}{r}{0.41\linewidth}
\centering
    \includegraphics[width=\linewidth]{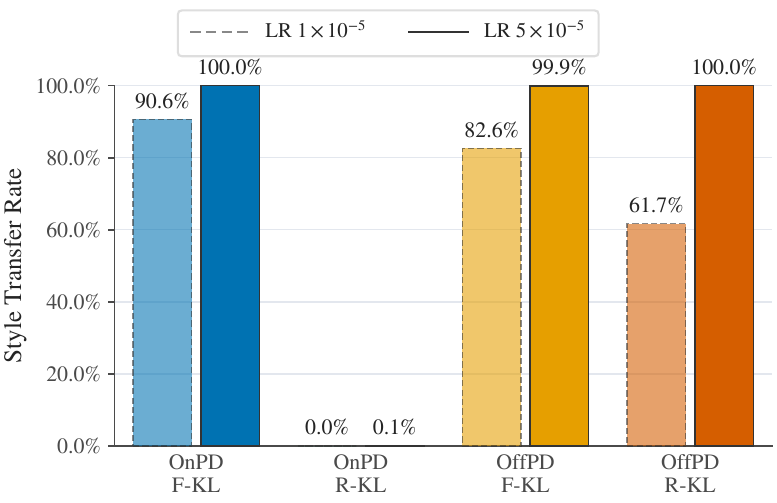}
    \caption{\textbf{Style transfer from teacher to student model.} OnPD with R-KL does not learn to generate traces in Spanish.}
    \label{fig:app:spanish_transfer}
\end{wrapfigure}

An important setting in which rollout policy may matter, even when final task performance is similar, is the transfer of incidental teacher behaviours. We therefore test whether rollout policy affects the adoption of a teacher's stylistic preferences under different distillation objectives.

\textbf{Experimental Setup.} We train Qwen2.5-1.5B student models on Science, using as teacher a Qwen2.5-3B model conditioned on a correct demonstration (similar to self-distillation setup~\citep{shenfeld_self-distillation_2026}), as well as an instruction to only speak in Spanish. We then evaluate the propensity of the student model to generate responses in Spanish, using GPT-5 to annotate student-generated trajectories at test time. Full training, evaluation, and annotation details are provided in Appendix~\ref{app:spanish_style_details}.

\textbf{Results}
Figure~\ref{fig:app:spanish_transfer} shows the results across rollout policies, KL objectives and learning rates. Teacher style transfers to the student for most setups, except OnPD with reverse KL. The mode-seeking behaviour of reverse KL makes it more robust to style transfer than forward KL: style tokens favoured by the teacher but not by the student contribute little to the reverse KL loss, but have a large effect under forward KL. However, off-policy rollouts still lead to style transfer from the teacher to the student, even under reverse KL. This suggests that training on states produced by the teacher can bias the student towards the teacher's style even when using the more robust reverse KL objective.

The robustness of OnPD with forward KL to adopting the teacher style likely depends on the ability of the teacher to provide meaningful supervision on student-generated trajectories. In this setting, the teacher still understands English, so when it is queried on English student-generated prefixes, its next-token distribution can retain an English mode. Reverse KL can therefore match a teacher-supported continuation in terms of final performance, without forcing the student into the incidental Spanish style.

\clearpage
\subsection{Rollout-policy temperature ablation}
\label{app:temperature_ablation}

\textbf{Setup.} In our main experiments we consistently rescale our trained teacher model with temperature 0.5, affecting both the the rollout generation temperature and the temperature used in the KL objective. As we explain in Appendix~\ref{app:details:trained_teacher}, this is to avoid degenerate generations which are occasionally produced at temperature 1.0. For the student, we consistently use temperature 1.0. In this appendix, we ablate whether using different generation temperatures for the student and the teacher could have significantly affected our results. 

We repeat the results from \Cref{fig:on_vs_offpd} on the MedReason dataset, where the teacher had a tendency to produce particularly degenerate generations at temperature 1.0 (probably due to the RLVR training). We use Llama-3.2-1B-Instruct student and the dataset-specific trained Llama-3.1-8B-Instruct teacher. We consider OnPD and OffPD, forward and reverse KL, learning rates $\{1\times10^{-5},5\times10^{-5}\}$, and seeds 42, 43, and 44. The original configuration samples OnPD trajectories from the
student at temperature $T_{\mathrm{gen}}=1.0$ and OffPD trajectories from the
teacher at $T_{\mathrm{gen}}=0.5$. In the ablation, we swap only these sampling
temperatures: OnPD now uses $T_{\mathrm{gen}}=0.5$, while OffPD now uses
$T_{\mathrm{gen}}=1.0$. Crucially, the temperatures used to compute the KL
objective are unchanged in every condition: the student distribution is scored
at $T_{S}=1.0$ and the teacher distribution at $T_{T}=0.5$. This isolates the
effect of the trajectory-generating distribution from the effect of temperature
on the training objective. We report held-out MedReason accuracy at step 150. Each bar in
Figure~\ref{fig:app:medreason-generation-temperature} is the mean over three
seeds, and error bars denote the standard error of the mean.

\begin{figure}[t]
    \centering
    \includegraphics[width=0.8\linewidth]{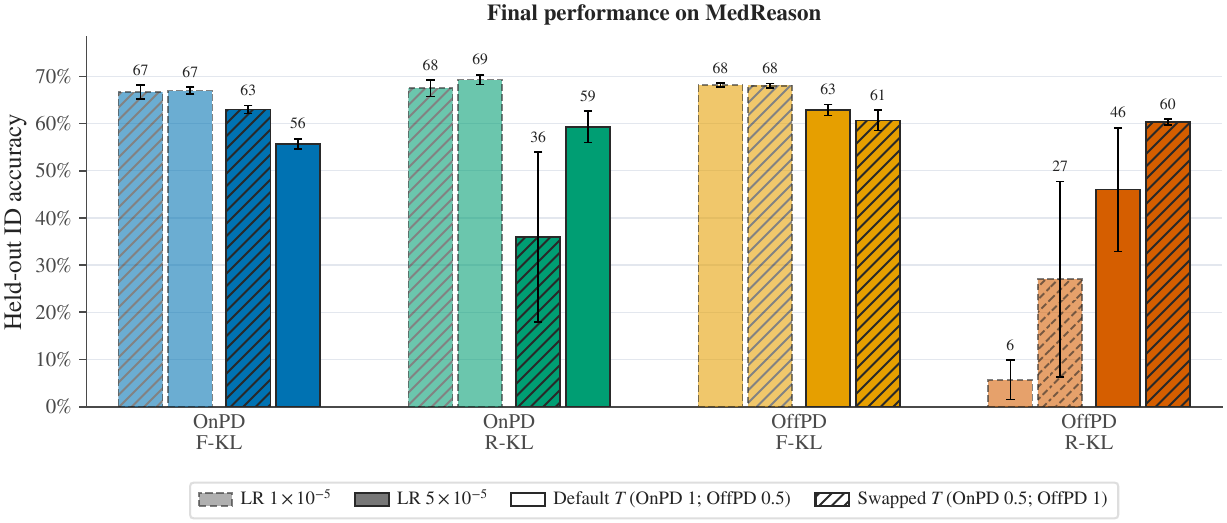}
    \caption{\textbf{Generation-temperature ablation on MedReason.} Modifying the generation temperature does not affect the best performance achieved by either OnPD or OffPD. Error bars mark the SEM over $N=3$ seeds.}
    \label{fig:app:medreason-generation-temperature}
\end{figure}

\textbf{Results.} Forward-KL performance is largely stable under the temperature swap. At the lower learning rate, OnPD changes from $67.0\%$ to $66.7\%$ and OffPD from
$68.2\%$ to $68.0\%$. At the higher learning rate, lowering the OnPD generation
temperature improves accuracy from $55.7\%$ to $63.0\%$, whereas raising the
OffPD generation temperature reduces accuracy from $62.8\%$ to $60.7\%$.
Thus, the low-learning-rate forward-KL comparison is not explained by the
different default sampling temperatures, although generation temperature has a
moderate effect on OnPD at the higher learning rate. 

Reverse KL is more sensitive to the generation distribution. For OnPD, the
temperature swap changes accuracy from $69.3\%$ to $67.5\%$ at
$1\times10^{-5}$, but from $59.3\%$ to $36.0\%$ at $5\times10^{-5}$; the latter
condition also has high seed-to-seed variability. For OffPD, raising the
generation temperature changes accuracy from $5.7\%$ to $27.0\%$ at
$1\times10^{-5}$ and from $46.0\%$ to $60.3\%$ at $5\times10^{-5}$. The
low-learning-rate OffPD comparison is itself highly variable across seeds.
Overall, the ablation indicates that the best achievable performance of OnPD and OffPD is not affected by the generation temperature, thus supporting our existing analysis.

The results in \cref{sec:rollout_spectrum} support these conclusions further, where in creating the rollout-policy spectrum we are mixing the student and teacher rollout policies when both are scaled with temperature 1.0. Hence, at $\lambda=1$ and $\lambda=-1$ we recover a comparison with OnPD and OffPD where both generation temperatures are set to 1.0.

\clearpage
\section{Derivation of the Token-Level KL Gradients}
\label{app:kl-gradients}

We derive the gradients at a fixed prefix, treating the teacher distribution as independent of \(\theta\). The forward and reverse KL objectives are
\begin{align}
D_{\mathrm{F-KL}}
&=
\sum_{u\in\mathcal{V}}
\pi_T(u)
\log\frac{\pi_T(u)}{\pi_S^\theta(u)},
\\
D_{\mathrm{R-KL}}
&=
\sum_{u\in\mathcal{V}}
\pi_S^\theta(u)
\log\frac{\pi_S^\theta(u)}{\pi_T(u)}.
\end{align}
For the student softmax,
\begin{align}
\frac{\partial \pi_S^\theta(u)}
{\partial (z_S^\theta)_v}
&=
\pi_S^\theta(u)
\left(
\mathbb{I}[u=v]-\pi_S^\theta(v)
\right),
\\
\frac{\partial\log\pi_S^\theta(u)}
{\partial (z_S^\theta)_v}
&=
\mathbb{I}[u=v]-\pi_S^\theta(v).
\end{align}

\paragraph{Forward KL.}
Differentiating the full vocabulary sum gives
\begin{align}
\frac{\partial D_{\mathrm{F-KL}}}
{\partial (z_S^\theta)_v}
&=
-\sum_{u\in\mathcal{V}}
\pi_T(u)
\left(
\mathbb{I}[u=v]-\pi_S^\theta(v)
\right)
\\
&=
-\pi_T(v)
+
\pi_S^\theta(v)
\sum_{u\in\mathcal{V}}\pi_T(u)
\\
&=
\pi_S^\theta(v)-\pi_T(v).
\end{align}
Applying the chain rule therefore yields
\[
\nabla_\theta D_{\mathrm{F-KL}}
=
\sum_{v\in\mathcal{V}}
\left(
\pi_S^\theta(v)-\pi_T(v)
\right)
\nabla_\theta(z_S^\theta)_v.
\]

\paragraph{Reverse KL.}
Let
\[
r(u)=\log\frac{\pi_S^\theta(u)}{\pi_T(u)}.
\]
Using \(\partial[x\log x]/\partial x=\log x+1\), we obtain
\begin{align}
\frac{\partial D_{\mathrm{R-KL}}}
{\partial (z_S^\theta)_v}
&=
\sum_{u\in\mathcal{V}}
\frac{\partial\pi_S^\theta(u)}
{\partial(z_S^\theta)_v}
\left(r(u)+1\right)
\\
&=
\pi_S^\theta(v)\left(r(v)+1\right)
-
\pi_S^\theta(v)
\sum_{u\in\mathcal{V}}
\pi_S^\theta(u)\left(r(u)+1\right).
\end{align}
Because
\[
\sum_{u\in\mathcal{V}}\pi_S^\theta(u)r(u)
=
D_{\mathrm{R-KL}}
\qquad\text{and}\qquad
\sum_{u\in\mathcal{V}}\pi_S^\theta(u)=1,
\]
this simplifies to
\[
\frac{\partial D_{\mathrm{R-KL}}}
{\partial(z_S^\theta)_v}
=
\pi_S^\theta(v)
\left[
\log\frac{\pi_S^\theta(v)}{\pi_T(v)}
-
D_{\mathrm{R-KL}}
\right].
\]
The parameter gradient is consequently
\[
\nabla_\theta D_{\mathrm{R-KL}}
=
\sum_{v\in\mathcal{V}}
\pi_S^\theta(v)
\left[
\log\frac{\pi_S^\theta(v)}{\pi_T(v)}
-
D_{\mathrm{R-KL}}
\right]
\nabla_\theta(z_S^\theta)_v.
\]
This expression assumes that \(\pi_T(v)>0\) wherever \(\pi_S^\theta(v)>0\); otherwise, the reverse KL is infinite.

\clearpage
\section{Rollout-Policy Sensitivity of Forward and Reverse KL Semi-Gradients}
\label{app:bullshit}

This appendix formalizes a stability distinction between forward- and
reverse-KL distillation.  The relevant object is the update used by the
algorithm, which stops gradients through sampled trajectories, rather than the
total derivative of an on-policy objective.  We show that the forward-KL
semi-gradient is uniformly Lipschitz in the rollout-induced distribution over
prefixes.  Reverse KL has no analogous distribution-free guarantee: its
semi-gradient can be arbitrarily sensitive to an arbitrarily small change in
the prefix distribution.  A bound for reverse KL is recovered only after
controlling student--teacher log-likelihood ratios.

\subsection{Setup}
\label{app:stability-setup}

We retain the notation introduced in Section~3.  Let $\rho$ denote the
autoregressive rollout policy.  For a prompt $x$ and completion
$y=(y_1,\ldots,y_{L_y})$, let $y_{<n}=(y_1,\ldots,y_{n-1})$ denote the output
prefix preceding token $y_n$, and write
\begin{equation}
    h_n := (x,y_{<n}).
    \label{eq:prefix-notation}
\end{equation}
The rollout policy induces the distribution
\begin{equation}
    \rho(y\mid x)
    =
    \prod_{n=1}^{L_y}\rho(y_n\mid h_n)
    \label{eq:rollout-completion-distribution}
\end{equation}
over complete output sequences.

For
$D\in\{D_{\mathrm{F\text{-}KL}},D_{\mathrm{R\text{-}KL}}\}$, define the
rollout-conditioned distillation objective
\begin{equation}
    \mathcal L_D(\theta;\rho)
    :=
    \mathbb E_{x\sim p_{\mathrm{data}}}
    \mathbb E_{y\sim\rho(\cdot\mid x)}
    \left[
        \frac{1}{L_y}
        \sum_{n=1}^{L_y}
        D\left(
            \pi_S^\theta(\cdot\mid h_n),
            \pi_T(\cdot\mid h_n)
        \right)
    \right].
    \label{eq:rollout-conditioned-objective}
\end{equation}
Off-policy distillation corresponds to $\rho=\pi_T$, whereas on-policy
distillation corresponds to $\rho=\pi_S^\theta$.

When $\rho=\pi_S^\theta$, the objective depends on $\theta$ both through the
student distribution in the token-level divergence and through the sampled
completion.  As in our training procedure, we stop gradients through the
sampling process and treat $y$, including every prefix $h_n$, as fixed.  We
denote the resulting semi-gradient by
\begin{equation}
    \bar\nabla_\theta\mathcal L_D(\theta;\rho)
    :=
    \mathbb E_{x\sim p_{\mathrm{data}}}
    \mathbb E_{y\sim\rho(\cdot\mid x)}
    \left[
        \frac{1}{L_y}
        \sum_{n=1}^{L_y}
        \nabla_\theta
        D\left(
            \pi_S^\theta(\cdot\mid h_n),
            \pi_T(\cdot\mid h_n)
        \right)
    \right],
    \label{eq:rollout-conditioned-semigradient}
\end{equation}
where the distributions inside the expectations are held fixed during
differentiation.  If $\rho$ does not depend on $\theta$, this semi-gradient is
the ordinary gradient of $\mathcal L_D(\theta;\rho)$.

Let $z_S^\theta(h_n)\in\mathbb R^{|\mathcal V|}$ denote the student logits at
prefix $h_n$, so that
\begin{equation}
    \pi_S^\theta(\cdot\mid h_n)
    =
    \operatorname{softmax}\!\left(z_S^\theta(h_n)\right),
\end{equation}
and define the corresponding logit Jacobian by
\begin{equation}
    J_\theta(h_n)
    :=
    \frac{\partial z_S^\theta(h_n)}{\partial\theta}.
    \label{eq:student-logit-jacobian}
\end{equation}
The parameter gradient of either token-level divergence can therefore be
written as
\begin{align}
    &\nabla_\theta
    D\left(
        \pi_S^\theta(\cdot\mid h_n),
        \pi_T(\cdot\mid h_n)
    \right)\nonumber\\
    &\qquad=
    J_\theta(h_n)^\top
    \nabla_{z_S^\theta(h_n)}
    D\left(
        \pi_S^\theta(\cdot\mid h_n),
        \pi_T(\cdot\mid h_n)
    \right).
    \label{eq:token-gradient-chain-rule}
\end{align}

\subsection{A rollout-stability bound for forward KL}
\label{app:forward-rollout-stability}

At any fixed prefix, the forward-KL gradient with respect to the student
logits is
\begin{align}
    &\nabla_{z_S^\theta(h_n)}
    D_{\mathrm{F\text{-}KL}}\left(
        \pi_S^\theta(\cdot\mid h_n),
        \pi_T(\cdot\mid h_n)
    \right)\nonumber\\
    &\qquad=
    \pi_S^\theta(\cdot\mid h_n)
    -\pi_T(\cdot\mid h_n).
    \label{eq:forward-kl-logit-gradient}
\end{align}
For conciseness, define the forward-KL parameter-gradient contribution of a
complete sampled trajectory as
\begin{align}
    g_{\mathrm{F\text{-}KL}}^\theta(x,y)
    :=\frac{1}{L_y}\sum_{n=1}^{L_y}
    J_\theta(h_n)^\top
    \Bigl(
        \pi_S^\theta(\cdot\mid h_n)
        -\pi_T(\cdot\mid h_n)
    \Bigr).
    \label{eq:forward-trajectory-gradient}
\end{align}
Equation~\ref{eq:rollout-conditioned-semigradient} then gives
\begin{equation}
    \bar\nabla_\theta
    \mathcal L_{D_{\mathrm{F\text{-}KL}}}(\theta;\rho)
    =
    \mathbb E_{x\sim p_{\mathrm{data}}}
    \mathbb E_{y\sim\rho(\cdot\mid x)}
    \left[g_{\mathrm{F\text{-}KL}}^\theta(x,y)\right].
    \label{eq:forward-semigradient-trajectory-form}
\end{equation}

We use
$\mathrm{TV}(P,Q)=\tfrac12\lVert P-Q\rVert_1$.  When applied to
$\rho(\cdot\mid x)$ and $\rho'(\cdot\mid x)$, this is the total-variation
distance between the distributions over complete output sequences induced by
the two rollout policies for the same prompt $x$.

\begin{theorem}[Forward-KL rollout stability]
\label{thm:forward-rollout-stability}
Suppose that there exists $B>0$ such that, for every prompt $x$, every
completion $y$ having positive probability under $\rho(\cdot\mid x)$ or
$\rho'(\cdot\mid x)$, every token position $n$, and every vector
$v\in\mathbb R^{|\mathcal V|}$,
\begin{equation}
    \left\lVert J_\theta(h_n)^\top v\right\rVert_2
    \le B\lVert v\rVert_2.
    \label{eq:jacobian-bound}
\end{equation}
Equivalently, $B$ uniformly bounds the largest singular value of the
student-logit Jacobian over all prefixes encountered under either rollout
policy.  Then
\begin{align}
    &\left\lVert
        \bar\nabla_\theta
        \mathcal L_{D_{\mathrm{F\text{-}KL}}}(\theta;\rho)
        -
        \bar\nabla_\theta
        \mathcal L_{D_{\mathrm{F\text{-}KL}}}(\theta;\rho')
    \right\rVert_2\nonumber\\
    &\qquad\le
    2\sqrt2 B\,
    \mathbb E_{x\sim p_{\mathrm{data}}}
    \left[
        \mathrm{TV}\!\left(
            \rho(\cdot\mid x),
            \rho'(\cdot\mid x)
        \right)
    \right].
    \label{eq:forward-rollout-stability-bound}
\end{align}
Moreover, for every rollout policy $\rho$,
\begin{equation}
    \mathbb E_{x\sim p_{\mathrm{data}}}
    \mathbb E_{y\sim\rho(\cdot\mid x)}
    \left[
        \left\lVert
            g_{\mathrm{F\text{-}KL}}^\theta(x,y)
        \right\rVert_2^2
    \right]
    \le 2B^2.
    \label{eq:forward-gradient-second-moment}
\end{equation}
\end{theorem}

\begin{proof}
We first bound the forward-KL gradient at one prefix.  For any probability
vectors $p$ and $q$, $\lVert p-q\rVert_2\le\sqrt2$.  Applying
Equation~\ref{eq:jacobian-bound} with
\begin{equation}
    v=\pi_S^\theta(\cdot\mid h_n)-\pi_T(\cdot\mid h_n)
\end{equation}
and using Equation~\ref{eq:forward-kl-logit-gradient} gives the per-token
bound
\begin{align}
    &\left\lVert
        J_\theta(h_n)^\top
        \Bigl(
            \pi_S^\theta(\cdot\mid h_n)
            -\pi_T(\cdot\mid h_n)
        \Bigr)
    \right\rVert_2
    \le
    B\left\lVert
        \pi_S^\theta(\cdot\mid h_n)
        -\pi_T(\cdot\mid h_n)
    \right\rVert_2
    \le\sqrt2B.
    \label{eq:forward-token-gradient-bound}
\end{align}

We next pass from a single token to a complete sampled trajectory.  Substituting
Equation~\ref{eq:forward-trajectory-gradient} and applying the triangle
inequality to the sum over token positions yields
\begin{align}
    \left\lVert
        g_{\mathrm{F\text{-}KL}}^\theta(x,y)
    \right\rVert_2
    &=
    \left\lVert
        \frac{1}{L_y}\sum_{n=1}^{L_y}
        J_\theta(h_n)^\top
        \Bigl(
            \pi_S^\theta(\cdot\mid h_n)
            -\pi_T(\cdot\mid h_n)
        \Bigr)
    \right\rVert_2\\
    &\le
    \frac{1}{L_y}\sum_{n=1}^{L_y}
    \left\lVert
        J_\theta(h_n)^\top
        \Bigl(
            \pi_S^\theta(\cdot\mid h_n)
            -\pi_T(\cdot\mid h_n)
        \Bigr)
    \right\rVert_2\\
    &\le
    \frac{1}{L_y}\sum_{n=1}^{L_y}\sqrt2B
    =\sqrt2B.
    \label{eq:forward-trajectory-gradient-bound}
\end{align}
This holds pointwise for every $x$ and $y$.

We now compare the expectations of this trajectory-level gradient under two
rollout policies.  For any bounded vector-valued function $g$ and discrete
probability distributions $P$ and $Q$,
\begin{align}
    \left\lVert
        \mathbb E_P[g]-\mathbb E_Q[g]
    \right\rVert_2
    \le
    2\sup_y\lVert g(y)\rVert_2\,\mathrm{TV}(P,Q).
    \label{eq:vector-tv-bound}
\end{align}
Applying Equation~\ref{eq:vector-tv-bound} conditionally for each prompt,
with $P=\rho(\cdot\mid x)$, $Q=\rho'(\cdot\mid x)$, and
$g(y)=g_{\mathrm{F\text{-}KL}}^\theta(x,y)$, yields
\begin{align}
    &\left\lVert
        \mathbb E_{y\sim\rho(\cdot\mid x)}
        \left[g_{\mathrm{F\text{-}KL}}^\theta(x,y)\right]
        -
        \mathbb E_{y\sim\rho'(\cdot\mid x)}
        \left[g_{\mathrm{F\text{-}KL}}^\theta(x,y)\right]
    \right\rVert_2\nonumber\\
    &\qquad\le
    2\sqrt2B\,
    \mathrm{TV}\!\left(
        \rho(\cdot\mid x),
        \rho'(\cdot\mid x)
    \right).
\end{align}

Finally, using Equation~\ref{eq:forward-semigradient-trajectory-form} and
moving the norm inside the expectation over prompts gives
\begin{align}
    &\left\lVert
        \bar\nabla_\theta
        \mathcal L_{D_{\mathrm{F\text{-}KL}}}(\theta;\rho)
        -
        \bar\nabla_\theta
        \mathcal L_{D_{\mathrm{F\text{-}KL}}}(\theta;\rho')
    \right\rVert_2\\
    &=
    \left\lVert
        \mathbb E_{x\sim p_{\mathrm{data}}}
        \left[
            \mathbb E_{y\sim\rho(\cdot\mid x)}
            \left[g_{\mathrm{F\text{-}KL}}^\theta(x,y)\right]
            -
            \mathbb E_{y\sim\rho'(\cdot\mid x)}
            \left[g_{\mathrm{F\text{-}KL}}^\theta(x,y)\right]
        \right]
    \right\rVert_2\\
    &\le
    \mathbb E_{x\sim p_{\mathrm{data}}}
    \left[
        \left\lVert
            \mathbb E_{y\sim\rho(\cdot\mid x)}
            \left[g_{\mathrm{F\text{-}KL}}^\theta(x,y)\right]
            -
            \mathbb E_{y\sim\rho'(\cdot\mid x)}
            \left[g_{\mathrm{F\text{-}KL}}^\theta(x,y)\right]
        \right\rVert_2
    \right]\\
    &\le
    2\sqrt2B\,
    \mathbb E_{x\sim p_{\mathrm{data}}}
    \left[
        \mathrm{TV}\!\left(
            \rho(\cdot\mid x),
            \rho'(\cdot\mid x)
        \right)
    \right],
\end{align}
which proves Equation~\ref{eq:forward-rollout-stability-bound}.  The first
inequality above is the triangle inequality for an expectation; the second is
the fixed-prompt bound derived from Equation~\ref{eq:vector-tv-bound}.

For the second claim, Equation~\ref{eq:forward-trajectory-gradient-bound}
holds for every $x$ and $y$, so
\begin{equation}
    \left\lVert
        g_{\mathrm{F\text{-}KL}}^\theta(x,y)
    \right\rVert_2^2
    \le 2B^2.
\end{equation}
Taking the expectation of both sides under
$x\sim p_{\mathrm{data}}$ and $y\sim\rho(\cdot\mid x)$ proves
Equation~\ref{eq:forward-gradient-second-moment}.
\end{proof}

\paragraph{Interpretation.}
The first conclusion of Theorem~\ref{thm:forward-rollout-stability} states that
changing the rollout policy can change the expected forward-KL update only in
proportion to the total-variation distance between the resulting completion
distributions.  Its constant does not depend on student--teacher probability
ratios.  Thus, when two rollout policies generate similar distributions over
completions, their expected forward-KL updates must also be similar.  This is a
worst-case stability statement rather than rollout invariance: the guarantee
can be loose when the two completion distributions have total variation close
to one, or when the Jacobian bound $B$ is large.

The second conclusion bounds the second moment of the gradient contribution
from an individual sampled completion.  In particular, the forward-KL
objective cannot produce arbitrarily large gradient contributions solely
because the student assigns very little probability to a token preferred by
the teacher.  This distinguishes the gradient from the forward-KL objective
value, which is itself unbounded.  The result also bounds the variance by
\begin{equation}
    \mathbb E\left[
        \left\lVert
            g_{\mathrm{F\text{-}KL}}^\theta
            -\mathbb E[g_{\mathrm{F\text{-}KL}}^\theta]
        \right\rVert_2^2
    \right]
    \le 2B^2.
\end{equation}

These conclusions provide a possible explanation for the empirical robustness
of forward KL in Section~5.  Across the rollout-policy spectrum, forward-KL
training attains similar final performance despite substantial changes in the
source of the generated trajectories.  The theorem shows that this behaviour
is compatible with a uniformly controlled change in the expected update, and
that extreme student--teacher likelihood ratios do not by themselves create
unbounded forward-KL logit gradients.  This is also consistent with the
gradient-clipping ablation in Section~6, where forward KL remains substantially
more stable than reverse KL when clipping is removed.

The theorem does not, however, prove that different rollout policies must
produce similar models or similar final performance.  It controls a one-step
expected semi-gradient, and its bound depends on both the distance between the
complete-trajectory distributions and the largest singular value $B$ of the
student-logit Jacobian.  Nor does it directly explain the observed differences
in catastrophic forgetting or parameter-update sparsity, which the experiments
show are governed primarily by the learning rate.  The empirical results
should therefore be interpreted as consistent with the stability result,
rather than as a direct consequence of it.

\subsection{Reverse KL admits no distribution-free rollout-stability bound}
\label{app:reverse-rollout-sensitivity}

At a fixed prefix $h_n$, define the student--teacher log-likelihood ratio for
token $v\in\mathcal V$ by
\begin{equation}
    r_n(v)
    :=
    \log
    \frac{
        \pi_S^\theta(v\mid h_n)
    }{
        \pi_T(v\mid h_n)
    }.
    \label{eq:reverse-log-ratio-vector}
\end{equation}
We write $r_n=(r_n(v))_{v\in\mathcal V}$ for the corresponding vector over
the vocabulary.
The reverse-KL gradient with respect to the student logits is
\begin{align}
    &\nabla_{z_S^\theta(h_n)}
    D_{\mathrm{R\text{-}KL}}\left(
        \pi_S^\theta(\cdot\mid h_n),
        \pi_T(\cdot\mid h_n)
    \right)\nonumber\\
    &\qquad=
    \pi_S^\theta(\cdot\mid h_n)
    \odot
    \left[
        r_n
        -
        D_{\mathrm{R\text{-}KL}}\left(
            \pi_S^\theta(\cdot\mid h_n),
            \pi_T(\cdot\mid h_n)
        \right)\mathbf 1
    \right].
    \label{eq:reverse-kl-logit-gradient}
\end{align}
All operations in Equation~\ref{eq:reverse-kl-logit-gradient} are
coordinate-wise.  Analogously to
Equation~\ref{eq:forward-trajectory-gradient}, define the reverse-KL
parameter-gradient contribution of a complete sampled trajectory as
\begin{align}
    &g_{\mathrm{R\text{-}KL}}^\theta(x,y)\nonumber\\
    &\quad:=
    \frac{1}{L_y}\sum_{n=1}^{L_y}
    J_\theta(h_n)^\top
    \left\{
        \pi_S^\theta(\cdot\mid h_n)
        \odot
        \left[
            r_n
            -
            D_{\mathrm{R\text{-}KL}}\left(
                \pi_S^\theta(\cdot\mid h_n),
                \pi_T(\cdot\mid h_n)
            \right)\mathbf 1
        \right]
    \right\}.
    \label{eq:reverse-trajectory-gradient}
\end{align}
The reverse-KL semi-gradient is therefore
\begin{equation}
    \bar\nabla_\theta
    \mathcal L_{D_{\mathrm{R\text{-}KL}}}(\theta;\rho)
    =
    \mathbb E_{x\sim p_{\mathrm{data}}}
    \mathbb E_{y\sim\rho(\cdot\mid x)}
    \left[g_{\mathrm{R\text{-}KL}}^\theta(x,y)\right].
    \label{eq:reverse-semigradient-trajectory-form}
\end{equation}

Unlike its forward-KL counterpart, the vector in
Equation~\ref{eq:reverse-kl-logit-gradient} is not uniformly bounded over
student and teacher distributions.

\begin{theorem}[No distribution-free reverse-KL rollout-stability bound]
\label{thm:reverse-no-bound}
There is no finite constant $C$, independent of the student and teacher token
distributions, for which
\begin{align}
    &\left\lVert
        \bar\nabla_\theta
        \mathcal L_{D_{\mathrm{R\text{-}KL}}}(\theta;\rho)
        -
        \bar\nabla_\theta
        \mathcal L_{D_{\mathrm{R\text{-}KL}}}(\theta;\rho')
    \right\rVert_2\nonumber\\
    &\qquad\le
    C\,
    \mathbb E_{x\sim p_{\mathrm{data}}}
    \left[
        \mathrm{TV}\!\left(
            \rho(\cdot\mid x),
            \rho'(\cdot\mid x)
        \right)
    \right]
    \label{eq:reverse-impossible-bound}
\end{align}
holds for all rollout policies $\rho$ and $\rho'$.  This remains true for a
two-token vocabulary and a direct logit parameterization whose Jacobian has
largest singular value one.
\end{theorem}

\begin{proof}
It is sufficient to construct a counterexample for one fixed prompt $x$.
Consider two length-two completions, denoted by $y^{(0)}$ and
$y^{(\delta)}$.  At the initial prefix, which is shared by both completions,
let the student and teacher distributions coincide.  Their reverse-KL gradient
at that prefix is therefore zero.  At the second prefix of $y^{(0)}$, again let
the two distributions coincide.  Consequently,
\begin{equation}
    g_{\mathrm{R\text{-}KL}}^\theta(x,y^{(0)})=0.
    \label{eq:reverse-zero-trajectory}
\end{equation}

At the second prefix of $y^{(\delta)}$, let the vocabulary contain two tokens
and set
\begin{equation}
    \pi_S^\theta(\cdot\mid h_2)
    =\left(\tfrac12,\tfrac12\right),
    \qquad
    \pi_T(\cdot\mid h_2)
    =(\delta,1-\delta),
    \qquad 0<\delta<\tfrac12.
    \label{eq:reverse-counterexample-distributions}
\end{equation}
Use a direct logit parameterization, so that
$J_\theta(h_2)^\top=I$.  If
\begin{equation}
    r_1=\log\frac{1/2}{\delta},
    \qquad
    r_2=\log\frac{1/2}{1-\delta},
\end{equation}
then
\begin{equation}
    D_{\mathrm{R\text{-}KL}}\left(
        \pi_S^\theta(\cdot\mid h_2),
        \pi_T(\cdot\mid h_2)
    \right)
    =\frac{r_1+r_2}{2}.
\end{equation}
Substituting into Equation~\ref{eq:reverse-kl-logit-gradient} gives
\begin{align}
    &\nabla_{z_S^\theta(h_2)}
    D_{\mathrm{R\text{-}KL}}\left(
        \pi_S^\theta(\cdot\mid h_2),
        \pi_T(\cdot\mid h_2)
    \right)\nonumber\\
    &\qquad=
    \frac14
    \log\left(\frac{1-\delta}{\delta}\right)(1,-1).
    \label{eq:reverse-counterexample-logit-gradient}
\end{align}
The first-token contribution is zero and $L_y=2$.  Therefore,
\begin{align}
    \left\lVert
        g_{\mathrm{R\text{-}KL}}^\theta(x,y^{(\delta)})
    \right\rVert_2
    &=
    \frac12
    \left\lVert
        \nabla_{z_S^\theta(h_2)}
        D_{\mathrm{R\text{-}KL}}\left(
            \pi_S^\theta(\cdot\mid h_2),
            \pi_T(\cdot\mid h_2)
        \right)
    \right\rVert_2\\
    &=
    \frac{\sqrt2}{8}
    \left|\log\left(\frac{1-\delta}{\delta}\right)\right|,
    \label{eq:reverse-counterexample-trajectory-gradient}
\end{align}
which diverges as $\delta\downarrow0$.

Let $\Delta_y$ denote the point mass on completion $y$.  For any fixed
$\varepsilon\in(0,1)$, define
\begin{equation}
    \rho'(\cdot\mid x)=\Delta_{y^{(0)}},
    \qquad
    \rho(\cdot\mid x)
    =(1-\varepsilon)\Delta_{y^{(0)}}
      +\varepsilon\Delta_{y^{(\delta)}}.
    \label{eq:reverse-counterexample-rollouts}
\end{equation}
Their total-variation distance is
\begin{equation}
    \mathrm{TV}\!\left(
        \rho(\cdot\mid x),
        \rho'(\cdot\mid x)
    \right)
    =\varepsilon.
\end{equation}
Using Equations~\ref{eq:reverse-semigradient-trajectory-form},
\ref{eq:reverse-zero-trajectory}, and
\ref{eq:reverse-counterexample-rollouts}, the difference between their
expected reverse-KL updates is
\begin{align}
    &\left\lVert
        \bar\nabla_\theta
        \mathcal L_{D_{\mathrm{R\text{-}KL}}}(\theta;\rho)
        -
        \bar\nabla_\theta
        \mathcal L_{D_{\mathrm{R\text{-}KL}}}(\theta;\rho')
    \right\rVert_2\nonumber\\
    &\qquad=
    \varepsilon
    \left\lVert
        g_{\mathrm{R\text{-}KL}}^\theta(x,y^{(\delta)})
    \right\rVert_2\\
    &\qquad=
    \frac{\varepsilon\sqrt2}{8}
    \left|\log\left(\frac{1-\delta}{\delta}\right)\right|.
    \label{eq:reverse-counterexample-update-difference}
\end{align}
The rollout total-variation distance remains equal to $\varepsilon$, while
Equation~\ref{eq:reverse-counterexample-update-difference} becomes
arbitrarily large as $\delta\downarrow0$.  Hence no universal constant $C$ can
satisfy Equation~\ref{eq:reverse-impossible-bound}.
\end{proof}

The construction requires no zero-probability token: every teacher probability
is strictly positive for $\delta>0$.  It therefore applies to softmax language
models.  The theorem states that proximity between two rollout distributions
alone cannot uniformly control the difference between their expected
reverse-KL updates.  It does not assert that reverse-KL training is always
unstable or that either on- or off-policy rollouts are universally preferable.

\subsection{A reverse-KL bound under controlled likelihood ratios}
\label{app:reverse-bounded-ratios}

The absence of a distribution-free guarantee does not preclude a conditional
stability result.  Such a result follows when the range of the token-level
student--teacher log-likelihood ratios is bounded.

\begin{corollary}[Reverse-KL stability under a bounded log-ratio range]
\label{cor:reverse-bounded-ratio}
Suppose that the Jacobian condition in
Equation~\ref{eq:jacobian-bound} holds.  In addition, suppose that there
exists $R<\infty$ such that, at every prefix occurring under $\rho$ or
$\rho'$,
\begin{align}
    &\max_{v\in\mathcal V}
    \log\frac{\pi_S^\theta(v\mid h_n)}{\pi_T(v\mid h_n)}
    -
    \min_{v\in\mathcal V}
    \log\frac{\pi_S^\theta(v\mid h_n)}{\pi_T(v\mid h_n)}
    \le R.
    \label{eq:reverse-log-ratio-range}
\end{align}
Then
\begin{align}
    &\left\lVert
        \bar\nabla_\theta
        \mathcal L_{D_{\mathrm{R\text{-}KL}}}(\theta;\rho)
        -
        \bar\nabla_\theta
        \mathcal L_{D_{\mathrm{R\text{-}KL}}}(\theta;\rho')
    \right\rVert_2\nonumber\\
    &\qquad\le
    2BR\,
    \mathbb E_{x\sim p_{\mathrm{data}}}
    \left[
        \mathrm{TV}\!\left(
            \rho(\cdot\mid x),
            \rho'(\cdot\mid x)
        \right)
    \right].
    \label{eq:reverse-rollout-stability-bounded-ratio}
\end{align}
Moreover, for every rollout policy $\rho$,
\begin{equation}
    \mathbb E_{x\sim p_{\mathrm{data}}}
    \mathbb E_{y\sim\rho(\cdot\mid x)}
    \left[
        \left\lVert
            g_{\mathrm{R\text{-}KL}}^\theta(x,y)
        \right\rVert_2^2
    \right]
    \le B^2R^2.
    \label{eq:reverse-gradient-second-moment-bounded-ratio}
\end{equation}
\end{corollary}

\begin{proof}
For a fixed prefix, let
\begin{equation}
    \overline r_n
    :=
    \sum_{v\in\mathcal V}
    \pi_S^\theta(v\mid h_n)r_n(v)
    =
    D_{\mathrm{R\text{-}KL}}\left(
        \pi_S^\theta(\cdot\mid h_n),
        \pi_T(\cdot\mid h_n)
    \right).
    \label{eq:reverse-mean-log-ratio}
\end{equation}
Because $\overline r_n$ is a weighted average of the coordinates of
$r_n$, it lies between their minimum and maximum.  The assumption in
Equation~\ref{eq:reverse-log-ratio-range} therefore gives
\begin{equation}
    \left|r_n(v)-\overline r_n\right|\le R
    \qquad\text{for every }v\in\mathcal V.
    \label{eq:centered-log-ratio-bound}
\end{equation}
Using Equation~\ref{eq:reverse-kl-logit-gradient},
\begin{align}
    &\left\lVert
        \nabla_{z_S^\theta(h_n)}
        D_{\mathrm{R\text{-}KL}}\left(
            \pi_S^\theta(\cdot\mid h_n),
            \pi_T(\cdot\mid h_n)
        \right)
    \right\rVert_2\nonumber\\
    &\qquad\le
    \left\lVert
        \pi_S^\theta(\cdot\mid h_n)
        \odot(r_n-\overline r_n\mathbf1)
    \right\rVert_1\\
    &\qquad=
    \sum_{v\in\mathcal V}
    \pi_S^\theta(v\mid h_n)
    \left|r_n(v)-\overline r_n\right|\\
    &\qquad\le R.
    \label{eq:reverse-logit-gradient-bounded-ratio}
\end{align}
The Jacobian condition now implies that the norm of each token-level parameter
gradient is at most $BR$.  Applying the triangle inequality to the token
average in Equation~\ref{eq:reverse-trajectory-gradient} gives
\begin{equation}
    \left\lVert
        g_{\mathrm{R\text{-}KL}}^\theta(x,y)
    \right\rVert_2
    \le BR.
    \label{eq:reverse-trajectory-gradient-bounded-ratio}
\end{equation}
Applying the total-variation argument from the proof of
Theorem~\ref{thm:forward-rollout-stability} with
$g(y)=g_{\mathrm{R\text{-}KL}}^\theta(x,y)$ proves
Equation~\ref{eq:reverse-rollout-stability-bounded-ratio}.  Squaring
Equation~\ref{eq:reverse-trajectory-gradient-bounded-ratio} and taking
expectations proves Equation~\ref{eq:reverse-gradient-second-moment-bounded-ratio}.
\end{proof}

\paragraph{Interpretation.}
Theorem~\ref{thm:reverse-no-bound} provides a worst-case distinction between
the two KL directions.  For forward KL, rollout proximity alone controls the
difference between expected updates, subject only to the common Jacobian bound.
For reverse KL, the same rollout proximity is insufficient: a rollout can place
even a small amount of probability on completions containing a prefix at which
the student assigns substantial mass to a teacher-improbable token, producing
an arbitrarily large change in the expected update.  The two factors therefore
play distinct and complementary roles: the rollout policy determines which
prefixes are visited and how strongly they are weighted, while the local
student--teacher log-likelihood ratios determine how large their reverse-KL
gradient contributions can be.  A rollout-probability change of order
$\varepsilon$ can consequently be amplified by a local gradient of order
$\log(1/\delta)$, producing an update difference of order
$\varepsilon\log(1/\delta)$.

Extreme likelihood ratios may already be present at initialization, but the
result should not be interpreted solely as a statement about poor
initialization or unusually large student logits.  In the counterexample, the
student distribution is uniform; the gradient becomes large because the
teacher assigns vanishing probability to a student-supported token.  Such
disagreement may instead emerge during training or occur only at particular
prefixes that are emphasized by one rollout policy.

This distinction is consistent with the results in Section~5, where reverse-KL
performance varies substantially across the rollout-policy spectrum and is
less stable than forward-KL performance.  It is also consistent with the
gradient-clipping ablation in Section~6: removing clipping disproportionately
harms reverse KL, whose logit gradient contains an unbounded
student--teacher log-likelihood ratio.  Global gradient clipping bounds the
realized gradient norm and can therefore mitigate this failure mode, although
it does not impose the likelihood-ratio condition in
Corollary~\ref{cor:reverse-bounded-ratio} or guarantee similar update
directions under different rollout policies.

Neither result proves that reverse KL must be unstable in a particular run, or
that student-generated rollouts must outperform teacher-generated rollouts.
The empirical findings should instead be interpreted as instances of the
greater sensitivity permitted by the reverse-KL gradient.  The bounded-ratio
corollary additionally predicts that reverse-KL sensitivity should be
associated with prefixes having extreme student--teacher log-likelihood
ratios, suggesting a directly testable diagnostic for the observed training
instabilities.

\subsection{Relation to sequence-level KL divergences}
\label{app:sequence-level-kl}

There is also an exact structural relationship between KL direction and
rollout source.  To state it cleanly, consider a fixed prompt $x$ and
completions with a fixed generation horizon $H$, where $H$ denotes the number
of generated token positions.  Variable-length completions can be represented
by including an end-of-sequence token and treating the terminated state as
absorbing up to position $H$.

Using the prefix notation $h_n=(x,y_{<n})$, the teacher and student induce the
sequence distributions
\begin{equation}
    \pi_T(y\mid x)
    :=\prod_{n=1}^H\pi_T(y_n\mid h_n),
    \qquad
    \pi_S^\theta(y\mid x)
    :=\prod_{n=1}^H\pi_S^\theta(y_n\mid h_n).
    \label{eq:student-sequence-distribution}
\end{equation}
Their marginal distributions over the prefix preceding token $n$ are
\begin{equation}
    \pi_T(y_{<n}\mid x)
    :=\prod_{m=1}^{n-1}\pi_T(y_m\mid h_m),
     \qquad
    \pi_S^\theta(y_{<n}\mid x)
    :=\prod_{m=1}^{n-1}\pi_S^\theta(y_m\mid h_m).
    \label{eq:student-prefix-distribution}
\end{equation}
For $n=1$, these are empty products, so both distributions place unit mass on
the empty output prefix.

\begin{proposition}[Matched rollouts recover sequence-level KL]
\label{prop:sequence-chain-rule}
For every prompt $x$,
\begin{align}
D_{\mathrm{KL}}\!\left(\pi_T(\cdot\mid x)\Vert\pi_S^\theta(\cdot\mid x)\right)
&=
\sum_{n=1}^H
\mathbb E_{y_{<n}\sim\pi_T(\cdot\mid x)}
\!\left[
D_{\mathrm{F\text{-}KL}}\!\left(
\pi_S^\theta(\cdot\mid h_n),
\pi_T(\cdot\mid h_n)
\right)
\right],
\label{eq:sequence-forward-kl}\\
D_{\mathrm{KL}}\!\left(\pi_S^\theta(\cdot\mid x)\Vert\pi_T(\cdot\mid x)\right)
&=
\sum_{n=1}^H
\mathbb E_{y_{<n}\sim\pi_S^\theta(\cdot\mid x)}
\!\left[
D_{\mathrm{R\text{-}KL}}\!\left(
\pi_S^\theta(\cdot\mid h_n),
\pi_T(\cdot\mid h_n)
\right)
\right].
\label{eq:sequence-reverse-kl}
\end{align}
\end{proposition}

\begin{proof}
For the forward direction, the autoregressive factorization gives
\begin{align}
    &D_{\mathrm{KL}}\!\left(
        \pi_T(\cdot\mid x)\Vert\pi_S^\theta(\cdot\mid x)
    \right)\nonumber\\
    &\quad=
    \mathbb E_{y\sim\pi_T(\cdot\mid x)}
    \left[
        \log
        \frac{\pi_T(y\mid x)}{\pi_S^\theta(y\mid x)}
    \right]\\
    &\quad=
    \mathbb E_{y\sim\pi_T(\cdot\mid x)}
    \left[
        \sum_{n=1}^H
        \log
        \frac{\pi_T(y_n\mid h_n)}
             {\pi_S^\theta(y_n\mid h_n)}
    \right]\\
    &\quad=
    \sum_{n=1}^H
    \mathbb E_{y_{<n}\sim\pi_T(\cdot\mid x)}
    \left[
        \mathbb E_{y_n\sim\pi_T(\cdot\mid h_n)}
        \left[
            \log
            \frac{\pi_T(y_n\mid h_n)}
                 {\pi_S^\theta(y_n\mid h_n)}
        \right]
    \right]\\
    &\quad=
    \sum_{n=1}^H
    \mathbb E_{y_{<n}\sim\pi_T(\cdot\mid x)}
    \left[
        D_{\mathrm{F\text{-}KL}}\left(
            \pi_S^\theta(\cdot\mid h_n),
            \pi_T(\cdot\mid h_n)
        \right)
    \right].
\end{align}
The third equality first conditions on the prefix $y_{<n}$ and then averages
over the teacher distribution of the next token.  This proves
Equation~\ref{eq:sequence-forward-kl}.  Exchanging the student and teacher
distributions proves Equation~\ref{eq:sequence-reverse-kl}.
\end{proof}

Proposition~\ref{prop:sequence-chain-rule} identifies teacher rollouts with
forward KL and student rollouts with reverse KL as the two chain-rule
decompositions of sequence-level KL.  The crossed pairings do not, in general,
correspond to either sequence-level KL divergence.  This is an identity between
objective values, not an ordering of the objectives or their optimization
properties.  In particular, under student rollouts, we stop gradients through
the sampled prefixes.  Consequently, although the expected OnPD reverse-KL
objective equals sequence-level reverse KL in the fixed-horizon, summed-loss
setting, the update used in training is only a semi-gradient and need not equal
the total gradient of that sequence-level divergence.

Finally, Proposition~\ref{prop:sequence-chain-rule} uses sums of token-level
divergences.  For fixed $H$, the corresponding token-averaged objectives are
exactly $1/H$ times the sequence-level KL divergences.  In the experiments,
however, each variable-length completion is normalized by its realized length
$L_y$.  The implemented objective is therefore a length-weighted token-level
divergence and is not generally a constant multiple of sequence-level KL.

\clearpage
\section{Experimental Details}
\label{app:experimental_details}


\subsection{Datasets}
\label{app:experimental_details:datasets}
\subsubsection{Training datasets.}

We train on three reasoning datasets, each divided into a training split and a disjoint held-out evaluation split. Every training example consists of a task prompt and a worked demonstration. We use the task-specific prompts in Table~\ref{tab:dataset-prompts}. During teacher supervised fine-tuning (SFT), the vanilla task prompt is provided as input and the demonstration is used as the target completion. The held-out demonstrations are never used for training or provided during evaluation.

\textit{\textsc{Countdown-3}.}
We procedurally generated 5,000 training and 500 evaluation candidates using random seeds 42 and 43, respectively. Each problem contains three operands sampled with replacement from
$\{1,\ldots,10,12,15,20,25,50,75\}$ and a target obtained by combining all operands through addition, subtraction, multiplication, or exact integer division. We queried GPT-5 for worked demonstrations and retained only correctly solved, properly tagged examples. This produced 4,163 training examples and 426 held-out evaluation examples.

\textit{\textsc{Science}.}
The Science dataset contains four-option multiple-choice questions from SciKnowEval. Its fixed splits contain 2,674 training examples and 507 held-out evaluation examples. Each training example includes a worked reasoning trace followed by the correct option letter in an \texttt{<answer>} block.

\textit{\textsc{MedReason}.}
MedReason consists of medical multiple-choice questions from \texttt{UCSC-VLAA/MedReason}. From the prepared source partition, we selected candidate pools of 5,000 training and 500 evaluation questions. We generated worked demonstrations with GPT-5 and retained only responses that were properly formatted and selected the ground-truth option. The resulting splits contain 4,748 training examples and 471 held-out evaluation examples.

\begin{table*}[h]
\centering
\small
\setlength{\tabcolsep}{5pt}
\renewcommand{\arraystretch}{1.15}
\begin{tabular}{p{0.17\textwidth}p{0.20\textwidth}p{0.60\textwidth}}
\toprule
Dataset & System prompt & User prompt \\
\midrule
\textsc{Countdown-3}
&
You are a careful arithmetic solver.
&
Solve the Countdown arithmetic problem. Using the provided numbers, create an equation that equals the target. You may use the operations $+$, $-$, $\times$, and $/$, and each number must be used exactly once; every division must have an integer result. Show your work in \texttt{<think>...</think>} tags and return the solution as compact, comma-separated equations in \texttt{<answer>...</answer>} tags. For example, \texttt{<answer>2+3=5,5*4=20</answer>}. The question is appended as \texttt{Numbers: [\ldots]} and \texttt{Target: \ldots}.
\\
\addlinespace
\textsc{Science}
&
You are a careful solver of multiple-choice science questions.
&
Given a question and four options, select the correct answer. Reason step by step in \texttt{<think>...</think>} tags and place only the corresponding option letter (A, B, C, or D) in \texttt{<answer>...</answer>} tags. The question and answer choices are appended to this instruction.
\\
\addlinespace
\textsc{MedReason}
&
You are a careful medical reasoning assistant.
&
Given a medical multiple-choice question and its answer choices, select the correct answer. Reason step by step in \texttt{<think>...</think>} tags and place only the letter corresponding to the correct option in \texttt{<answer>...</answer>} tags; do not repeat the answer text. The question and answer choices are appended to this instruction.
\\
\bottomrule
\end{tabular}
\caption{Task prompts used for teacher SFT and policy-distillation experiments. During SFT, the worked demonstration is supplied as the target completion rather than included in the input prompt.}
\label{tab:dataset-prompts}
\end{table*}

\paragraph{In-Distribution Evaluation Procedure.} To ensure computational efficiency, unless otherwise stated for our evaluations we use the first 200 examples from each dataset, giving 600 generated responses per evaluated checkpoint. We generate one response per prompt from the student at temperature 0.5 and top-$p=1.0$. The maximum prompt length is 2,048 tokens for Countdown-3 and Science and 4,096 tokens for MedReason; the maximum completion length is 2,048 tokens for all three datasets. We report the accuracy on the dataset matching the student's training dataset. Unless otherwise stated, the in-distribution evaluation uses temperature 0.5.

\subsubsection{OOD Evaluation Datasets.}
\label{app:ood_eval}
To measure out-of-distribution (OOD) capability retention, we evaluate our trained student models on 7 benchmarks spanning three broad categories: knowledge and factual reliability (MMLU-Pro \citep{NEURIPS2024_ad236edc} and TruthfulQA \citep{lin-etal-2022-truthfulqa}), instruction following and code generation (IFEval \citep{zhou2023instructionfollowingevaluationlargelanguage} and HumanEval \citep{chen2021evaluating}), and social reasoning and safety (EQ-Bench \citep{paech2024eqbenchemotionalintelligencebenchmark}, BBQ \citep{parrish-etal-2022-bbq}, and ToxiGen \citep{hartvigsen-etal-2022-toxigen}).

We use the interface provided by the Evaluation Harness \citep{eval-harness}, ensuring a standardised evaluation protocol. We evaluate 200 examples each from TruthfulQA-MC1, IFEval, ToxiGen, and BBQ; all 164 HumanEval-Instruct examples; and all 171 EQ-Bench examples. For MMLU-Pro, the 200-example limit is applied independently to each of its 14 subject areas, yielding 2,800 examples. In total, this gives 3,935 examples across seven benchmarks. We use each benchmark's standard metric: extracted exact match for MMLU-Pro, accuracy for TruthfulQA-MC1, ToxiGen, and BBQ, pass@1 for HumanEval-Instruct, strict prompt-level accuracy for IFEval, and the normalized EQ-Bench score. The aggregate OOD score is the unweighted mean of these seven task scores.

\subsection{Models}

\subsubsection{Student models}
All primary experiments use \texttt{Meta-Llama-3.2-1B-Instruct} as the student model. Each run begins from the same checkpoint and trains a separate student for each dataset, data-collection policy, KL direction, learning rate, and random seed. The student is fully fine-tuned rather than adapted with LoRA; all teacher parameters remain frozen. Each student is paired with the teacher from the same family described above. These models use a compatible tokenizer and vocabulary, allowing the forward- and reverse-KL objectives to compare their token distributions directly. Both models receive the same vanilla task prompt and chat-template formatting. The distinction between on- and off-policy training is therefore solely the model used to generate the training trajectory: the student generates trajectories for on-policy distillation, whereas the frozen teacher generates them for off-policy distillation.

\subsubsection{Teacher Model Training}
\label{app:details:trained_teacher}

We construct a separate teacher for each dataset and model family using two stages: supervised fine-tuning (SFT) on worked demonstrations, followed, where applicable, by reinforcement learning with Group Relative Policy Optimization (GRPO). Teachers are therefore task-specific rather than jointly trained across datasets.

\paragraph{Supervised fine-tuning.}
We first fine-tune each teacher backbone on the training split of the corresponding dataset. The input is the vanilla task prompt from Table~\ref{tab:dataset-prompts}, without a demonstration in context, and the target completion is the associated worked demonstration. Thus, SFT trains the teacher to reproduce both the reasoning trace and the formatted final answer. For Countdown-3 and MedReason, these targets are the verified GPT-5 demonstrations described above; for Science, we use the worked demonstrations included in the processed training split.

We train using token-level cross-entropy over the demonstration tokens for 500 optimizer steps. We use a learning rate of \(5\times10^{-5}\), a cosine learning-rate schedule with 10 warm-up steps, an effective batch size of 32, and gradient clipping at norm 1.0. Prompts and completions are truncated to at most 2,048 tokens each. For computational efficiency, all teachers are adapted with LoRA of rank 256 and scaling parameter 256, applied to the query, key, value, output, gate, up, and down projection matrices, with no LoRA dropout or bias. Separate SFT runs are performed for Countdown-3, Science, and MedReason.

\paragraph{GRPO refinement.}
For Science and MedReason, where the performance after the SFT stage has not yet saturated as in the case of Countdown-3, we further refine the dataset-specific SFT checkpoints using GRPO. GRPO receives only the vanilla task prompt and samples eight candidate completions per prompt at temperature 1.0 and nucleus-sampling threshold \(p=0.95\). The reward is the sum of task correctness and formatting rewards. Correct answers receive a reward of 2.0, using exact option accuracy for Science and MedReason. Additional rewards encourage exactly one pair of \texttt{<think>} and \texttt{<answer>} tags, the correct ordering of these tags, and strict adherence to the required response structure; each formatting component contributes at most 0.5. We use group-normalized rewards with the DAPO loss and no reference-model KL penalty (\(\beta=0\)).

GRPO training runs for 500 optimizer steps with a learning rate of \(10^{-5}\), a cosine schedule, a warm-up ratio of 0.1, an eight-bit AdamW optimizer, and gradient clipping at norm 1.0. The per-device batch size is 8 with 16 gradient-accumulation steps. Each rollout batch contains 32 completions for Science and 48 for MedReason. The maximum prompt and completion lengths are 2,048 and 824 tokens, respectively. GRPO also uses rank-256 LoRA with scaling parameter 256 over the same projection matrices as SFT; models are loaded in four-bit precision, while optimization uses bfloat16. Generation is performed with the colocated vLLM backend.

\paragraph{Selected teachers.} Teachers are paired with students from the same model family; Llama-3.1 teachers are paired with Llama-3.2 students because they share a tokenizer and vocabulary. For Science and MedReason, the selected teachers are the GRPO-refined checkpoints, generally at step 500; the selected Qwen2.5 MedReason and Llama-3 Science checkpoints are taken at step 450 due to better performance. For Countdown-3, we use the step-500 SFT checkpoint directly, without an additional GRPO stage. We define the fixed teacher policy as \(\pi_T(\cdot\mid h)=\operatorname{softmax}(z_T(h)/0.5)\) and the student policy as \(\pi_S^\theta(\cdot\mid h)=\operatorname{softmax}(z_S^\theta(h)/1.0)\). These definitions are used consistently both for trajectory generation and when computing the token-level KL objective. Accordingly, OffPD samples trajectories from \(\pi_T\), whereas OnPD samples trajectories from \(\pi_S^\theta\); in both conditions, the loss compares the same \(\pi_T\) and \(\pi_S^\theta\) at the resulting prefixes. Temperature is therefore not varied as an additional experimental factor between OnPD and OffPD: the sole intervention is whether the rollout distribution is \(\pi_T\) or \(\pi_S^\theta\). We use temperature 0.5 for the teacher because preliminary sampling at temperature 1.0 occasionally produced degenerate continuations near the end of long trajectories. We ablate the effect of this choice in Appendix~\ref{app:temperature_ablation}. We characterise the performance of the selected teacher models after this training procedure in \Cref{tab:trained-teacher-performance}.

\begin{table}[t]
    \centering
    \caption{
        In-distribution performance of the trained teachers.
        Results use 200 test examples per dataset, temperature $0.5$,
        one sampled response per prompt, and a 2,048-token completion limit.
        Response lengths are reported as mean $\pm$ standard deviation.
    }
    \label{tab:trained-teacher-performance}
    \small
    \setlength{\tabcolsep}{8pt}
    \begin{tabular}{llrr}
        \toprule
        \textbf{Teacher} & \textbf{Dataset}
        & \textbf{Accuracy (\%)} & \textbf{Response length (tokens)} \\
        \midrule
        Qwen2.5-7B
            & MedReason   & 77.5  & $90.95 \pm 28.58$ \\
            & Science     & 72.0  & $288.82 \pm 90.38$ \\
            & Countdown-3 & 95.5  & $168.42 \pm 316.40$ \\

        \midrule
        Llama-3.1-8B
            & MedReason   & 85.5  & $93.74 \pm 30.81$ \\
            & Science     & 67.5  & $134.65 \pm 50.77$ \\
            & Countdown-3 & 99.0  & $98.82 \pm 78.54$ \\
        \bottomrule
    \end{tabular}
\end{table}

\subsection{Training Details}
\label{app:details:training}


The most important hyperparameters are presented in \Cref{tab:student-training-hyperparameters}. We fully fine-tune the student for 150 optimizer steps, using one sampled completion per prompt. 
We choose the 150-step training budget based on the learning curves in \Cref{fig:app:learning_curves}, which show that target-task performance and forgetting have largely stabilised by this point.
We use the same training budget across configurations.
Each optimizer step uses an effective batch of 32 prompts, implemented with a per-device batch size of 1 and 32 gradient-accumulation steps. Prompts are sampled without replacement within each batch, but may recur across optimizer steps. We use fused AdamW with $\beta_1=0.9$, $\beta_2=0.999$, $\epsilon=10^{-8}$, no weight decay, and gradient clipping at norm 1.0. The learning rate is held constant following 10 linear warm-up steps. Our main experiments consider learning rates in $\{1\times10^{-5},5\times10^{-5}\}$; the dedicated Countdown-3 learning-rate study additionally considers $\{2,3,4,6\}\times10^{-5}$. Results are reported over random seeds 42, 43, and 44.

We train with either forward KL, $\mathrm{KL}(p_{\mathrm{teacher}}\|p_{\mathrm{student}})$, or reverse KL, $\mathrm{KL}(p_{\mathrm{student}}\|p_{\mathrm{teacher}})$. The student and teacher distributions use temperatures 1.0 and 0.5, respectively. The loss is computed over the full vocabulary at every non-padding completion position and then averaged over completion tokens and examples. The teacher remains frozen throughout training.

For on-policy distillation, trajectories are sampled from the student at temperature 1.0. For off-policy distillation, trajectories are sampled from the teacher at temperature 0.5. Both use $p=1.0$ nucleus sampling and vanilla task prompts without demonstrations. The maximum prompt length is 2,048 tokens. Maximum completion lengths are 512 tokens for Countdown-3 and 1,024 tokens for Science and MedReason. Gradient checkpointing is enabled for on-policy runs and disabled for off-policy runs. Students are trained without LoRA, so all student parameters are updated.

\begin{table*}[t]
\centering
\small
\setlength{\tabcolsep}{5pt}
\renewcommand{\arraystretch}{1.12}
\begin{tabularx}{\textwidth}{p{0.25\textwidth}p{0.27\textwidth}X}
\toprule
Hyperparameter & Value & Notes \\
\midrule

Optimization steps
& 150
& Reported primary comparisons use the checkpoint at step 150. \\

Optimizer
& Fused AdamW
& $\beta_1=0.9$, $\beta_2=0.999$, and $\epsilon=10^{-8}$. \\

Weight decay
& 0
&  \\

Learning-rate schedule
& Constant with linear warm-up
& 10 warm-up steps. \\

Per-device batch size
& 1
&  \\

Gradient accumulation
& 32 steps
& Gives an effective batch size of 32 prompts per optimizer step. \\

Trajectories per prompt
& 1
& One completion is sampled for every selected prompt. \\

Maximum gradient norm
& 1.0
&  \\

Student temperature
& 1.0
& Used for student sampling and student logits in the KL objective. \\

Teacher temperature
& 0.5
& Used for teacher sampling and teacher logits in the KL objective. \\

On-policy trajectory source
& Student
& Sampled at temperature 1.0. \\

Off-policy trajectory source
& Teacher
& Sampled at temperature 0.5. \\

Top-$p$
& 1.0
& Used for both on- and off-policy trajectory sampling. \\

Maximum prompt length
& 2,048 tokens
& Shared across datasets. \\

Maximum completion length
& 512 / 1,024 tokens
& 512 for Countdown-3; 1,024 for Science and MedReason. \\

Loss reduction
& Mean over completion tokens
& Computed over the full vocabulary at every non-padding completion position. \\

Gradient checkpointing
& On-policy: enabled; off-policy: disabled
& This difference is used for memory management. \\

\bottomrule
\end{tabularx}
\caption{Principal hyperparameters used for distillation training.}
\label{tab:student-training-hyperparameters}
\end{table*}

\subsection{Parameter-Update Sparsity}
\label{app:update-sparsity}

Following \citet{mukherjee_reinforcement_2025}, we measure update sparsity as the fraction of model parameters whose absolute change during training falls below a fixed threshold.
Let $\theta^{\mathrm{init}}$ and $\theta^{\mathrm{final}}$ denote the initial and final student parameters, and let $P$ be the total number of parameters.
We compute
\[
    S_{\tau}
    =
    \frac{1}{P}
    \sum_{i=1}^{P}
    \mathbb{I}
    \left[
        \left|
            \theta_i^{\mathrm{final}}
            -
            \theta_i^{\mathrm{init}}
        \right|
        < \tau
    \right].
\]
All students are fully fine-tuned, so this comparison includes the entire model rather than only an adapter or a subset of trainable parameters.
Higher values indicate that a larger fraction of parameters remain close to their initial values; they do not necessarily indicate exactly zero updates.

We use $\tau = 10^{-6}$ for the main results.
Appendix~A.6 reports the corresponding analysis at $\tau = 10^{-8}$, which preserves the main qualitative conclusions.

\subsection{Implementation of the Rollout-Policy Spectrum}
\label{app:rollout_spectrum}

At each generation step, the student and teacher are evaluated on the same prefix. The unconstrained rollout-policy score is
\begin{align}
(\hat z_\lambda)_v
&=
\frac{1}{2}
\left(
\log\pi_S(v)+\log\pi_T(v)
\right)
+
\frac{\lambda}{2}
\left(
\log\pi_S(v)-\log\pi_T(v)
\right)
\\
&=
\frac{1+\lambda}{2}\log\pi_S(v)
+
\frac{1-\lambda}{2}\log\pi_T(v).
\end{align}
The corresponding policy is obtained by normalising these scores:
\[
\pi_\lambda(v)
=
\frac{\exp((\hat z_\lambda)_v)}
{\sum_{u\in\mathcal{V}}\exp((\hat z_\lambda)_u)}.
\]
Thus, \(\lambda=-1\) recovers \(\pi_T\), \(\lambda=1\) recovers \(\pi_S\), and \(\lambda=0\) gives their normalised geometric midpoint.

For \(|\lambda|>1\), one of the coefficients above becomes negative. This can disproportionately amplify tokens that are unlikely under both policies but relatively less unlikely under the favoured policy. We prevent this using one-sided log-ratio clipping and a plausibility mask, similar in spirit to \citet{li2023contrastivedecodingopenendedtext}.

Define the token-level log-likelihood ratio
\[
r_v
=
\log\pi_S(v)-\log\pi_T(v),
\]
and the student and teacher plausibility masks
\begin{align}
M_S(v)
&=
\mathbb{I}\left[
\log\pi_S(v)
\geq
\max_{u\in\mathcal{V}}\log\pi_S(u)+\log\alpha
\right],
\\
M_T(v)
&=
\mathbb{I}\left[
\log\pi_T(v)
\geq
\max_{u\in\mathcal{V}}\log\pi_T(u)+\log\alpha
\right].
\end{align}
A token is therefore eligible for contrastive amplification only if its probability under the favoured policy is at least an \(\alpha\) fraction of that policy's maximum token probability.

We define the clipped student- and teacher-favoured bonuses as
\begin{align}
b_S(v)
&=
M_S(v)\min\!\left\{\max\{r_v,0\},c\right\},
\\
b_T(v)
&=
M_T(v)\min\!\left\{\max\{-r_v,0\},c\right\},
\end{align}
where \(c>0\) is the clipping threshold. The safeguarded sampling scores are then
\[
(\tilde z_\lambda)_v
=
\begin{cases}
(\hat z_\lambda)_v,
& |\lambda|\leq 1,\\[4pt]
\log\pi_S(v)
+
\displaystyle\frac{\lambda-1}{2}b_S(v),
& \lambda>1,\\[8pt]
\log\pi_T(v)
+
\displaystyle\frac{-\lambda-1}{2}b_T(v),
& \lambda<-1.
\end{cases}
\]
The final rollout policy is
\[
\pi_\lambda(v)
=
\operatorname{softmax}(\tilde z_\lambda)_v.
\]

The plausibility masks gate only the additional extrapolation bonus: tokens outside the mask retain their probability under the corresponding anchor policy but receive no amplification. This preserves \(\pi_{-1}=\pi_T\), \(\pi_1=\pi_S\), and continuity at the boundaries of the interpolation region. In our experiments, we use
\[
\alpha=0.2,
\qquad
c=2\log 3 \approx 2.2.
\]
For the range \(\lambda\in[-2,2]\), this limits the maximum pairwise-odds amplification relative to the anchor policy to a factor of three.

\subsection{RLVR after Distillation}
\label{app:downstream_rlvr}

\paragraph{Purpose and initialisation.}
We test whether the differences induced by distillation affect subsequent RLVR on a harder task. We initialise Llama-3.2-1B-Instruct from the step-150 Countdown-3 checkpoints for all combinations of OnPD/OffPD, forward/reverse KL, and distillation learning rates $\{1\times10^{-5},5\times10^{-5}\}$. Each configuration has three checkpoints from distillation seeds 42, 43, and 44, giving 24 RLVR runs. We train each checkpoint on Countdown-4 with the same RLVR configuration. The RLVR seed is fixed at 42, so the three trajectories per configuration reflect different distillation seeds rather than independently varied RLVR seeds.

\paragraph{Training data and reward.}
We use 5,000 procedurally generated Countdown-4 problems, with four operands sampled with replacement from $\{1,\ldots,10,12,15,20,25,50,75\}$. Targets lie between 1 and 10,000 and are constructed by combining all four operands through addition, subtraction, multiplication, or exact integer division. Dataset generation uses seed 42. Prompts follow the Countdown format in \Cref{tab:dataset-prompts}, rendered with the student's chat template, without demonstrations or teacher supervision. The reward is correctness-only: a completion receives $2$ if the Countdown verifier reduces the operand pool to the target through valid arithmetic steps, and $0$ otherwise. We add no separate formatting or length rewards.

\paragraph{Optimisation and generation.}
We use GRPO~\citep{shao_deepseekmath_2024} with the DAPO loss~\citep{yu_dapo_2025}, group-normalised rewards, and no reference-model KL penalty. Unlike the full-parameter distillation stage, RLVR trains fresh LoRA adapters on the frozen, four-bit-loaded checkpoint. We use rank and scaling parameter 256, zero dropout, and no bias adaptation, targeting the query, key, value, output, gate, up, and down projections. Training uses Unsloth with TRL's GRPO trainer and colocated vLLM generation; gradient checkpointing is enabled. Each prompt receives eight sampled completions at temperature $1.0$ and top-$p=0.95$. All runs use 300 optimisation steps and an RLVR learning rate of $2\times10^{-6}$, independently of the learning rate used during distillation. Full settings are summarised in \Cref{tab:downstream_rlvr_protocol}.

\begin{table*}[t]
\centering
\small
\setlength{\tabcolsep}{5pt}
\renewcommand{\arraystretch}{1.12}
\begin{tabularx}{\textwidth}{p{0.25\textwidth}p{0.27\textwidth}X}
\toprule
Setting & Value & Notes \\
\midrule
Optimisation steps & 300 & Starting from each step-150 distilled checkpoint. \\
RLVR learning rate & $2\times10^{-6}$ & Shared across all 24 runs. \\
Learning-rate schedule & Cosine & Warm-up over the first 10\% of steps (30 steps). \\
Optimizer & Eight-bit AdamW & Bfloat16 training. \\
Trainable parameters & LoRA adapters & Rank 256, scaling parameter 256, dropout 0; no bias adaptation. \\
Backbone weights & Frozen, loaded in four-bit precision & Gradient checkpointing enabled. \\
Per-device batch size & 8 & Gradient accumulation over 16 steps. \\
Completions per prompt & 8 & One reward-normalisation group per prompt. \\
Generation batch size & 48 completions & Six groups of eight completions. \\
Sampling & Temperature $1.0$, top-$p=0.95$ & Colocated vLLM generation. \\
Maximum prompt length & 2,048 tokens & Vanilla Countdown prompt with chat template. \\
Maximum completion length & 1,024 tokens & Shared across all configurations. \\
Loss and reward scaling & DAPO; group normalisation & No reference-model KL penalty ($\beta=0$). \\
Reward & $2$ for correctness, $0$ otherwise & No auxiliary reward terms. \\
Maximum gradient norm & $1.0$ & Global gradient clipping. \\
RLVR random seed & 42 & Distillation seeds are 42, 43, and 44. \\
Logging / checkpoint interval & 5 / 50 steps & Reward curves show logged training rewards. \\
\bottomrule
\end{tabularx}
\caption{Training settings for RLVR on Countdown-4 after Countdown-3 distillation. These settings are identical across rollout policies, KL objectives, and learning rates used in the preceding distillation stage.}
\label{tab:downstream_rlvr_protocol}
\end{table*}

\paragraph{Reported metrics.}
The curves in \Cref{fig:countdown4_downstream_rl_1e5,fig:app:countdown4_downstream_rl_5e5} show the logged mean correctness reward on training rollouts after distillation at learning rates $1\times10^{-5}$ and $5\times10^{-5}$, respectively, with each distilled checkpoint plotted separately and no additional smoothing. These are training rewards, not held-out accuracies: because the reward takes values $0$ or $2$, the corresponding fraction of correct rollouts is half the mean reward. These learning rates refer to distillation, not the fixed RLVR learning rate of $2\times10^{-6}$.

\begin{figure}[htbp]
    \centering
    \includegraphics[width=0.5\linewidth]{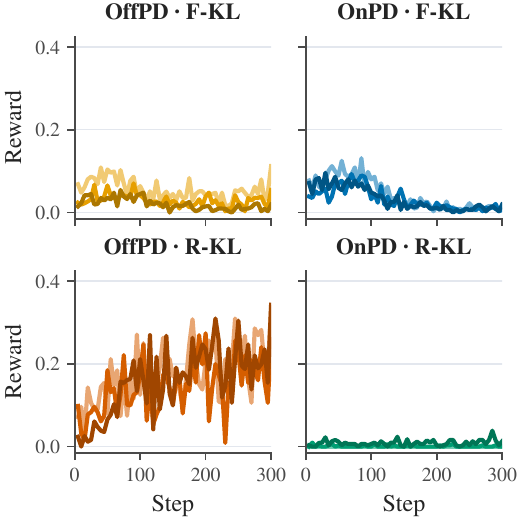}
    \caption{Countdown-4 RLVR rewards after distillation at LR $5\times10^{-5}$; each curve is a distillation seed.}
    \label{fig:app:countdown4_downstream_rl_5e5}
\end{figure}

\clearpage
\subsection{Spanish-Style Transfer Experiment}
\label{app:spanish_style_details}

\paragraph{Purpose and model setup.}
This experiment tests whether rollout source changes the transfer of an incidental teacher behaviour unrelated to task correctness. We distil a frozen Qwen2.5-3B-Instruct teacher into a Qwen2.5-1.5B-Instruct student on ScienceQA, conditioning the teacher to reason in Spanish. Unlike the main experiments, the teacher is not task-fine-tuned; it instead receives a correct demonstration for each training example. Deviations from the standard experimental setup are summarised in Table~\ref{tab:spanish_style_protocol}.

\paragraph{Teacher conditioning.}
For every training example, the teacher receives the Science task prompt followed by the example's correct worked demonstration and the instruction ``Now answer with a response of your own, including the thinking process.'' We additionally append the following style instruction to the teacher's final user message:
\begin{quote}
\small
Only speak in Spanish. This does not apply to the special \texttt{<think></think>} and \texttt{<answer></answer>} tokens or your final answer, but any reasoning or thoughts you produce should be strictly written in Spanish.
\end{quote}
The student never receives the demonstration or Spanish instruction. Teacher distributions are conditioned on both for every loss evaluation; OnPD samples completions from the student under the vanilla prompt, whereas OffPD samples them from the conditioned teacher.

\begin{table*}[t]
\centering
\small
\setlength{\tabcolsep}{5pt}
\renewcommand{\arraystretch}{1.12}
\begin{tabularx}{\textwidth}{p{0.25\textwidth}p{0.27\textwidth}X}
\toprule
Setting & Value & Notes \\
\midrule

Student
& \texttt{Qwen2.5-1.5B-Instruct}
& Fully fine-tuned on Science. \\

Teacher
& \texttt{Qwen2.5-3B-Instruct}
& Frozen base instruct model; no task-specific fine-tuning. \\

Experimental grid
& OnPD/OffPD $\times$ forward/reverse KL
& Learning rates $\{1\times10^{-5},5\times10^{-5}\}$. \\

Random seeds
& 42
& One run per condition. \\

Teacher conditioning
& Correct demonstration and Spanish instruction
& Used for teacher rollouts and teacher distributions. \\

Student conditioning
& Vanilla Science prompt
& Used for student rollouts, likelihoods, and evaluation. \\

Sampling temperature
& 1.0 for OnPD and OffPD
& The conditioned teacher therefore differs from the standard OffPD temperature of 0.5. \\

Distribution temperature
& 1.0 for student and teacher
& The teacher differs from the standard temperature of 0.5. \\

Maximum training completion length
& 2,048 tokens
& The standard Science experiments use 1,024 tokens. \\

Gradient checkpointing
& Enabled for OnPD and OffPD
& The standard OffPD runs disable gradient checkpointing. \\

Style evaluation
& 1,000 responses per condition
& 200 each from Science train and test, MedReason test, GSM8K test, and Countdown-3 test. \\

\bottomrule
\end{tabularx}
\caption{Experiment-specific settings for the Spanish-style transfer experiment. All unlisted training and evaluation settings follow the standard protocol.}
\label{tab:spanish_style_protocol}
\end{table*}

\paragraph{Language annotation.}
We use GPT-5 to independently classify each generated response as primarily Spanish or not, treating mathematical notation, formatting tags, code, and short answer labels as language-neutral. All 1,000 responses per condition received a valid annotation. We use one judge and one training run per condition.

\clearpage
\section{Checkpoint Sparsity Analysis}
\label{app:sparsity}

In \Cref{tab:checkpoint_learning_rates} we summarise the learning rate reported by each of the checkpoints analysed by \citet{mukherjee_reinforcement_2025}.

\begin{table*}[bh]
\centering
\caption{Learning rates used to produce the checkpoints in the main RL comparison and the explicit SFT comparison of \citet{mukherjee_reinforcement_2025}. ``From'' denotes the checkpoint against which parameter-update sparsity is measured. Rates refer to the language model or policy actor, using the peak rate where a schedule is reported.}
\label{tab:checkpoint_learning_rates}
\scriptsize
\setlength{\tabcolsep}{3pt}
\renewcommand{\arraystretch}{1.15}
\begin{tabularx}{\textwidth}{
    @{}
    l
    >{\raggedright\arraybackslash}X
    >{\raggedright\arraybackslash}X
    >{\raggedright\arraybackslash}p{2.4cm}
    l
    @{}
}
\toprule
Stage & From & Studied checkpoint & Training method & Model/actor LR \\
\midrule
\multicolumn{5}{@{}l}{\textit{SFT checkpoints (Appendix C)}} \\
\addlinespace[2pt]
SFT &
\path{meta-llama/Llama-3.1-8B} &
\path{allenai/Llama-3.1-Tulu-3-8B-SFT} &
Full-parameter SFT \citep{lambert2024tulu3} &
$5\times10^{-6}$ \\

SFT &
\path{meta-llama/Llama-3.1-70B} &
\path{allenai/Llama-3.1-Tulu-3-70B-SFT} &
Full-parameter SFT \citep{lambert2024tulu3} &
$2\times10^{-6}$ \\

SFT &
\path{Qwen/Qwen2.5-Math-7B} &
\path{PRIME-RL/Eurus-2-7B-SFT} &
Math-reasoning SFT \citep{cui2025process} &
$2\times10^{-5}$ \\

\addlinespace
\multicolumn{5}{@{}l}{\textit{RL and preference-optimization checkpoints (Table 1)}} \\
\addlinespace[2pt]
DPO &
\path{allenai/Llama-3.1-Tulu-3-8B-SFT} &
\path{allenai/Llama-3.1-Tulu-3-8B-DPO} &
Offline DPO \citep{lambert2024tulu3} &
$5\times10^{-7}$ \\

DPO &
\path{allenai/Llama-3.1-Tulu-3-70B-SFT} &
\path{allenai/Llama-3.1-Tulu-3-70B-DPO} &
Offline DPO \citep{lambert2024tulu3} &
$5\times10^{-7}$ \\

GRPO &
\path{deepseek-ai/deepseek-math-7b-instruct} &
\path{deepseek-ai/deepseek-math-7b-rl} &
On-policy GRPO \citep{shao_deepseekmath_2024} &
$1\times10^{-6}$ \\

RL-Zero &
\path{deepseek-ai/DeepSeek-V3-Base} &
\path{deepseek-ai/DeepSeek-R1-Zero} &
On-policy GRPO \citep{deepseek-ai_deepseek-r1_2025} &
Not disclosed \\

ORPO &
\path{mistralai/Mistral-7B-v0.1} &
\path{kaist-ai/mistral-orpo-beta} &
Offline, reference-free ORPO \citep{hong2024orpo} &
$8\times10^{-6}$ \\

KTO &
\path{openbmb/Eurus-7b-sft} &
\path{openbmb/Eurus-7b-kto} &
Offline KTO \citep{yuan2024advancing} &
$5\times10^{-7}$ \\

KTO &
\path{princeton-nlp/Llama-3-Base-8B-SFT} &
\path{princeton-nlp/Llama-3-Base-8B-SFT-KTO} &
Offline KTO \citep{meng2024simposimplepreferenceoptimization} &
$5\times10^{-7}$ \\

PPO &
\path{peiyi9979/mistral-7b-sft} &
\path{peiyi9979/math-shepherd-mistral-7b-rl} &
On-policy PPO \citep{wang2024mathshepherdverifyreinforcellms} &
$1\times10^{-6}$ \\

SimPO &
\path{meta-llama/Meta-Llama-3-8B-Instruct} &
\path{princeton-nlp/Llama-3-Instruct-8B-SimPO} &
Offline, reference-free SimPO \citep{meng2024simposimplepreferenceoptimization} &
$1\times10^{-6}$ \\

PRIME &
\path{PRIME-RL/Eurus-2-7B-SFT} &
\path{PRIME-RL/Eurus-2-7B-PRIME} &
On-policy PRIME \citep{cui2025process} &
$5\times10^{-7}$ \\
\bottomrule
\end{tabularx}

\vspace{2pt}
\begin{minipage}{\textwidth}
\footnotesize
\textit{Notes.}
The SFT mean and median are $9\times10^{-6}$ and $5\times10^{-6}$, respectively.
The RL/preference mean and median are $1.5\times10^{-6}$ and
$5\times10^{-7}$, calculated over the nine checkpoints with disclosed actor
rates and excluding \path{DeepSeek-R1-Zero}. For PRIME, the separately trained
implicit reward model used a learning rate of $1\times10^{-6}$; this value is
excluded because its parameters are not part of the evaluated policy
checkpoint.
\end{minipage}
\end{table*}
\end{document}